\documentclass{article}

\PassOptionsToPackage{numbers, compress}{natbib}

\usepackage{fullpage}
\usepackage{natbib}
\usepackage{tikz}
\usepackage{placeins}

\usepackage[utf8]{inputenc} 
\usepackage[OT1]{fontenc}   
\usepackage{hyperref}       
\usepackage{url}            
\usepackage{booktabs}       
\usepackage{amsfonts,amsmath,amsthm,amssymb}       
\usepackage{nicefrac,multirow}      
\usepackage[table]{xcolor} 
\usepackage{silence}
\usepackage{microtype}      
\usepackage{xcolor}        
\usepackage{algorithm,algorithmic}
\usepackage{subcaption}
\usepackage{graphicx}
\usepackage{csquotes}
\usepackage{placeins,enumitem}

\newtheorem{assumption}{Assumption}
\newtheorem{theorem}{Theorem}
\newtheorem{lemma}{Lemma}
\newtheorem{proposition}{Proposition}
\newtheorem{corollary}{Corollary}
\newtheorem{remark}{Remark}

\usepackage{amsmath, amssymb, amsfonts, amsthm, mathtools}

\usepackage{tikz}
\usetikzlibrary{arrows.meta}
\usepackage{graphicx}

\definecolor{igaCyan}{RGB}{102,193,240}    
\definecolor{igaSteel}{RGB}{122,172,188}   
\definecolor{igaTan}{RGB}{173,143,110}     
\definecolor{igaOrange}{RGB}{221,123,60}   
\definecolor{igaBarLo}{RGB}{101,195,246}   
\definecolor{igaBarMid}{RGB}{153,154,138}  
\definecolor{igaBarHi}{RGB}{234,119,53}   
\definecolor{wallGray}{RGB}{120,118,112}

\DeclareMathOperator{\argmin}{arg\,min}

\newcommand{\R}{\mathbb{R}}

\newcommand{\norm}[1]{\left\lVert #1 \right\rVert}

\newcommand{\ip}[2]{\left\langle #1, #2 \right\rangle}

\newcommand{\E}{\mathbb{E}}

\newcommand{\defeq}{:=}

\usepackage{amsmath,amssymb,amsthm,mathtools}

\theoremstyle{definition}
\newtheorem{definition}{Definition}

\newcommand{\Pp}{\mathbb{P}}
\newcommand{\1}{\mathbf{1}}

\newcommand{\cX}{\mathcal{X}}
\newcommand{\cP}{\mathcal{P}}
\newcommand{\cQ}{\mathcal{Q}}
\newcommand{\cC}{\mathcal{C}}

\newcommand{\Pgen}{\mathcal{P}_{\mathrm{gen}}}
\newcommand{\Pdata}{P_{\mathrm{data}}}

\newcommand{\Hz}{H_0}
\newcommand{\He}{H_\varepsilon}
\newcommand{\Dref}{\mathcal{D}}
\newcommand{\KL}{\mathrm{KL}}

\newcommand{\TV}{\mathrm{TV}}

\DeclareMathOperator{\Tr}{Tr}
\DeclareMathOperator{\supp}{supp}

\newcommand{\opnorm}[1]{\left\lVert #1\right\rVert_{\mathrm{op}}}

\providecommand{\cX}{\mathcal{X}}
\providecommand{\R}{\mathbb{R}}
\providecommand{\E}{\mathbb{E}}
\providecommand{\Tr}{\operatorname{Tr}}
\providecommand{\defeq}{:=}

\providecommand{\KL}{\mathrm{KL}}

\providecommand{\TV}{\mathrm{TV}}
\providecommand{\argmin}{\operatorname*{arg\,min}}

\providecommand{\Pgen}{\mathcal{P}_{\mathcal G}}
\providecommand{\Pdata}{P_{\mathrm{data}}}
\providecommand{\Pref}{P_{\mathrm{ref}}}
\providecommand{\Hz}{H_0}
\providecommand{\He}{H_\varepsilon}
\providecommand{\Dref}{\mathcal{D}}

\providecommand{\cQ}{\mathcal{Q}}
\providecommand{\cC}{\mathcal{C}}
\providecommand{\cP}{\mathcal{P}}
\providecommand{\cD}{\mathcal{D}_{\mathrm{c}}}
\providecommand{\Pp}{\mathbb{P}}
\providecommand{\1}{\mathbf{1}}
\providecommand{\supp}{\operatorname{supp}}
\providecommand{\mHe}{\mathsf{H}}          
\providecommand{\Tset}{\mathbb{T}_\varepsilon} 
\providecommand{\cX}{\mathcal{X}}
\providecommand{\R}{\mathbb{R}}
\providecommand{\E}{\mathbb{E}}
\providecommand{\Tr}{\operatorname{Tr}}
\providecommand{\defeq}{:=}
\providecommand{\norm}[1]{\left\lVert #1 \right\rVert}
\providecommand{\ip}[2]{\left\langle #1,\, #2 \right\rangle}
\providecommand{\KL}{\mathrm{KL}}

\providecommand{\TV}{\mathrm{TV}}
\providecommand{\argmin}{\operatorname*{arg\,min}}
\providecommand{\opnorm}[1]{\left\lVert #1 \right\rVert_{\mathrm{op}}}
\providecommand{\Pgen}{\mathcal{P}_{\mathcal G}}
\providecommand{\Pdata}{P_{\mathrm{data}}}
\providecommand{\Pref}{P_{\mathrm{ref}}}
\providecommand{\Hz}{H_0}
\providecommand{\He}{H_\varepsilon}
\providecommand{\Dref}{\mathcal{D}}

\providecommand{\cQ}{\mathcal{Q}}
\providecommand{\cC}{\mathcal{C}}
\providecommand{\cP}{\mathcal{P}}
\providecommand{\cD}{\mathcal{D}_{\mathrm{c}}}
\providecommand{\Pp}{\mathbb{P}}
\providecommand{\1}{\mathbf{1}}
\providecommand{\supp}{\operatorname{supp}}
\providecommand{\mHe}{\mathsf{H}}            
\providecommand{\Gen}[1]{G^{\varepsilon}_{#1}} 
\providecommand{\Tset}{\mathbb{T}_\varepsilon} 

\definecolor{population}{HTML}{111827}
\definecolor{fitted}{HTML}{2563EB}
\definecolor{entropywall}{HTML}{059669}
\definecolor{beyondwall}{HTML}{D97706}
\definecolor{rareregion}{HTML}{C2413B}
\definecolor{igashade}{RGB}{245,247,250}

\newcommand{\toyfiglegenddot}[1]{
  \tikz[baseline=-0.55ex]\fill[#1] (0,0) circle[radius=2pt];}
\newcommand{\igacell}[1]{\cellcolor{igashade}#1}

\hypersetup{hidelinks}

\title{\textbf{Imaginative Generative AI\hspace{0.7mm}: Crossing the Entropy Wall\\ into Worlds Beyond Imitation}}

\author{
Farzan Farnia\thanks{Contributed Equally and Listed in Alphabetical Order. Department of Computer Science and Engineering, The Chinese University of Hong Kong, \{farnia,hosseingoli\}@cse.cuhk.edu.hk.} ,
Hossein Goli\footnotemark[1] ,
Amin Gohari\thanks{Department of Information Engineering, The Chinese University of Hong Kong, agohari@ie.cuhk.edu.hk},
}  

\date{}

\begin{document}

\maketitle

\begin{abstract}
Generative AI models are primarily designed to imitate the data distribution, an objective that neither corrects diversity lost by a learned generator nor defines how generation should extend beyond the diversity of the data itself. We introduce \textbf{\underline{I}maginative \underline{G}enerative \underline{A}I (\textbf{IGA})}, a framework that makes diversity part of the target-distribution design problem: among distributions close to a reference, IGA selects one whose spectral diversity reaches a prescribed level. Diversity is measured by the \textbf{von Neumann entropy} of the generated distribution's kernel covariance operator in a fixed representation space, providing a reference-free representation-guided measure of how broadly probability mass occupies embedding directions. The spectral entropy of the population data distribution defines an \textbf{Entropy Wall}. Below the wall, IGA performs \textbf{diversity repair}, recovering variation that a learned generator has lost while remaining within the diversity level of the data. Beyond the wall, the data distribution itself becomes infeasible, and IGA deliberately departs from it to produce distributions with greater representation-relative spectral diversity, an operational notion of \textbf{imaginative generation}. These regimes form a single regularization path from imitation to imagination and define an i.i.d.\ target distribution at each prescribed diversity level. We develop the theory of this entropy-constrained projection and show that, under a KL anchor to a pretrained generator, the optimum satisfies a self-consistent exponential-tilt relation. This characterization leads to \emph{IGA Guidance}, a retraining-free inference-time method for score-based and diffusion models, including DDPM and DDIM samplers. Experiments on synthetic and vision benchmarks demonstrate diversity repair below the Entropy Wall and controlled spectral extrapolation beyond it.
\end{abstract}

\section{Introduction}
\begin{center}  
  \begin{minipage}{0.99\textwidth}    \centering    \textit{``Imagination is more important than knowledge. Knowledge is limited. Imagination encircles the world.''} \\[0.5em]     --- \textsc{Albert Einstein}  \end{minipage}
\end{center}

The typical goal of a generative model is to reproduce the underlying distribution of its training data. For example, a successful image generator is expected to produce realistic images with approximately the same content and variation as the images on which it was trained. This principle underlies the standard paradigms of generative modeling in the literature, including generative adversarial networks (GANs)~\citep{goodfellow2014generative}, variational autoencoders (VAEs)~\citep{kingma2014auto}, and score-based and diffusion models~\citep{song2019generative,ho2020denoising,song2021score}. Although these frameworks differ significantly in architecture and training, their population-level goal can be summarized in the following distributional discrepancy minimization:
\begin{equation}
\label{eq:intro-imitation}
    \min_{Q\in\mathcal P_{\mathcal{G}}}\;
    \mathcal D\bigl(Q ; P_{\mathrm{data}}\bigr),
\end{equation}
where \(P_{\mathrm{data}}\) is the underlying data distribution, \(Q\) is the distribution produced by the generator over the set of feasible models $\mathcal P_{\mathcal{G}}$, and \(\mathcal D(\cdot ; \cdot)\) measures the discrepancy (or divergence) between the two distributions. We refer to this prevailing view of generative modeling as \textbf{distributional imitation}.

Imitation is a natural statistical objective, but it also places a ceiling on what the generator is asked to do. Even an ideal solution of~\eqref{eq:intro-imitation} is asked to match \(P_{\mathrm{data}}\), not to produce a distribution that is systematically more diverse or generate novel and creative content. Moreover, practical generators may not reach even the diversity of their training distribution. The recent study \citep{farnia2026exposing} by Farnia, et al. has found that generated samples can exhibit lower spectral diversity than real data when diversity is measured using reference-free diversity measures of the Vendi score \cite{friedman2023vendi} and Rényi kernel entropy \cite{jalali2023information}. This raises a fundamental question: 
\begin{center}
   \emph{How should the target of generative modeling be regularized when diversity and novelty, in addition to fidelity, are something we want to control?} 
\end{center}

To address this question, we propose \textbf{\underline{I}maginative \underline{G}enerative \underline{A}I} (\textbf{IGA}), a framework that makes diversity part of the target-distribution design problem. Instead of asking only for the distribution closest to a reference distribution, IGA asks for the closest distribution whose diversity score is at least above a given prescribed level. Let $P_{\mathrm{ref}}$ denote the reference distribution, which could be the empirical distribution $\widehat P_n$ of training data or the distribution $P_\theta$ of a pretrained generator. Then, IGA solves the following regularized discrepancy minimization problem:
\begin{equation}
\label{eq:intro-ega-constrained}
\begin{aligned}
    \min_{Q\in\mathcal{P_G}}\quad
        &\mathcal D\!\left(Q;P_{\mathrm{ref}}\right)
        \\
    \text{subject to}\quad
        &H(Q)\geq \rho,
\end{aligned}
\end{equation}
where \(H(Q)\) measures the spectral entropy (interpreted as diversity) of distribution \(Q\) and \(\rho\) is the desired diversity level. Under the duality conditions developed in our theoretical analysis, the constrained problem at level \(\rho\) can equivalently be stated using a Lagrangian penalty at a corresponding multiplier $\lambda\geq 0$:\vspace{-0.5mm}
\begin{equation}
\label{eq:intro-ega-penalized}
    \min_{Q\in\mathcal{P_G}}\;
    \mathcal D\!\left(Q;P_{\mathrm{ref}}\right)
    -\lambda\, H(Q)
\end{equation}
Note that the two terms have complementary roles: The discrepancy term keeps generated samples close to the reference distribution, while the entropy term rewards the spectral diversity in the Vendi score. Setting \(\lambda=0\) recovers standard reference matching; increasing \(\lambda\) gives more weight to the spectral diversity term.

We measure diversity using the \textbf{von~Neumann entropy (VNE)} of the normalized kernel covariance operator induced by \(Q\) in a fixed embedding space. Intuitively, VNE is low when generated samples concentrate along a few embedding directions and high when they spread across many directions; in the empirical setting, it is the logarithm of the Vendi score~\citep{friedman2023vendi,jalali2023information,ospanov2024scalable}. This measure is reference-free but representation-dependent: it evaluates the diversity of \(Q\) without requiring a comparison distribution, while the chosen embedding specifies which variations are meaningful. IGA thus controls spectral diversity relative to a given embedding.

The data distribution itself provides a natural reference level for this diversity. To characterize this wall, we define the underlying distribution's entropy as 
\begin{equation}
\label{eq:intro-entropy-wall}
    \rho_\star
    :=
    H\bigl(P_{\mathrm{data}}\bigr).
\end{equation}
We call \(\rho_\star\) the \textbf{Entropy Wall}. It is the spectral diversity of the data distribution in the chosen representation. The notion of entropy wall separates two different regimes in the IGA generative modeling approach:\vspace{2mm}

\noindent \textbf{(Regime I) Below the Entropy Wall: Diversity Repair.}
The below-the-wall regime concerns IGA when we choose entropy lower-bound $\rho$ to satisfy $\rho\leq\rho_\star$.

In this regime, the required diversity level is no greater than the diversity already present in the underlying data distribution. Therefore, in this regime, IGA can then be viewed as \textbf{repairing a diversity deficit in a learned generator}: it encourages the generator to recover variation that was present in the data but weakened or lost during training the generative model. Note that, as empirically demonstrated by Farnia et al. in recent work \cite{farnia2026exposing}, the standard generative models commonly suffer from a diversity bias, and the spectral entropy of their generated data cannot match that of the underlying distribution generating their training samples. In brief, the goal in this regime remains faithful modeling of the data, with an explicit mechanism for counteracting spectral diversity shortfall as shown in \cite{farnia2026exposing}.

\noindent \textbf{(Regime II) Beyond the Entropy Wall: Imaginative Generation.}
This regime of applying IGA is when we select the projection lower-bound to satisfy the strict inequality: $\rho>\rho_\star$.

Especially, we highlight that in this regime, the data distribution itself no longer satisfies the diversity constraint, and thus \textbf{the IGA solution in~\eqref{eq:intro-ega-constrained} must intentionally differ from \(P_{\mathrm{data}}\)}, regardless of whether it is anchored directly to the data or to a pretrained model. The discrepancy term prevents this solution from moving arbitrarily far from the chosen reference \(P_{\mathrm{ref}}\), while the entropy constraint pushes it to occupy a broader set of embedding directions. We call this regime \emph{imaginative} because the target has greater spectral diversity than the data distribution that defines the wall.

\begin{figure}[t]
    \centering
    \includegraphics[width=0.8\textwidth]{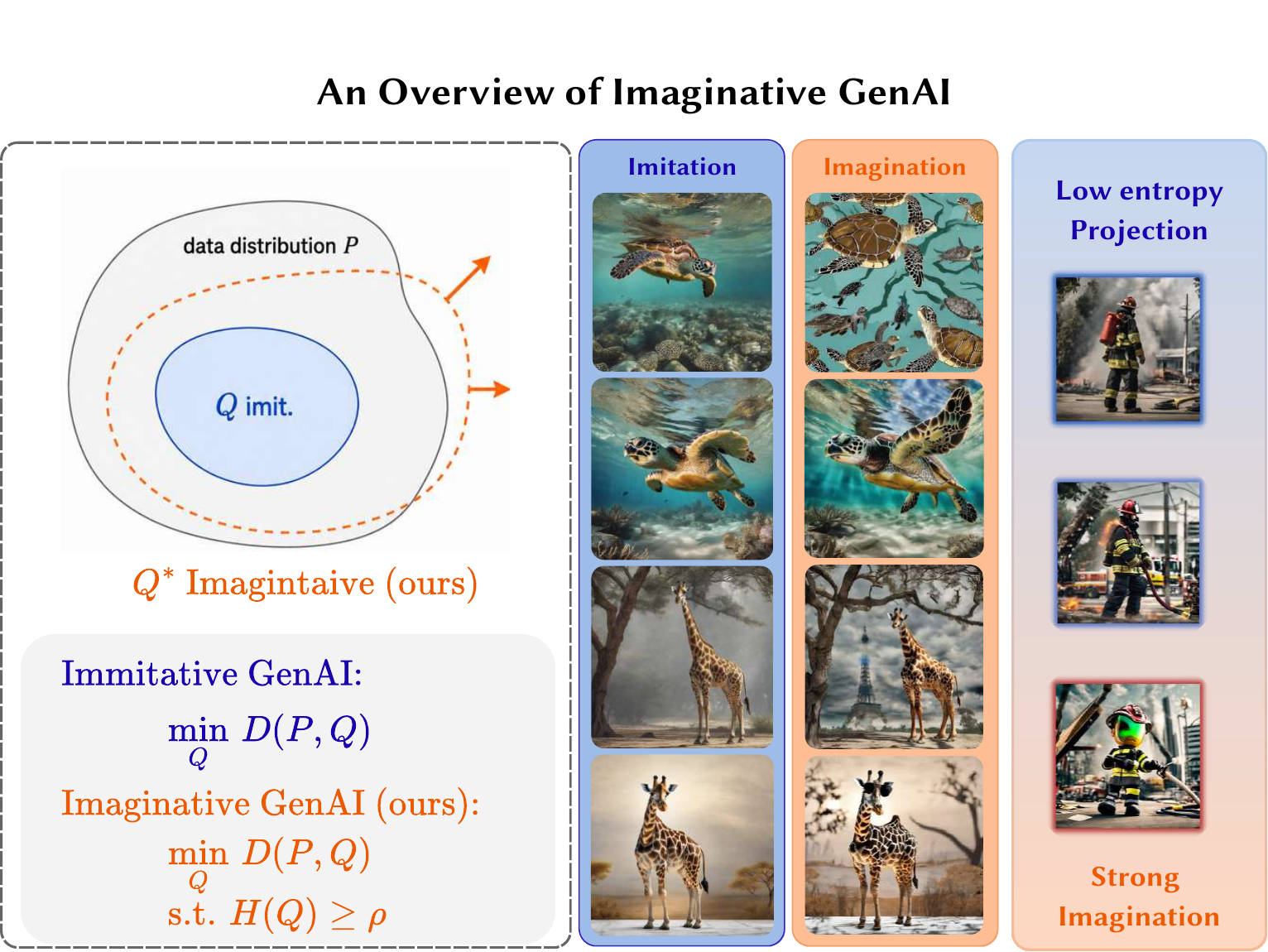} 
    \caption{
        \textbf{From Imitation to Imagination.} Imitation aims to return a distribution $Q$ close to
        $P_{\mathrm{data}}$; IGA returns the closest
        $Q$ with $H(Q)\ge\rho$, which meets the data distribution below and at the entropy wall
        $\rho_{\star}=H(P_{\mathrm{data}})$ and leaves it beyond. Images are Stable
        Diffusion XL at matched prompts and seeds; the right column sweeps one
        multiplier from the low-entropy projection to strong extrapolation.
    }
    \label{fig:overview}\vspace{2mm}
\end{figure}

\begin{figure}
    \centering
    \includegraphics[width=0.8\linewidth]{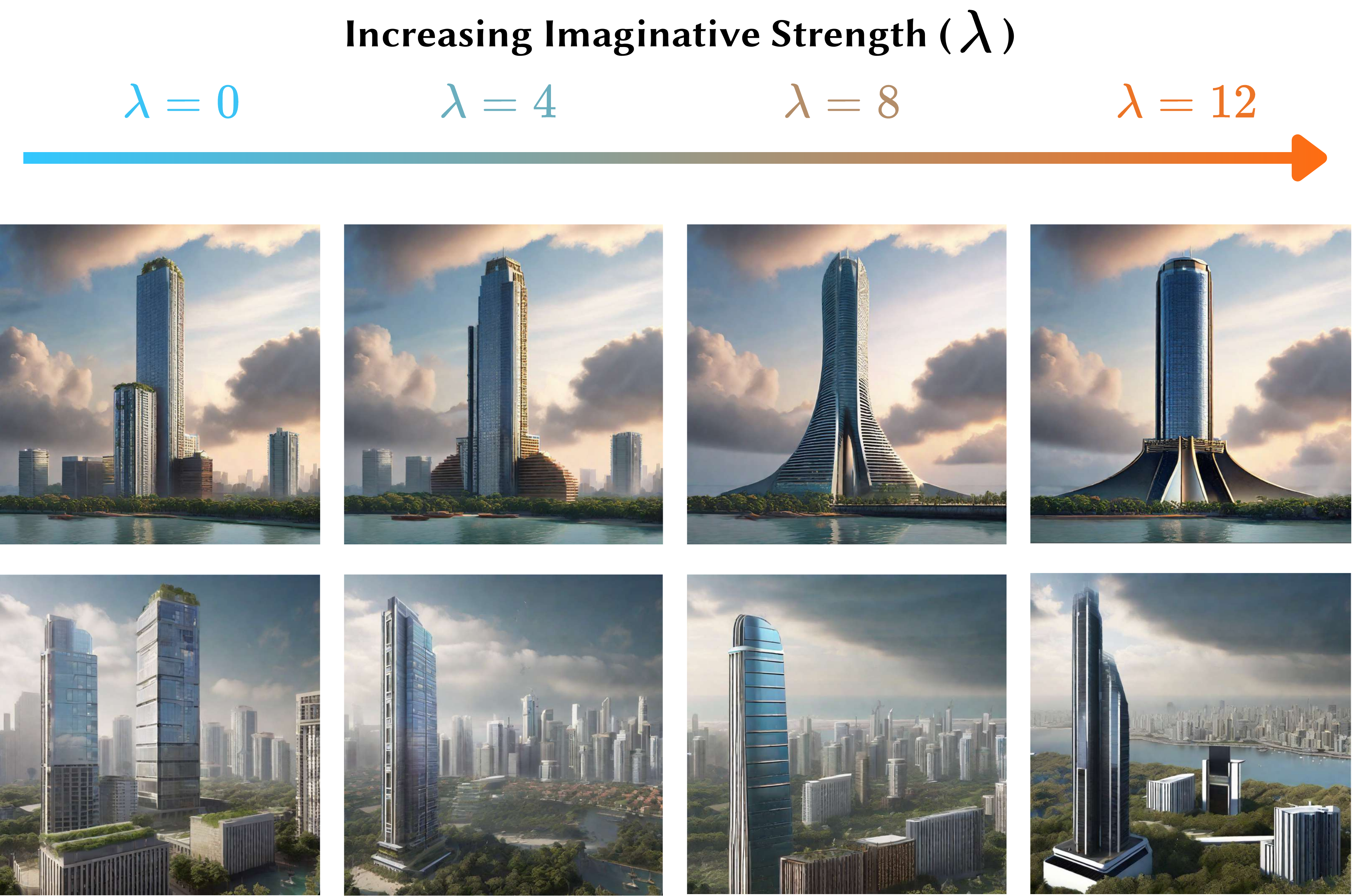}
    \caption{\textbf{From imitation to imagination with SDXL.}
    As the IGA multiplier $\lambda$ increases, SDXL produces progressively stronger structural and compositional variations while preserving the underlying concept.}
    \label{fig:iga_sdxl}
\end{figure}

Here, we use the term \textbf{imaginative} to describe generation whose spectral diversity exceeds that of the data in a specified representation space, e.g. CLIP or DINO embedding spaces for images. This representation-relative definition makes imagination operational while allowing the embedding to encode domain-relevant semantics. We evaluate the resulting variation using both the guiding representation and independent measures of sample quality and diversity.

The two regimes form a single regularization path. As the target 
$\rho$ increases, IGA moves from ordinary imitation to diversity repair and then, after crossing the entropy wall $\rho_\star = H(P_{\rm data})$ (corresponding to the real data distribution), to deliberate imagination. Therefore, our IGA framework defines diversity enhancement as a property of a single target distribution, enabling i.i.d. generation at a prescribed spectral-diversity level, including levels beyond the entropy of the data. We emphasize that this differs from methods that induce diversity through repulsive or sample-dependent interactions over the course of generating multiple samples, whose outputs form a coupled and generally non-i.i.d. batch~\citep{sadat2024cads,corso2024particle,jalali2025sparke}. IGA therefore provides an i.i.d. distributional alternative to interaction-based diversity promotion, while complementing work that evaluates diversity after generation~\citep{friedman2023vendi,jalali2023information}.

Next, we demonstrate that this distribution-level formulation further leads to a practical method for diversity-improved sampling from pretrained score-based and diffusion models. When the reference is a pretrained model \(P_\theta\) and the discrepancy measure is the KL-divergence $\mathrm{KL}(Q\| P_\theta)$, the analysis in the main body shows that the optimal target takes the form
\begin{equation}
\label{eq:intro-exponential-tilt}
    Q^\star(dx)
    \,\propto\,
    P_\theta(dx)\cdot
    \exp\bigl(
        \lambda\, G_{Q^\star}(x)
    \bigr),
\end{equation}
where $G_{Q^\star}(x)$ measures how placing probability mass near $x$ changes the VNE spectral entropy term. Because this energy depends on $Q^\star$, the relation is self-consistent. It reweights the base model to increase the spectral diversity of the generated population while remaining close to the original distribution as much as possible.

We specifically show that the application of this formulation to score-based and generative models can be performed by our proposed \textbf{IGA~Guidance}. IGA~Guidance is an inference-time approximation for score-based, DDPM, and DDIM samplers that requires no retraining. Our numerical experiments show promising results of the IGA guidance for large-scale diffusion models. For example, Figure~\ref{fig:iga_sdxl} shows the application of IGA-Guidance to the large-scale SD-XL model and how increasing the parameter $\lambda$ leads to visually more diverse and imaginative image outputs for the input prompt "A Skyscraper". The summary of our contributions are as follows:
\begin{itemize}[leftmargin=1.5em]
    \item We formulate \emph{IGA}, an entropy-constrained projection framework enabling i.i.d.\ generation at prescribed VNE diversity levels.

    \item We introduce the \emph{entropy wall}, separating diversity repair from controlled extrapolation beyond the data's spectral diversity.

    \item We characterize the IGA regularization path and derive a self-consistent exponential tilt for the KL-anchored optimum.

    \item We develop \emph{IGA guidance} for inference-time steering of pretrained score-based, DDPM, and DDIM samplers without retraining.

\end{itemize}

\newcommand{\HhatZero}{2.16}     
\newcommand{\HhatFour}{2.64}     
\newcommand{\HhatEight}{3.36}    
\newcommand{\HhatTwelve}{3.65}   
\definecolor{titleBlue}{RGB}{38,126,194}
\definecolor{titlePurple}{RGB}{112,92,200}
\definecolor{titleOrange}{RGB}{224,103,38}

\begin{figure}[H]
\centering
\resizebox{\linewidth}{!}{
\begin{tikzpicture}[x=1cm, y=1cm]

\def\SW{2.82}                  
\def\xL{1.7}
\def\xR{15.4}
\def\BandH{3.48}

\def\ybandA{0}                 
\def\ybandB{3.83}              
\def\ywall{7.46}               
\def\ybandC{8.81}              
\def\ybandD{12.64}             

\def\xmid{8.55}

\newcommand{\Slot}[3]{
  \IfFileExists{#3}{
    \node[
      anchor=south west,
      inner sep=0pt
    ] at (#1,{#2}) {
      \includegraphics[
        width=\SW cm,
        height=\SW cm
      ]{#3}
    };
  }{
    \draw[
      fill=gray!12,
      draw=gray!35,
      thin
    ]
      (#1,{#2})
      rectangle
      ++(\SW,\SW);

    \node[
      gray,
      font=\scriptsize,
      align=center
    ]
      at ({#1+0.5*\SW},{#2+0.5*\SW})
      {\detokenize{#3}};
  }
}

\newcommand{\SlotRow}[5]{
  \Slot{2.15}{#1+0.13}{#2}
  \Slot{5.48}{#1+0.13}{#3}
  \Slot{8.81}{#1+0.13}{#4}
  \Slot{12.14}{#1+0.13}{#5}
}

\newcommand{\BandBox}[3]{
  \draw[
    #3,
    draw=#2!80!black,
    fill=#2!7,
    rounded corners=6pt
  ]
    (\xL,#1)
    rectangle
    (\xR,{#1+\BandH});
}

\newcommand{\BandLabels}[4]{

  \node[
    anchor=west,
    inner sep=0pt,
    font=\normalsize\bfseries\boldmath,
    text=#2!55!black
  ]
    at (2.15,{#1+\BandH-0.24})
    {#3};

  \node[
    anchor=east,
    inner sep=0pt,
    font=\normalsize\bfseries\boldmath,
    text=#2!55!black
  ]
    at (14.96,{#1+\BandH-0.24})
    {#4};
}

\newcommand{\Frag}[5]{
  \draw[
    fill=wallGray!15,
    draw=wallGray!65,
    line width=0.4pt,
    rotate around={
      #5:({#1+#3/2},{\ywall+#2+#4/2})
    }
  ]
    (#1,{\ywall+#2})
    rectangle
    ++(#3,#4);
}

\fill[
  igaOrange!7,
  rounded corners=8pt
]
  (1.45,{\ywall+1.08})
  rectangle
  (15.65,16.8);

\fill[
  igaCyan!8,
  rounded corners=8pt
]
  (1.45,-0.15)
  rectangle
  (15.65,{\ywall-0.18});

\node[
  anchor=east,
  font=\large\bfseries\itshape\boldmath,
  text=igaOrange!60!black
]
  at (15.35,16.5)
  {imagined worlds ($\rho > \rho_\star$)};

\shade[
  bottom color=igaBarLo,
  middle color=igaBarMid,
  top color=igaBarHi
]
  (0.93,-0.20)
  rectangle
  (1.07,16.30);

\fill[igaBarHi]
  (0.78,16.26)
  --
  (1.22,16.26)
  --
  (1.00,16.72)
  --
  cycle;

\node[
  font=\large\bfseries\boldmath,
  anchor=south west,
  text=black!70
]
  at (0.20,16.82)
  {spectral diversity $H$};

\draw[
  dashed,
  black!55
]
  (1.10,{\ywall+0.50})
  --
  (\xL,{\ywall+0.50});

\node[
  anchor=east,
  font=\large\bfseries\boldmath,
  text=black!70
]
  at (0.80,{\ywall+0.50})
  {$\rho_\star$};

\BandBox
  {\ybandA}
  {igaCyan}
  {line width=0.6pt}

\SlotRow
  {\ybandA}
  {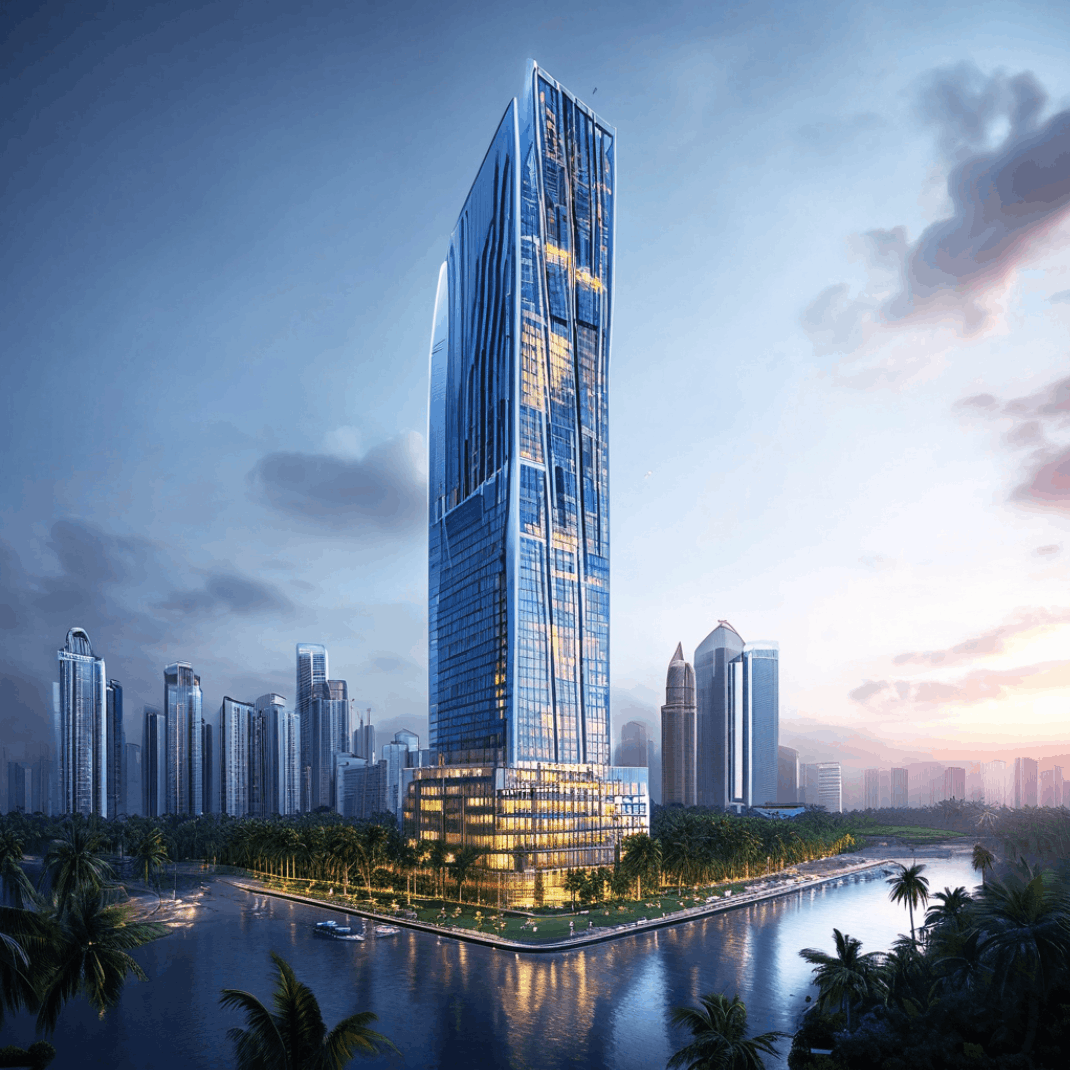}
  {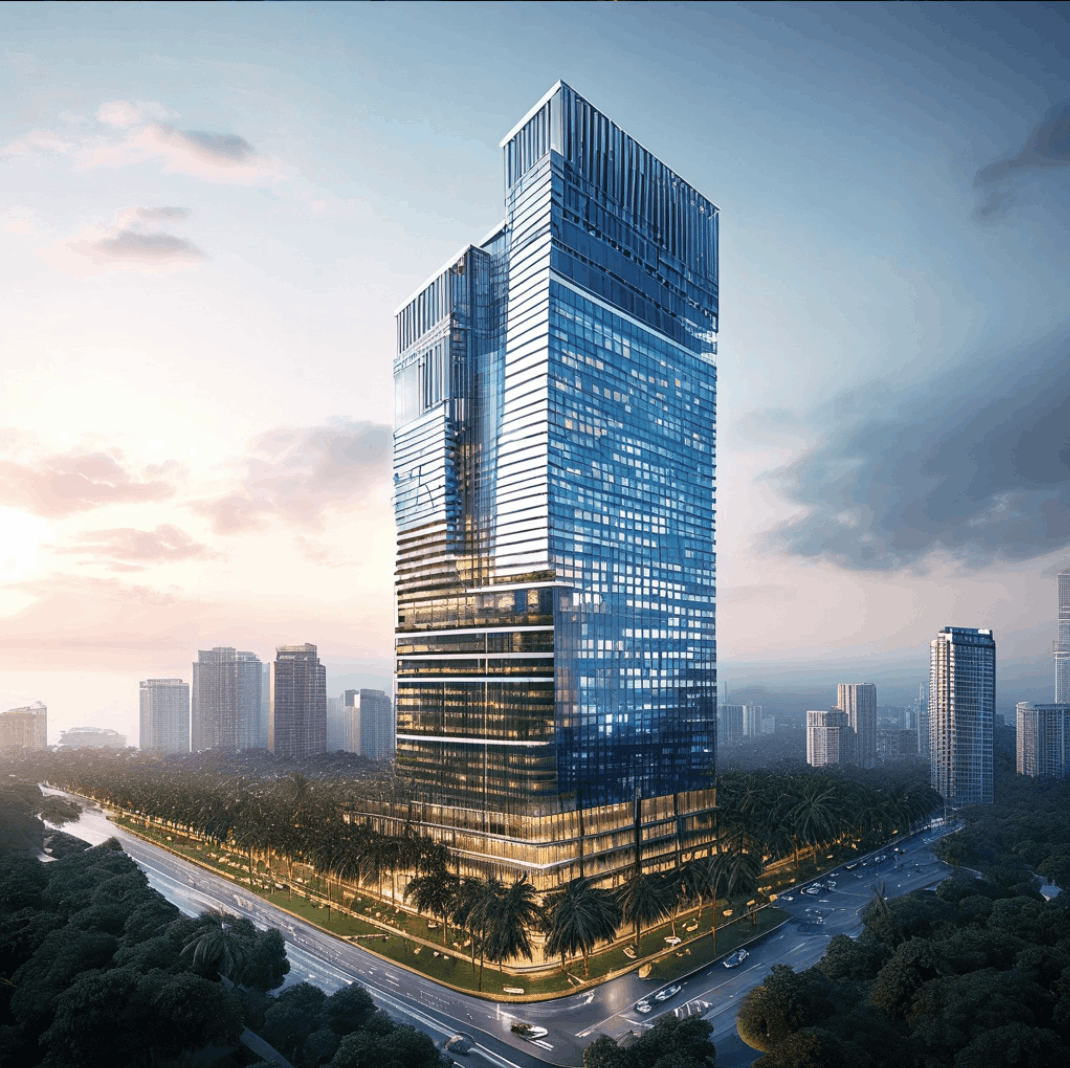}
  {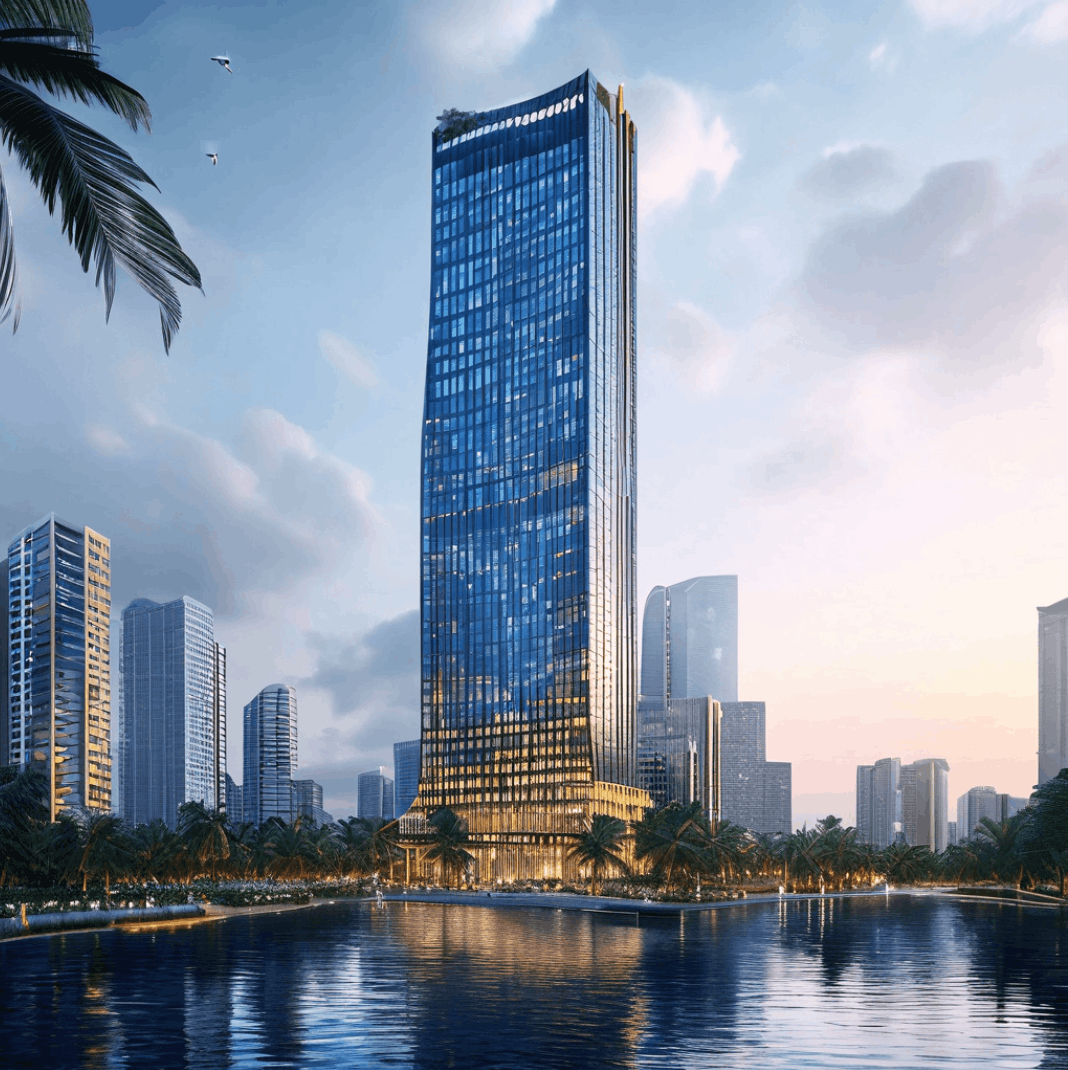}
  {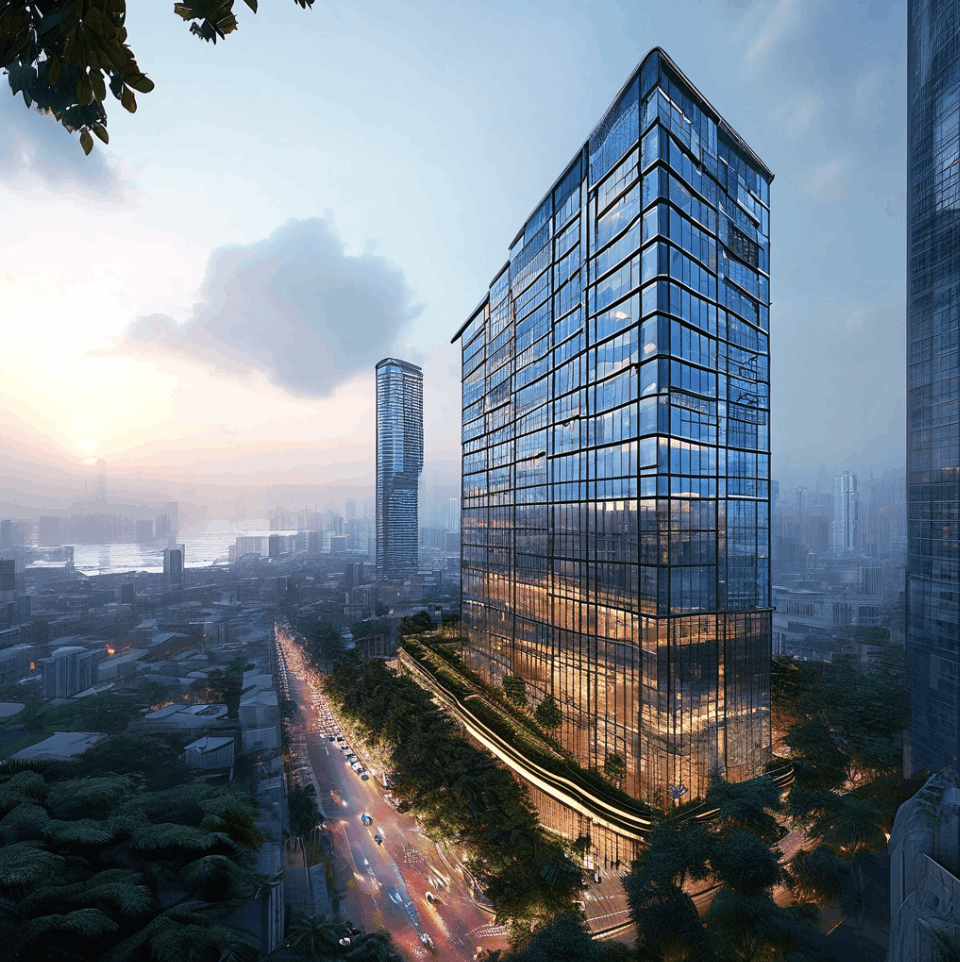}

\BandLabels
  {\ybandA}
  {igaCyan}
  {$\lambda = 0$ \, (anchor)}
  {$H = \HhatZero$}

\BandBox
  {\ybandB}
  {igaSteel}
  {line width=0.6pt}

\SlotRow
  {\ybandB}
  {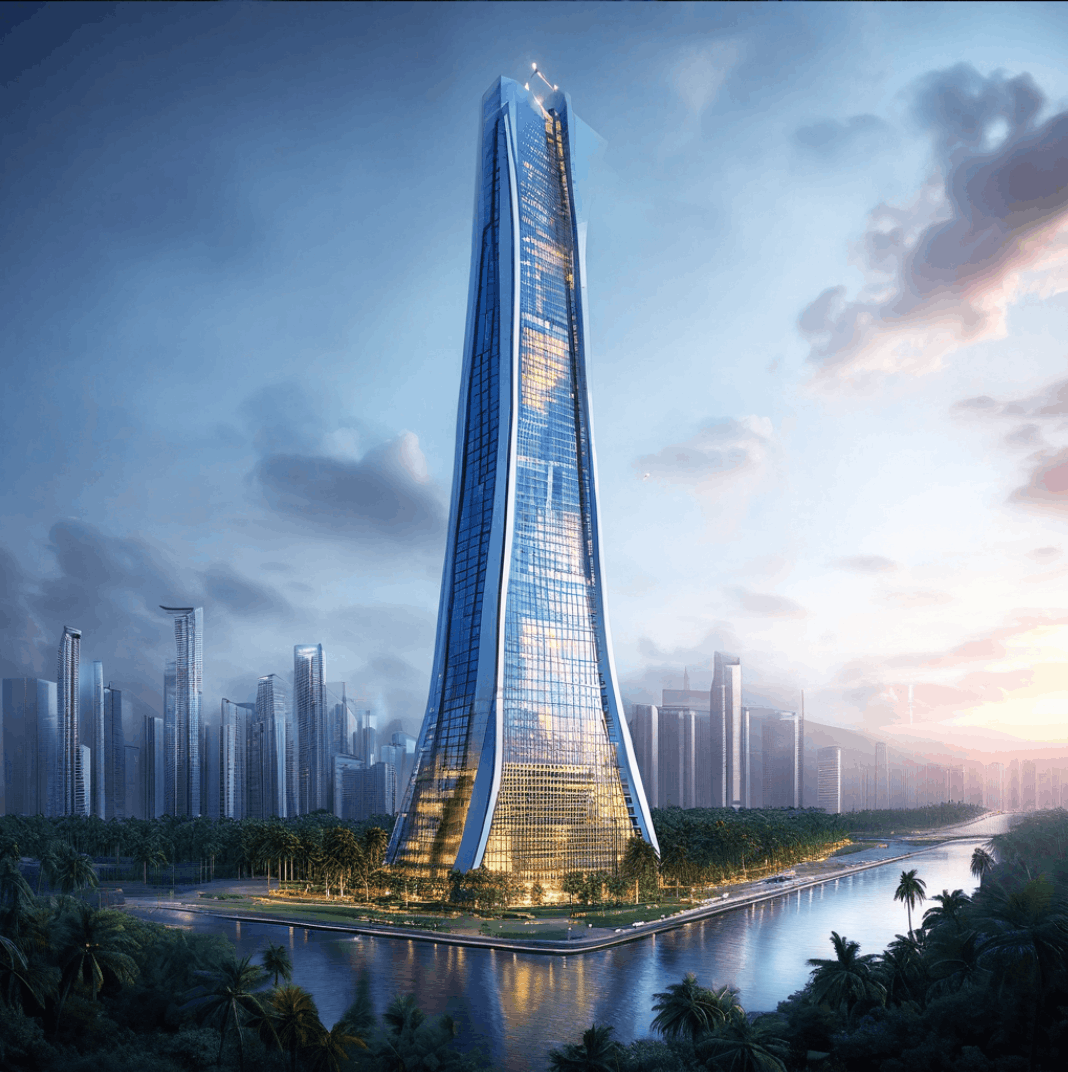}
  {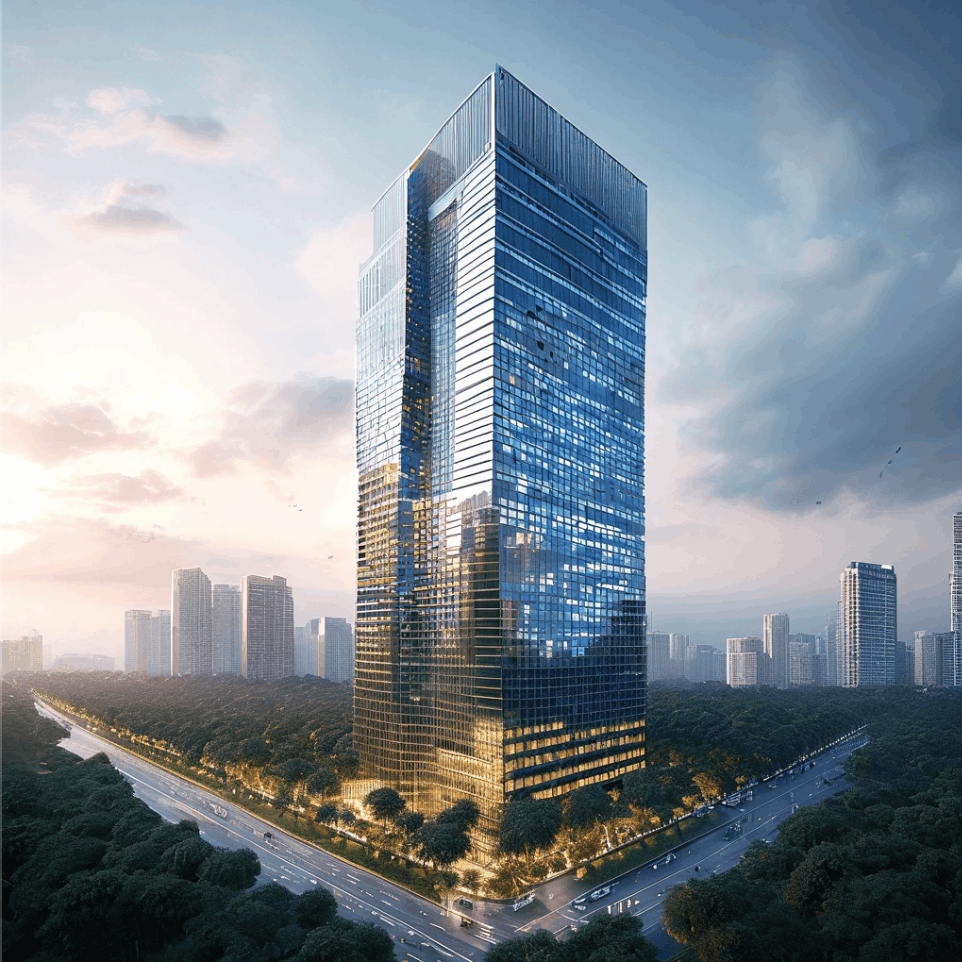}
  {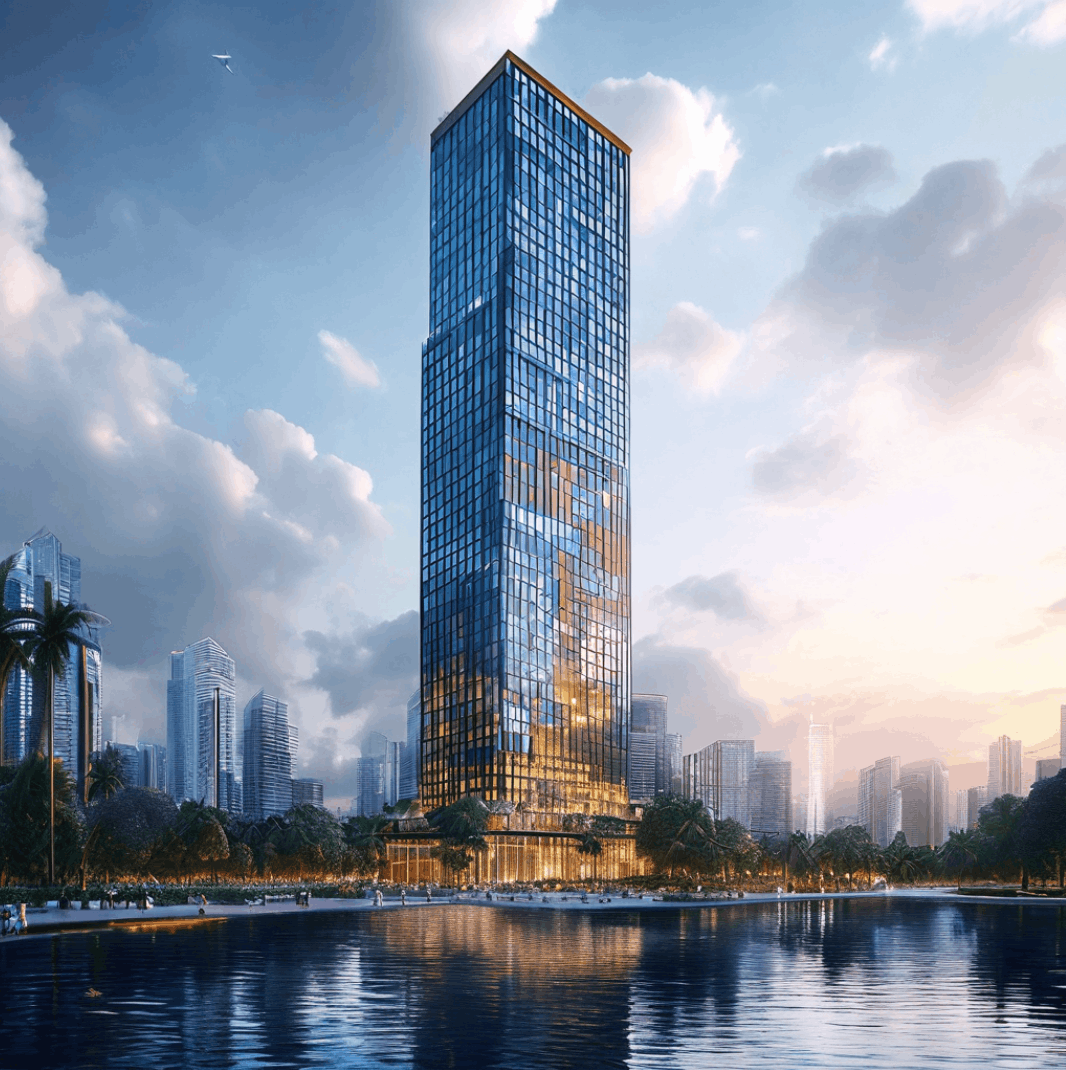}
  {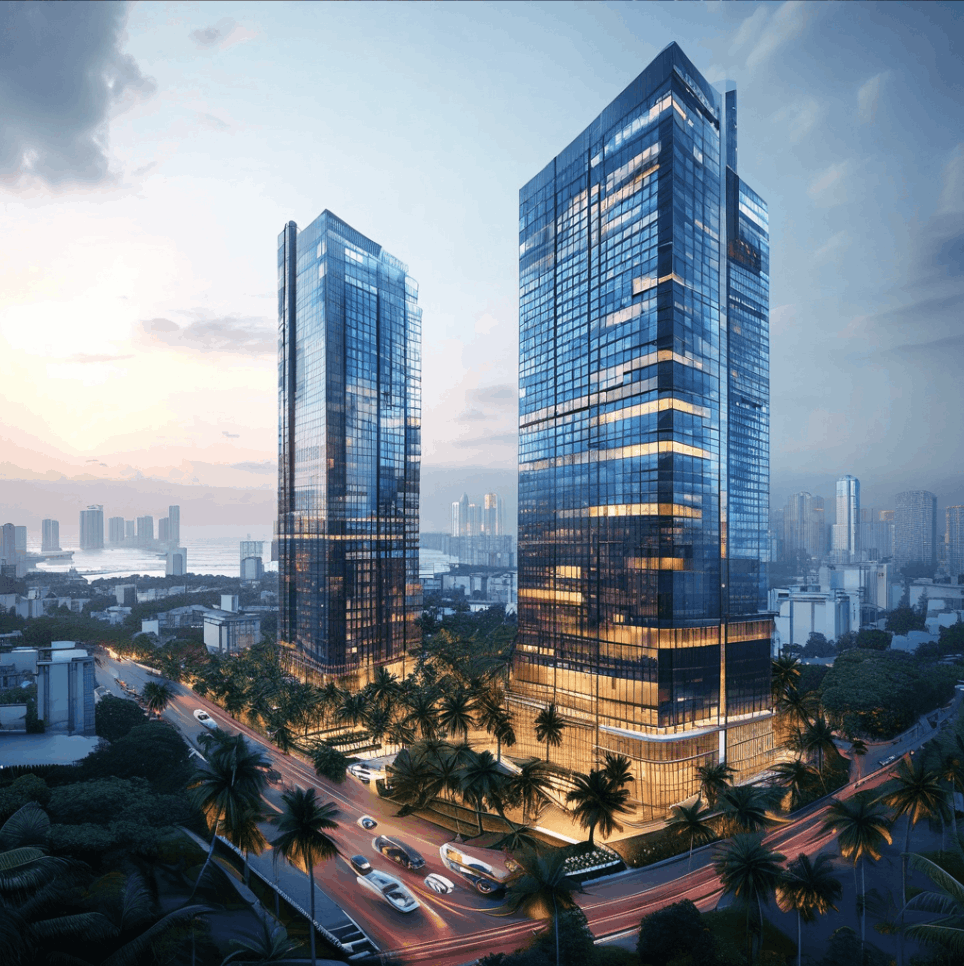}

\BandLabels
  {\ybandB}
  {igaSteel}
  {$\lambda = 5$}
  {$H = \HhatFour$}

\BandBox
  {\ybandC}
  {igaTan}
  {dash pattern=on 3.6pt off 2.4pt, line width=0.8pt}

\SlotRow
  {\ybandC}
  {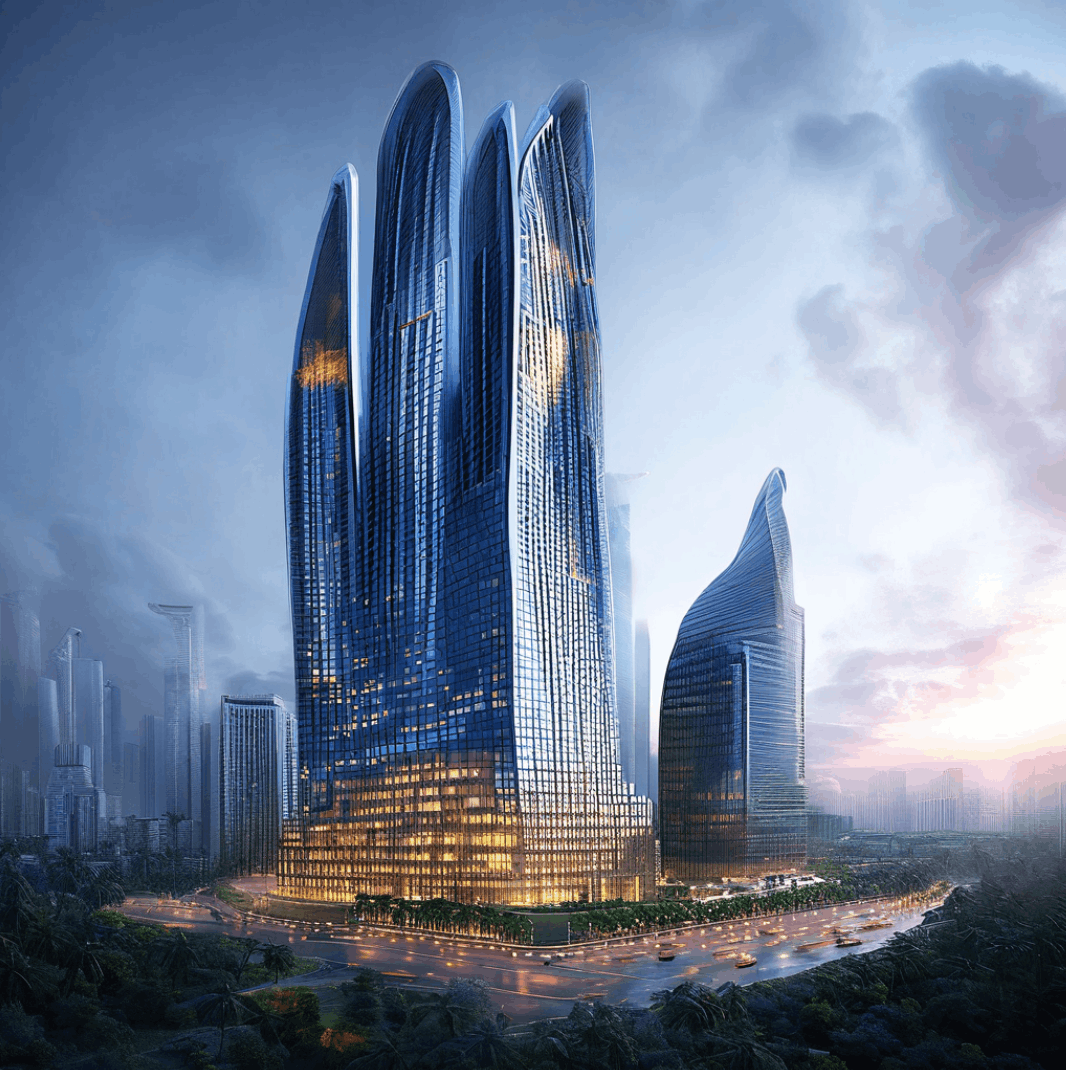}
  {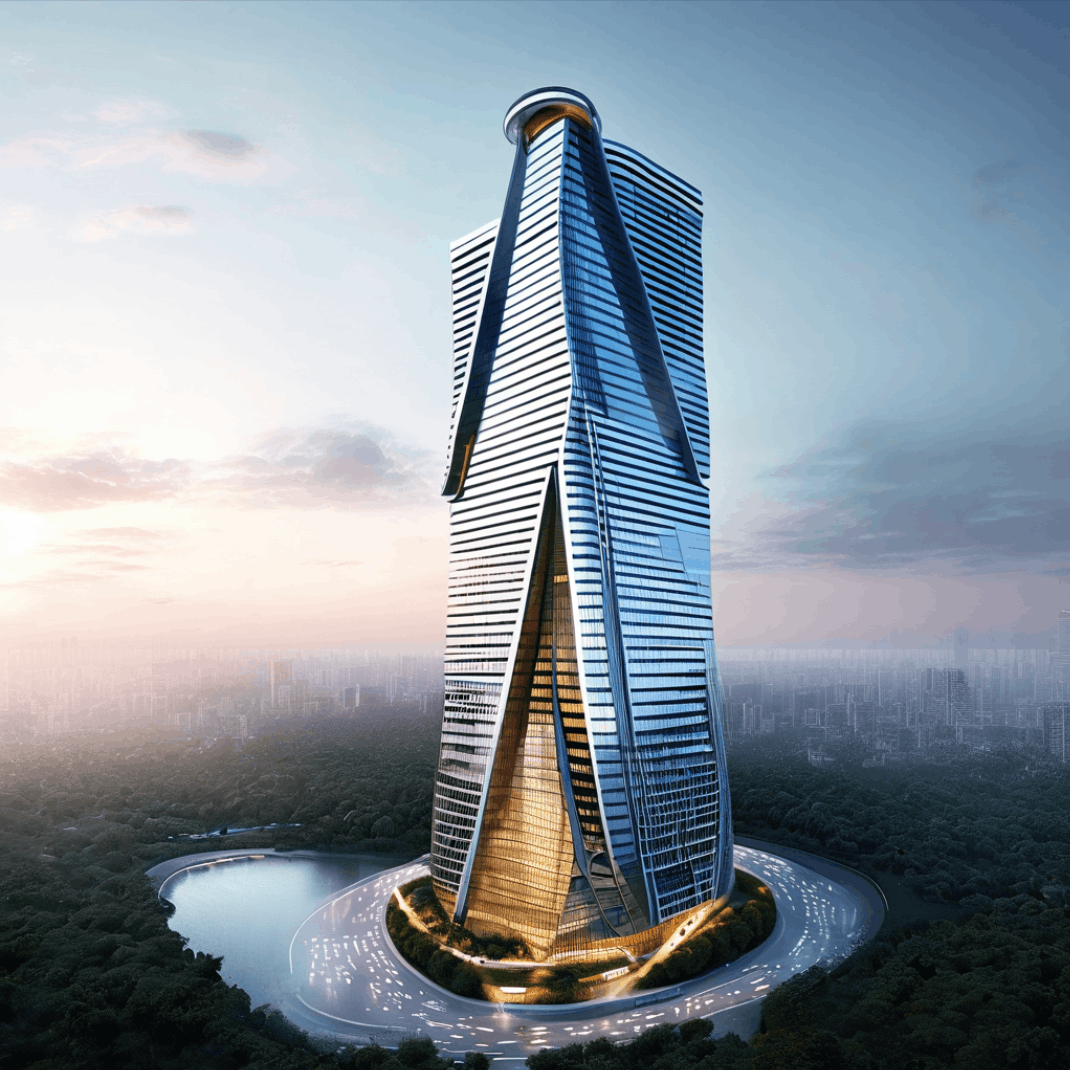}
  {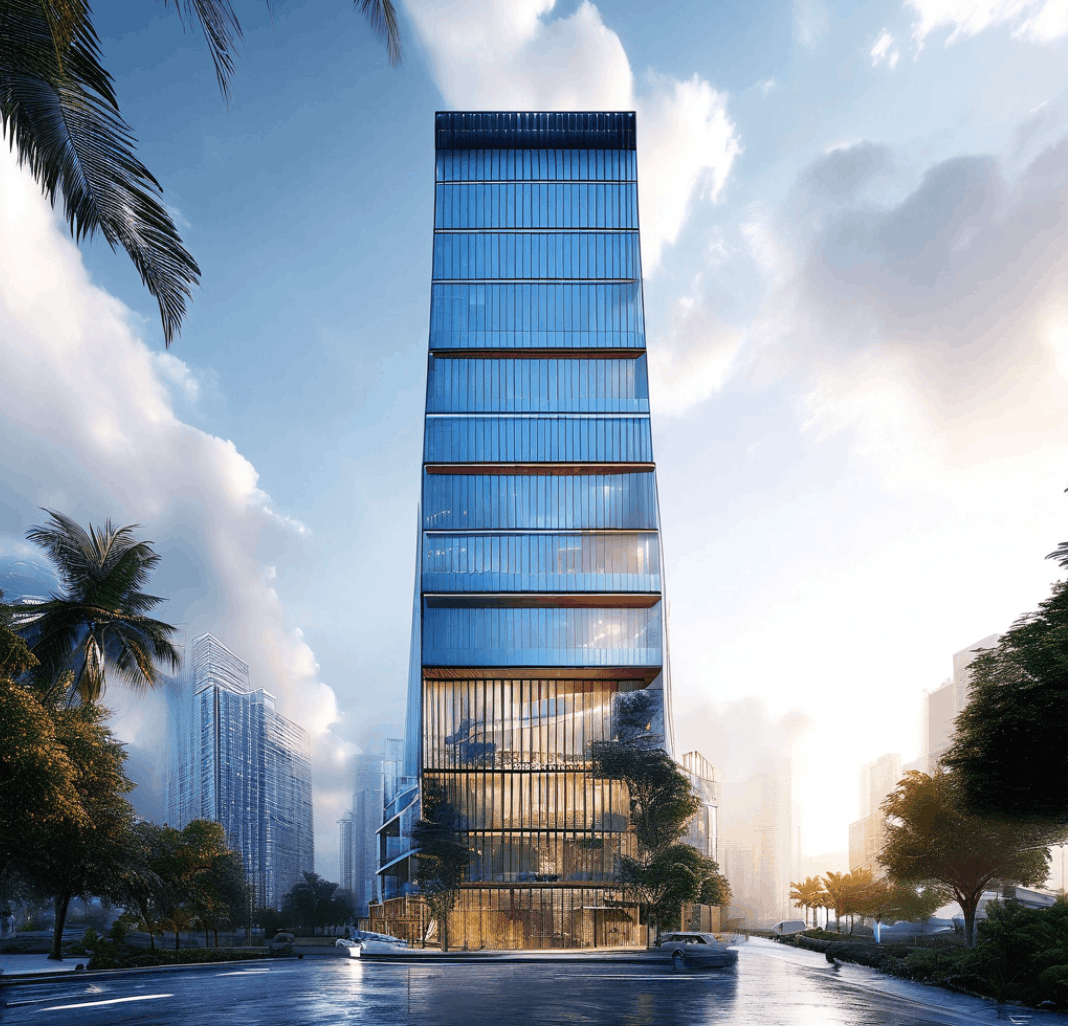}
  {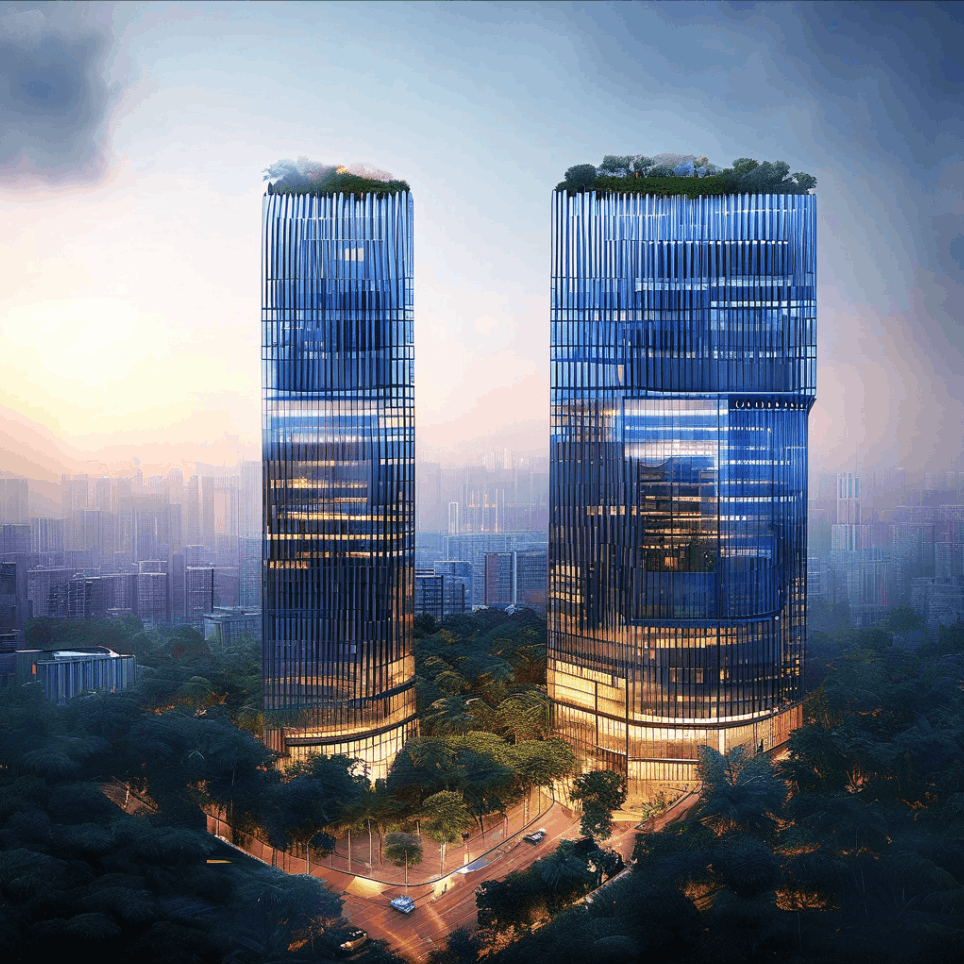}

\BandLabels
  {\ybandC}
  {igaTan}
  {$\lambda = 20$}
  {$H = \HhatEight$}

\BandBox
  {\ybandD}
  {igaOrange}
  {dash pattern=on 3.6pt off 2.4pt, line width=0.8pt}

\SlotRow
  {\ybandD}
  {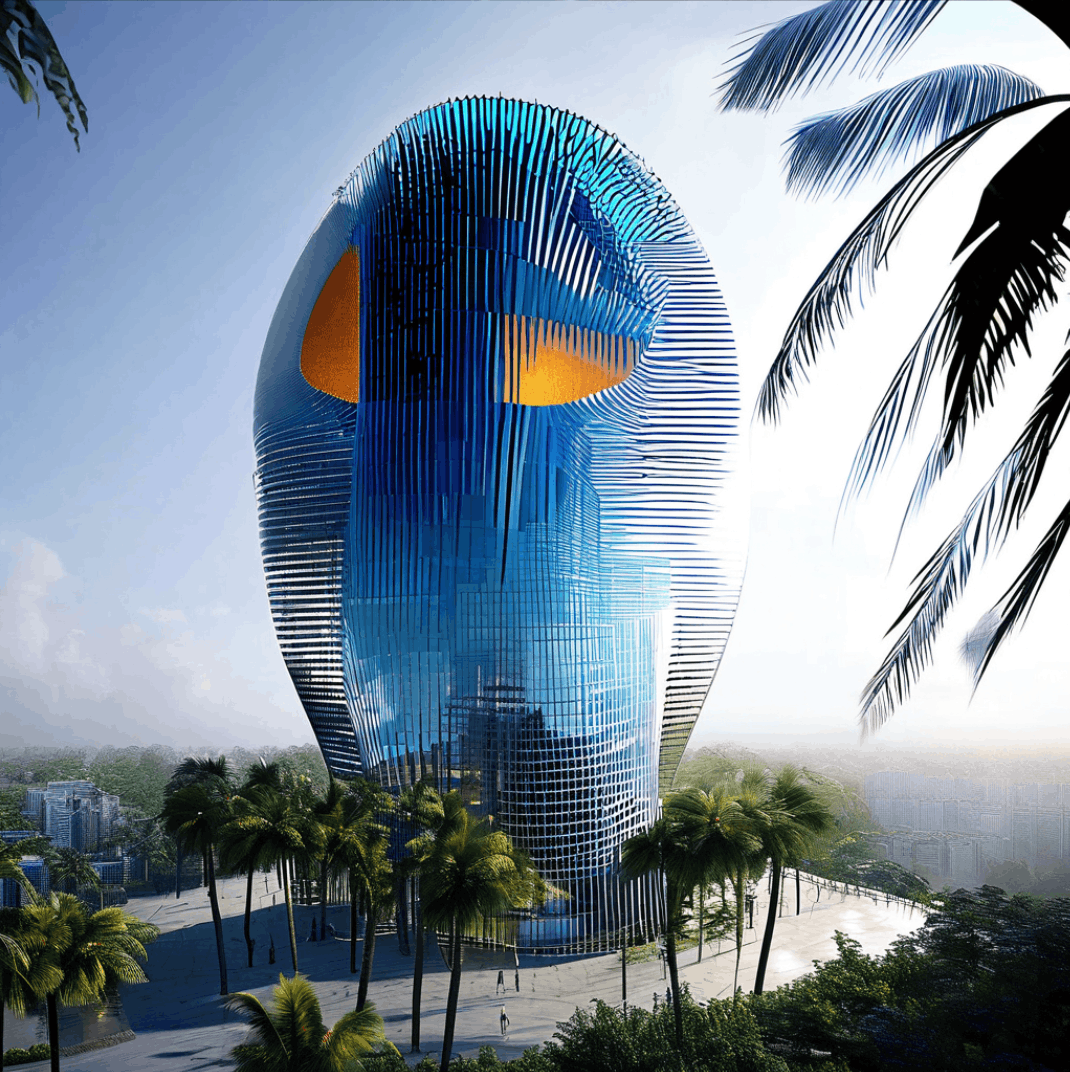}
  {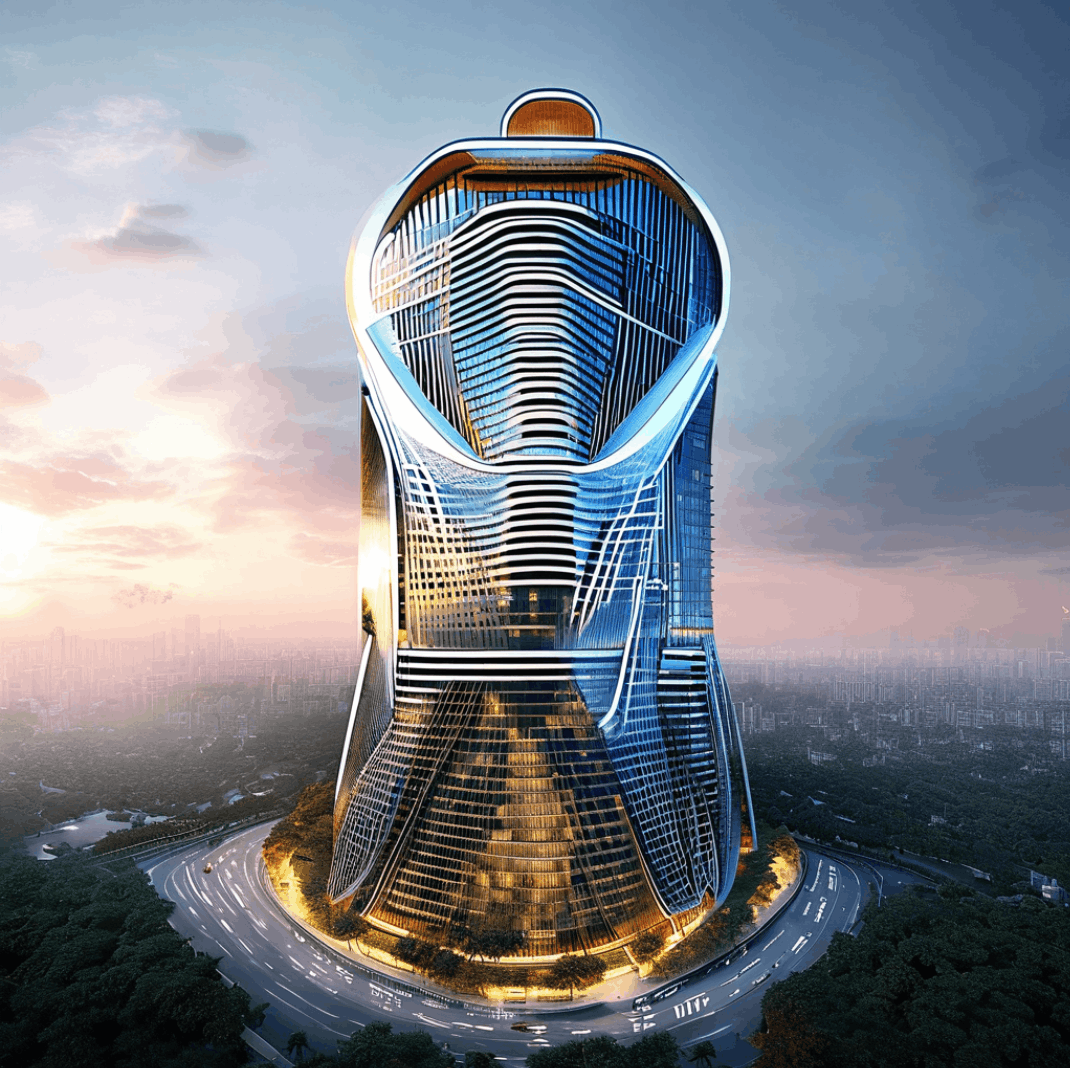}
  {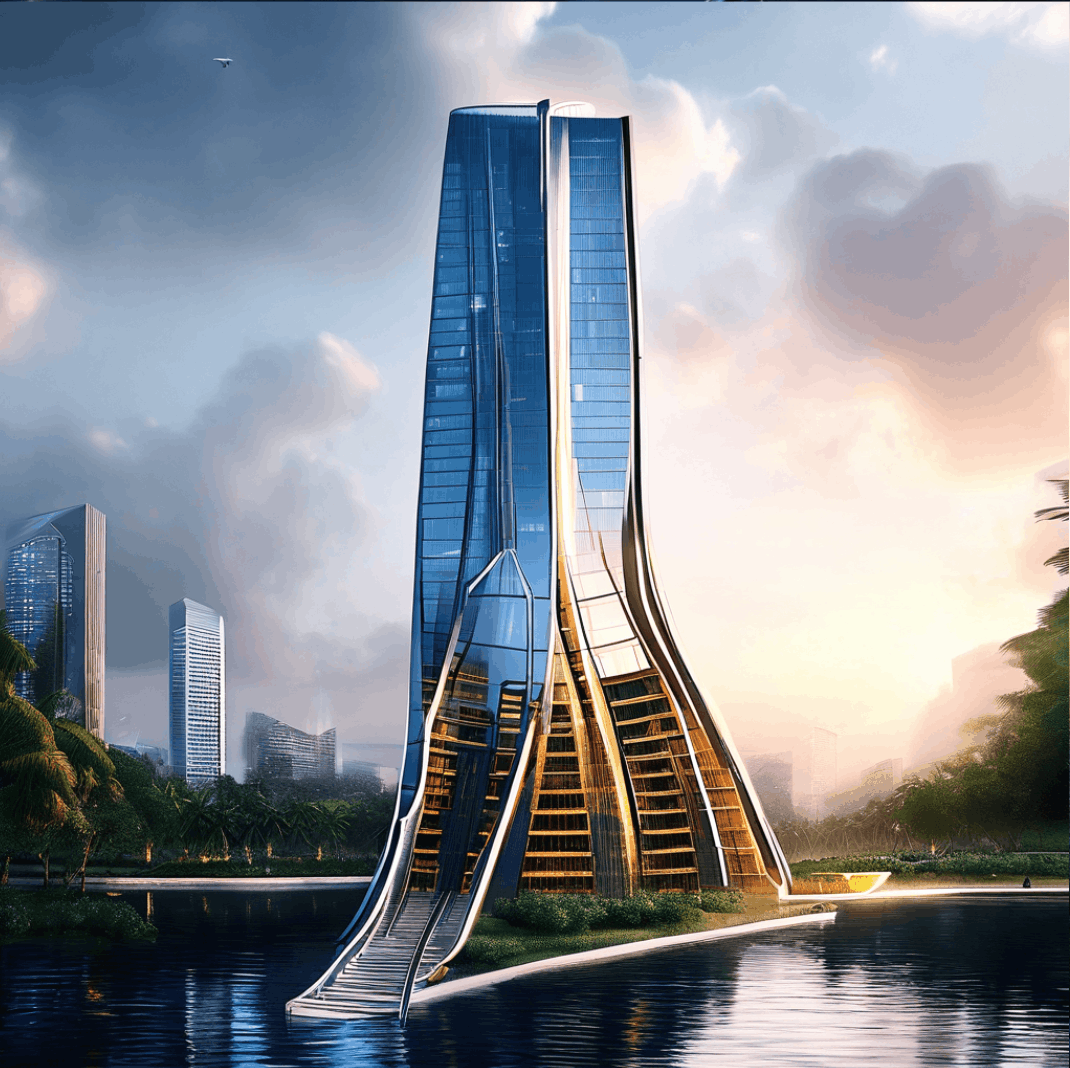}
  {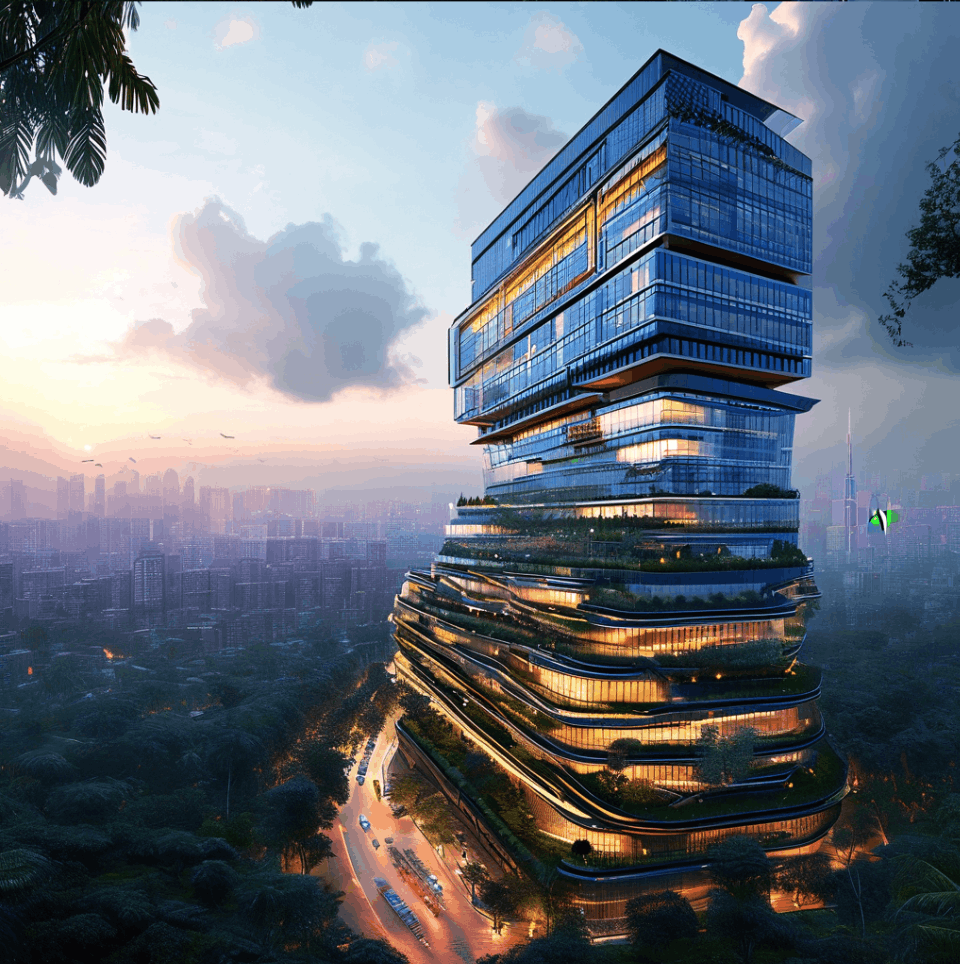}

\BandLabels
  {\ybandD}
  {igaOrange}
  {$\lambda = 30$}
  {$H = \HhatTwelve$}

\draw[
  draw=wallGray!65,
  line width=0.5pt,
  fill=wallGray!15
]
  (\xL,\ywall)
  --
  (7.95,\ywall)
  --
  (7.95,{\ywall+0.33})
  --
  (7.72,{\ywall+0.33})
  --
  (7.72,{\ywall+0.67})
  --
  (7.88,{\ywall+0.67})
  --
  (7.88,{\ywall+1.00})
  --
  (\xL,{\ywall+1.00})
  --
  cycle;

\draw[
  draw=wallGray!65,
  line width=0.5pt,
  fill=wallGray!15
]
  (9.15,\ywall)
  --
  (\xR,\ywall)
  --
  (\xR,{\ywall+1.00})
  --
  (9.22,{\ywall+1.00})
  --
  (9.22,{\ywall+0.67})
  --
  (9.38,{\ywall+0.67})
  --
  (9.38,{\ywall+0.33})
  --
  (9.15,{\ywall+0.33})
  --
  cycle;

\draw[
  wallGray!40,
  line width=0.35pt
]
  (\xL,{\ywall+0.33})
  --
  (7.72,{\ywall+0.33});

\draw[
  wallGray!40,
  line width=0.35pt
]
  (\xL,{\ywall+0.67})
  --
  (7.88,{\ywall+0.67});

\draw[
  wallGray!40,
  line width=0.35pt
]
  (9.38,{\ywall+0.33})
  --
  (\xR,{\ywall+0.33});

\draw[
  wallGray!40,
  line width=0.35pt
]
  (9.38,{\ywall+0.67})
  --
  (\xR,{\ywall+0.67});

\foreach \x in {2.8,3.9,5.0,6.1,7.2}{
  \draw[
    wallGray!40,
    line width=0.35pt
  ]
    (\x,\ywall)
    --
    (\x,{\ywall+0.33});

  \draw[
    wallGray!40,
    line width=0.35pt
  ]
    (\x,{\ywall+0.67})
    --
    (\x,{\ywall+1.00});
}

\foreach \x in {2.25,3.35,4.45,5.55,6.65}{
  \draw[
    wallGray!40,
    line width=0.35pt
  ]
    (\x,{\ywall+0.33})
    --
    (\x,{\ywall+0.67});
}

\foreach \x in {9.9,11.0,12.1,13.2,14.3}{
  \draw[
    wallGray!40,
    line width=0.35pt
  ]
    (\x,\ywall)
    --
    (\x,{\ywall+0.33});

  \draw[
    wallGray!40,
    line width=0.35pt
  ]
    (\x,{\ywall+0.67})
    --
    (\x,{\ywall+1.00});
}

\foreach \x in {10.45,11.55,12.65,13.75,14.85}{
  \draw[
    wallGray!40,
    line width=0.35pt
  ]
    (\x,{\ywall+0.33})
    --
    (\x,{\ywall+0.67});
}

\node[
  fill=wallGray!15,
  inner sep=3pt,
  font=\large\bfseries\boldmath,
  text=wallGray!25!black
]
  at (4.80,{\ywall+0.50})
  {entropy wall};

\node[
  fill=wallGray!15,
  inner sep=3pt,
  font=\large\bfseries\boldmath,
  text=wallGray!25!black
]
  at (12.30,{\ywall+0.50})
  {$\rho_\star = H(P_{\mathrm{data}})$};

\Frag{8.03}{0.02}{0.28}{0.09}{4}
\Frag{8.28}{0.03}{0.20}{0.08}{-6}
\Frag{8.86}{0.03}{0.24}{0.09}{7}

\Frag{8.13}{0.52}{0.18}{0.07}{22}
\Frag{8.97}{0.40}{0.16}{0.07}{-16}

\Frag{8.10}{1.10}{0.20}{0.075}{-18}
\Frag{7.82}{1.16}{0.15}{0.06}{28}
\Frag{9.02}{1.00}{0.15}{0.06}{12}

\draw[
  ->,
  line width=1.0pt,
  igaSteel!85!black
]
  (\xmid,{\ybandA+\BandH+0.10})
  --
  (\xmid,{\ybandB-0.08});

\draw[
  ->,
  line width=1.3pt,
  igaOrange
]
  (\xmid,{\ybandB+\BandH+0.10})
  --
  (\xmid,{\ybandC-0.10});

\draw[
  ->,
  line width=1.0pt,
  igaOrange!85
]
  (\xmid,{\ybandC+\BandH+0.10})
  --
  (\xmid,{\ybandD-0.08});

\node[
  anchor=center,
  font=\Large\bfseries
]
at (8.55,17.9) {
  \textcolor{titleBlue}{Crossing the }
  \textcolor{titlePurple}{Entropy Wall}
  \textcolor{titleOrange}{ into Worlds Beyond Imitation}
};

\end{tikzpicture}
}

\caption{\textbf{A multiverse of diversity levels.}
Matched-seed PixArt-\(\Sigma\) samples at increasing $\lambda$ in the IGA framework ($\lambda=0$ represents the original regularization-free PixArt-\(\Sigma\)), arranged by measured spectral entropy $H$. \emph{Prompt:} ``A skyscraper for a humid coastal city'' The entropy wall $\rho_\star = H(P_{\mathrm{data}})$, i.e., the entropy of the real data distribution, separates data-consistent imitated models from higher-diversity imagined distributions.}
\label{fig:multiverse}

\end{figure}
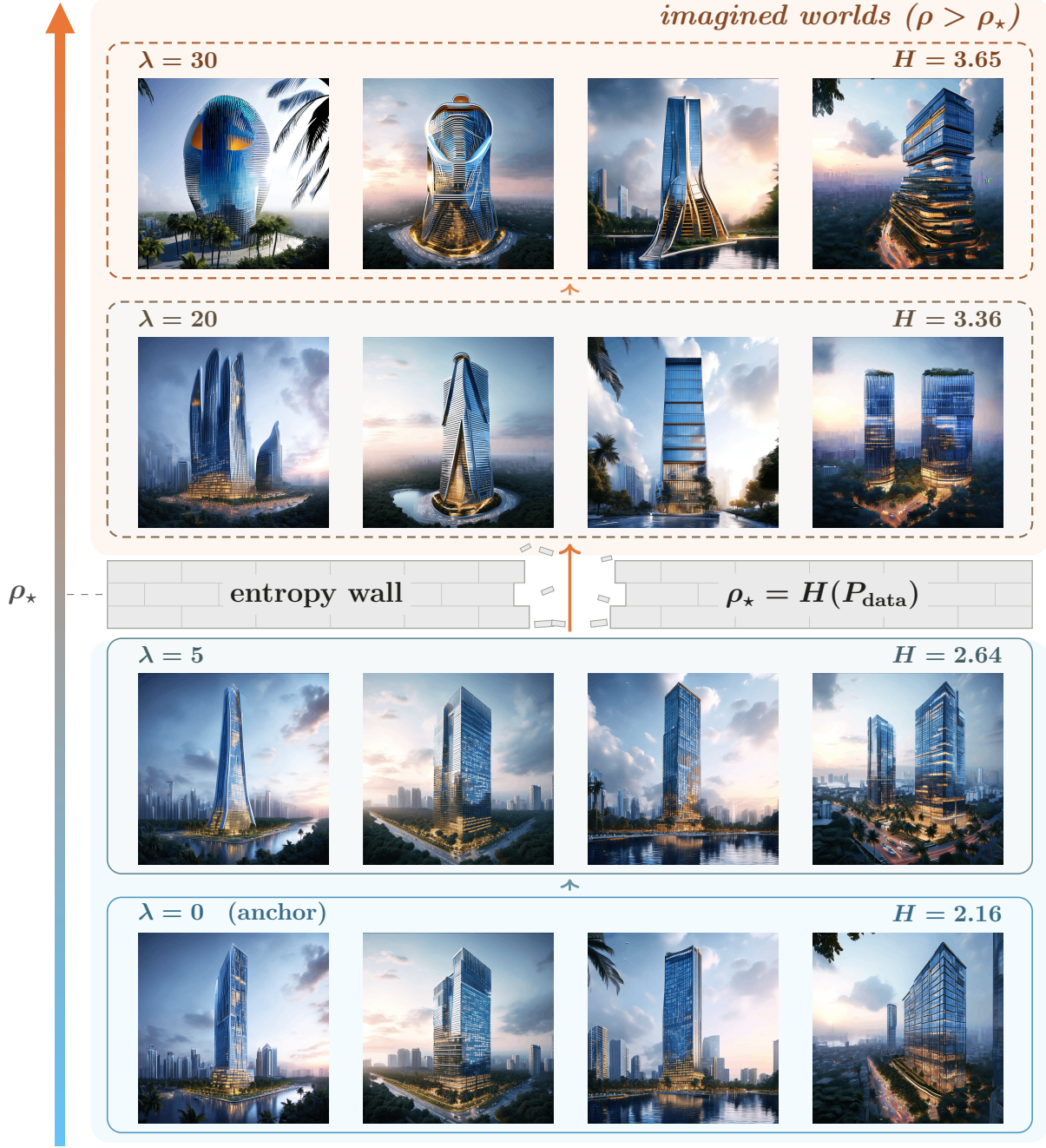

\section{Preliminaries}
\label{sec:preliminaries}

In this section, we introduce the representation-level and distributional quantities used throughout the paper. We associate every distribution with a normalized kernel covariance matrix whose spectrum defines our notion of diversity and introduce a differentiable smoothed surrogate of the resulting entropy together with its per-sample energy. Extended conventions and an elementary Gibbs-tilt identity are deferred to Appendix~\ref{app:prelim}. Throughout this work, $\Pdata$ denotes the population data distribution and $\widehat P_N=\tfrac1N\sum_{i=1}^N\delta_{X_i}$ the empirical measure of $N$ i.i.d.\ samples from it; the upper-case letter $N$ is reserved for generic empirical sample counts (evaluation batches, generated minibatches), and the lower-case letter $n$ for the size of the training set, whose empirical measure we denote with $\widehat P_n$.

\subsection{Representation Space and Kernel Covariance Matrix}
\label{subsec:kernel-covariance}

Consider a measurable representation map $\phi:\cX\to\R^d$ satisfying $\norm{\phi(x)}_2=1$ for every $x\in\cX$, which induces the normalized kernel $k(x,x')=\phi(x)^\top\phi(x')$. For a probability distribution $Q$ on $\cX$, we associate with the representation its \emph{kernel covariance matrix}
\[
    \Sigma_Q
    \defeq
    \E_{X\sim Q}\bigl[\phi(X)\phi(X)^\top\bigr]
    \in\R^{d\times d}.
\]
The unit-norm normalization makes $\Sigma_Q$ a density matrix, i.e., a positive semidefinite matrix with unit trace, whose spectrum records how the representation of $Q$ distributes its mass across orthogonal feature directions. The covariance spectrum carries the notion of diversity developed next.

The same spectrum is accessible from pairwise similarities. Given samples $x_1,\dots,x_N$ with Gram matrix $K=[k(x_i,x_j)]_{i,j=1}^N$, the empirical kernel covariance
\[
    \widehat\Sigma_N
    \defeq
    \frac1N\sum_{i=1}^N\phi(x_i)\phi(x_i)^\top
\]
shares its nonzero eigenvalues with $\tfrac{1}{N}K$. Consequently, every spectral quantity introduced below can be computed from the normalized Gram matrix without explicitly forming the feature vectors.

\subsection{Spectral Entropy Measures and Diversity Scores}
\label{subsec:spectral-diversity}

Following the discussion in \citep{bach2022information,friedman2023vendi,jalali2023information}, we use the following definition for the \emph{von Neumann entropy} of a distribution $Q$ as
\begin{equation}\label{eq:vne-def}
    \Hz(Q)\defeq-\Tr\bigl(\Sigma_Q\log\Sigma_Q\bigr),
\end{equation}
which is the Shannon entropy of the covariance spectrum. We note that the exponential of the above quantity $\mathrm{Vendi}(Q)\defeq\exp(\Hz(Q))$ is the Vendi score~\citep{friedman2023vendi}. This can be interpreted as an effective number of occupied feature directions. For example, a spectrum uniform over $r$ orthogonal directions yields $\Hz(Q)=\log r$ and $\mathrm{Vendi}(Q)=r$.

This notion of diversity is inherently reference-free, yet it depends on the representation model to embed the data: evaluating $\Hz(Q)$ requires no comparison distribution, while the fixed choice of $\phi$, or equivalently $k$, determines which variations count as distinct. Moreover, since $S\mapsto-\Tr(S\log S)$ is concave on density matrices and $Q\mapsto\Sigma_Q$ is affine, $\Hz$ is concave in $Q$ (Lemma~\ref{lem:concave}). At rank-deficient covariance matrices, however, $\Hz$ need not be differentiable.

The guidance analysis of Section~\ref{sec:guidance} requires a well-defined first variation, so we move the covariance uniformly away from the boundary of the density-matrix cone. For $\varepsilon\in(0,1)$, we define the smoothed covariance and entropy
\begin{equation}\label{eq:smoothed-def}
    S_Q^\varepsilon
    \defeq
    (1-\varepsilon)\Sigma_Q+\varepsilon\,\frac{I_d}{d},
    \qquad
    \He(Q)
    \defeq
    -\Tr\bigl(S_Q^\varepsilon\log S_Q^\varepsilon\bigr).
\end{equation}
The spectral floor $S_Q^\varepsilon\succeq \frac{\varepsilon}{d}I_d$ places every eigenvalue in $[\tfrac{\varepsilon}{d},\,1-\varepsilon(1-\tfrac1d)]$ and guarantees that $\He$ is differentiable throughout the feasible covariance set.

The derivative of $\He$ acts on individual samples through one quantity that recurs at every stage of the paper. We define the \emph{entropy energy}
\begin{equation}\label{eq:energy-def}
    \Gen{Q}(x)
    \defeq(1-\varepsilon)\,\phi(x)^\top\bigl(-\log S_Q^\varepsilon\bigr)\phi(x),
\end{equation}
which is the first variation of $\He$ at $Q$ (Lemma~\ref{lem:firstvar}, Section~\ref{sec:guidance}). Since $-\log S_Q^\varepsilon$ has large eigenvalues precisely where $S_Q^\varepsilon$ has small ones, $\Gen{Q}(x)$ is large when $\phi(x)$ aligns with feature directions that $Q$ underrepresents, and the spectral floor gives the uniform bound $0\le\Gen{Q}(x)\le(1-\varepsilon)\log(d/\varepsilon)$. The energy reappears as the payoff of the spectral adversary at training time (Section~\ref{sec:formulation}) and as the exponent of the guidance tilt at sampling time (Section~\ref{sec:guidance}).

The two entropy functionals play distinct roles in our analysis: Section~\ref{sec:wall} uses $\Hz$ to define population spectral diversity, whereas Section~\ref{sec:guidance} uses $\He$ to derive the guidance potential. Lemma~\ref{lem:smoothing} quantifies their uniform proximity as $\varepsilon$ becomes small.

\section{The IGA Framework: From Imitative to Imaginative Generative Modeling}
\label{sec:formulation}

Distributional imitation, which is mathematically formulated in \eqref{eq:intro-imitation}, requires a generative model to match a reference distribution. In the IGA framework which we formulate in this work, we intentionally augment this objective by requiring the generated distribution to \textbf{attain a prescribed level of spectral diversity}, and the resulting formulation separates two questions: (i) what distribution should be targeted, and (ii) how should that target be realized by a generative model? 

In what follows, we first present the distribution-level objective function and the regularized form in IGA. We then derive a min-max reformulation of the entropy reward in the IGA optimization, and then develop the framework's application to sampling-time and training-time settings.

\subsection{Formulating IGA via Constraining and Penalizing the Spectral Entropy}
\label{subsec:objectives}

Consider a reference distribution $\Pref$, which in our applications is either the empirical data distribution $\widehat P_n$ for the observed training data or the underlying distribution $P_\theta$ of the (already trained) generative model. We also consider the spectral entropy functional $H$. Given a divergence measure $\Dref(Q;\Pref)$, IGA selects the most faithful distribution whose spectral diversity reaches a target level $\rho\in\R$:
\begin{equation}
    \begin{aligned}
        \underset{Q\in\cP}{\operatorname{minimize}}
        \quad & \Dref(Q;\Pref)\\
        \text{subject to}
        \quad & H(Q)\ge\rho.
    \end{aligned}
    \tag{C$_\rho$}\label{eq:Crho}
\end{equation}
The divergence objective function anchors the solution to $\Pref$, ensuring the distribution solution remains as close as possible to the reference distribution, while the constraint specifies the desired spectral diversity level. In next section, we review and extend the discussion from \citep{farnia2026exposing} to interpret $\rho$ relative to the entropy of the data, distinguishing diversity repair from spectral extrapolation.

We note that due to the convex structure of the above optimization, the application of standard convex duality shows that \eqref{eq:Crho} is  equivalent to the Lagrangian formulation for a corresponding Lagrangian multiplier parameter $\lambda \ge 0$:
\begin{equation}
    Q_\lambda
    =
    \underset{Q\in\cP}{\arg\!\min}\;    
        F_\lambda(Q)
        \defeq
        \Dref(Q;\Pref)-\lambda H(Q)    \tag{P$_\lambda$}\label{eq:Plambda}
\end{equation}
\begin{remark}
For every constrained solution satisfying the standard regularity conditions, there is a multiplier $\lambda\ge0$ for which the same distribution solves~\eqref{eq:Plambda}. Conversely, each $Q_\lambda$ solves~\eqref{eq:Crho} at its attained diversity level. The precise duality and attainment statements are given in Appendix~\ref{app:formulation}.
\end{remark}

\paragraph{A Min-Max Formulation of the IGA Optimization.} 
The spectral entropy term in \eqref{eq:Plambda} is a nonlinear function of the covariance spectrum, yet it admits a precise convex dual formulation due to its concavity. For the smoothed entropy, the application of the Gibbs variational principle for the matrix-based entropy converts the smoothed IGA problem into a two-player game.

\begin{proposition}[Spectral Min-Max Formulation of IGA Optimization]\label{prop:spectral-minmax}
Consider smoothed spectral entropy with parameter $\varepsilon\in(0,1)$. Define the spectral adversary class as
$$ \Tset\defeq\bigl\{\Theta\in\R^{d\times d}:\ \Theta=\Theta^\top,\ \Tr(\Theta)=0,\ \opnorm\Theta\le\log(d/\varepsilon)\bigr\}. $$
Then, the following equivalences hold:
\begin{enumerate}[leftmargin=*]
    \item[\textup{(i)}] \textbf{Dual representation of the negative entropy.} For every distribution $Q$ on $\cX$, we have
    \begin{equation}\label{eq:vne-dual}
        -\He(Q)
        =\max_{\Theta\in\Tset}
        \Bigl\{-(1-\varepsilon)\,\E_{X\sim Q}\bigl[\phi(X)^\top\Theta\,\phi(X)\bigr]-\log\Tr\bigl(e^{-\Theta}\bigr)\Bigr\},
    \end{equation}
    and the maximum is attained at the unique traceless matrix
    \[
        \Theta^\star(Q)\defeq-\log S_Q^\varepsilon+\tfrac1d\Tr\bigl(\log S_Q^\varepsilon\bigr)I_d\in\Tset.
    \]
    \item[\textup{(ii)}] \textbf{Min-Max form of \eqref{eq:Plambda}.} For every $\lambda\ge0$, distribution $\Pref$, and every $Q\in\cP$, we have
    \begin{equation}\label{eq:IGA-minmax}
        \Dref(Q;\Pref)-\lambda\He(Q)
        =\max_{\Theta\in\Tset}\;
        \mathcal A_\lambda(Q,\Theta),
    \end{equation}
    where $\mathcal A_\lambda(Q,\Theta)\defeq\Dref(Q;\Pref)
        -\lambda(1-\varepsilon)\,\E_{X\sim Q}\bigl[\phi(X)^\top\Theta\,\phi(X)\bigr]
        -\lambda\log\Tr\bigl(e^{-\Theta}\bigr)$.
    
    Therefore, the above results show that \eqref{eq:Plambda} can be rewritten as the following two-player min-max problem (game): $$\min_{Q\in\cP}\max_{\Theta\in\Tset}\mathcal A_\lambda(Q,\Theta)$$
    \item[\textup{(iii)}] \textbf{Min-Max and Max-Min Equivalence.} If $\cP$ is a convex set and $\Dref(\cdot;\Pref)$ is a convex and lower semicontinuous function, the order of minimization and maximization in this game may be interchanged:   $$\min_{Q\in\cP}\,\max_{\Theta\in\Tset}\,\mathcal A_\lambda(Q,\Theta) \: = \: \max_{\Theta\in\Tset}\,\min_{Q\in\cP}\,\mathcal A_\lambda(Q,\Theta)$$
\end{enumerate}
\end{proposition}
\begin{proof}
We defer the proof to the Appendix.
\end{proof}

We note that the best response in part (i) is the centered log-spectrum of the smoothed covariance, and its per-sample payoff coincides, up to an additive constant, with the entropy energy: the spectral adversary pays the generator exactly $\lambda\Gen{Q}$ of \eqref{eq:energy-def}, rewarding samples along directions that $Q$ underrepresents (Lemma~\ref{lem:best-response}). Therefore, the training-time adversarial payoff and the sampling-time guidance field are induced by the same first-order quantity. The proof of Proposition~\ref{prop:spectral-minmax}, given in Appendix~\ref{app:spectral-minmax}, relies on Klein's matrix relative-entropy inequality and the Gibbs variational principle for matrix entropy, both of which are proved there in full. We refer to Remark~\ref{rem:minmax-readings} for further structural implications of \eqref{eq:IGA-minmax}, including the linearization of the entropy reward and the role of smoothing in compactifying the adversary class.

\subsection{IGA for Sampling From a Pretrained Model: the Special case of KL-divergence}
\label{subsec:deployment}

This subsection presents the application of IGA for sampling from an available (supposedly pretrained) model. At sampling time, the pretrained model is held fixed as the reference distribution in the framework. We specifically consider the reference $P_\theta$ as the \emph{underlying} distribution of the pretrained generator, and the Lagrangian IGA optimization problem becomes:
\begin{equation}\label{eq:IGA-sampling-bregman}
    Q^\star
    =
    \underset{Q\in\cP}{\arg\!\min}\;
    \Bigl\{
        D(Q,P_\theta)-\lambda\He(Q)
    \Bigr\}.
\end{equation}
Here, our goal is to sample from the optimal distribution $Q^\star$.

We recall that \eqref{eq:IGA-sampling-bregman} separates the desired target distribution from the algorithm used to sample it: the Bregman divergence determines which departures from the base generator are costly, while the $\He$ Lagrangian penalty rewards higher spectral diversity.

In our analysis, we specifically focus on the geometry resulting from choosing the discrepancy measure to be the \textbf{KL-divergence}. In this specific case, the optimality conditions yield an explicit density-ratio characterization of the optimal solution. Note that, in the case of KL-divergence, the sampling-based IGA aims to minimize the following objective function:
\begin{equation}\label{eq:IGA-sampling-kl}
    F(Q)
    \defeq
    \KL(Q \| P_\theta)
    -\lambda\He(Q)
\end{equation}
over $Q\in\cP$, where the extended-value convention for the KL term sets $F(Q)=+\infty$ off the set $\{Q\ll P_\theta\}$.

\begin{proposition}[Sampling-time target as a self-consistent exponential tilt]\label{prop:kl-tilt}
Let $\lambda\ge0$ and $\varepsilon\in(0,1)$. If $F$ attains a finite minimum over $\{Q\in\cP:Q\ll P_\theta\}$, then the minimizer $Q^\star$ is unique, $Q^\star$ and $P_\theta$ are mutually absolutely continuous, and, with the \emph{total reward} defined by
\begin{equation}\label{eq:energy-reward}
    R_{Q}(x)
    \defeq
    \lambda\,\Gen{Q}(x),
\end{equation}
the density ratio is the exponential tilt
\begin{equation}\label{eq:kl-tilt}
    \frac{dQ^\star}{dP_\theta}(x)
    =
    \frac{\exp\bigl(R_{Q^\star}(x)\bigr)}
         {\E_{X\sim P_\theta}\bigl[\exp\bigl(R_{Q^\star}(X)\bigr)\bigr]}
\end{equation}
\end{proposition}
\begin{proof}
We defer the proof to the Appendix.
\end{proof}

\begin{corollary}[Score-function relation under the exponential tilt]
\label{cor:score-tilt}
Under the assumptions of Proposition~\ref{prop:kl-tilt}, suppose that
$P_\theta$ and $Q^\star$ admit differentiable densities
$p_\theta$ and $q^\star$, respectively. Then their score functions satisfy
\begin{equation}
\label{eq:score-gen}
    \nabla \log q^\star(x)
    =
    \nabla \log p_\theta(x)
    +
    \lambda\,\nabla \Gen{Q^\star}(x).
\end{equation}
\end{corollary}
\begin{proof}
The result follows directly by taking the logarithm of~\eqref{eq:kl-tilt} and differentiating with respect to $x$, noting that the log-normalizing constant is independent of $x$.
\end{proof}

The target is an exponential reweighting of the pretrained law. We highlight that the reweighting is self-consistent rather than externally prescribed, since the total reward $R_{Q^\star}$ depends on the covariance of the unknown target itself. The multiplier $\lambda$ sets the strength of the reweighting, and the uniform bound on the energy noted after~\eqref{eq:energy-def} limits how strongly any single sample can be up- or down-weighted, while the KL term confines the redistribution of mass to the support of the base law.

We note that Proposition~\ref{prop:kl-tilt} characterizes the target distribution over clean outputs and does not yet provide a sampler; also, the finite-minimum hypothesis remains to be verified. Section~\ref{sec:guidance} addresses both points: under mild topological conditions the minimizer exists (Theorem~\ref{thm:tilt}), the tilt propagates exactly through the forward noising process to an explicit time-dependent guidance field (Theorem~\ref{thm:twist}), and the field is approximated with denoised predictions at a quantified endpoint error, all without changing the pretrained score network. Unlike guidance by a fixed sample-wise reward, the tilt depends on the target law itself, through its covariance; practical sampling therefore estimates this distribution-level quantity, for example from a pilot batch. We keep the sampling target $Q^\star$ notationally distinct from the training-time optimum $Q_\lambda^{\mathrm{train}}$ of Section~\ref{subsec:training}, realized by changing generator parameters.

\subsection{IGA for Training Generative Models with Entropy-Regularized Objective}
\label{subsec:training}

For the training-time application of the imaginative generative modeling in IGA, we change the original divergence minimization in standard generative modeling and include the additional Lagrangian term in the objective function $-\lambda H(Q)$ to promote higher spectral entropy in the trained model.

Mathematically, we choose the reference distribution to be the empirical distribution $\widehat{P}_n$ of $n$ training samples $x_1,\ldots , x_n$ (i.e., $\widehat{P}_n= \frac{1}{n}\sum_{i=1}^n \delta_{x_i}$). Then, the optimization problem for training-time IGA will be computing the optimal solution to the spectral entropy-regularized divergence minimization problem:
\begin{equation}\label{eq:IGA-training-population}
    Q_\lambda^{\mathrm{train}}
    \, =\, 
    \underset{Q\in\Pgen}{\arg\!\min}\;
    \Bigl\{
        \Dref(Q;\widehat P_n)-\lambda H(Q)
    \Bigr\}.
\end{equation}
The first term specifies how divergence to the training data distribution is measured, while the second is a distribution-level regularizer: it acts jointly on generated examples and rewards coverage of feature directions that would otherwise be underrepresented. IGA therefore only augments the model's distributional discrepancy objective and can be interpreted as a spectral entropy regularization in the divergence minimization task of training the generative model.  Specifically, in the following, we focus on and apply the IGA training framework to the adversarial training of generative adversarial networks (GANs). Further discussion on application of training-time IGA to other generative modeling frameworks is deferred to the Appendix.

\paragraph{Adversarial training and GANs.}
The min-max format appearing in Proposition~\ref{prop:spectral-minmax} composes smoothly with objective functions that are formed in the adversarial-learning formulations of generative modeling. We note that the standard GAN \cite{goodfellow2014generative,nowozin2016fgan,arjovsky2017wasserstein} objectives measure the discrepancy through a critic (discriminator) class $\cD$ and real-valued link functions $u,v: \mathbb{R}\rightarrow \mathbb{R}$:
\begin{equation}\label{eq:critic-fidelity}
    \Dref(Q;\widehat P_n)
    =\max_{D\in\cD}
    \Bigl\{\E_{X\sim\widehat P_n}\bigl[u\bigl(D(X)\bigr)\bigr]
    -\E_{X\sim Q}\bigl[v\bigl(D(X)\bigr)\bigr]\Bigr\}.
\end{equation}
As notable examples, Wasserstein GANs take a $1$-Lipschitz critic class with $u=v=\mathrm{id}$ being the identity map ~\citep{arjovsky2017wasserstein}; $f$-GANs choose $v=f^*$ for the convex conjugate of the convex $f$ function underlying the target $f$-divergence, recovering the original GAN objective as a special case for JS-divergence~\citep{goodfellow2014generative,nowozin2016fgan}. Substituting \eqref{eq:critic-fidelity} and the entropy dual \eqref{eq:vne-dual} into \eqref{eq:IGA-training-population} gives an exact reformulation in which the entropy reward joins the discriminator inside a single adversary.

\begin{proposition}[IGA-GAN formulation as min-max optimization]\label{prop:IGA-gan}
Let $\Dref(\cdot;\widehat P_n)$ admit the representation \eqref{eq:critic-fidelity}, and let $\lambda\ge0$, $\varepsilon\in(0,1)$. Then, for every class $\cQ$ of distributions, we have the following
\begin{equation}\label{eq:IGA-gan}
\begin{aligned}    \min_{Q\in\cQ}\bigl\{\Dref(Q;\widehat P_n)-\lambda\He(Q)\bigr\}   &\:=\:\min_{Q\in\cQ}\;\max_{(D,\Theta)\in\cD\times\Tset}\;
    \Bigl\{\E_{X\sim\widehat P_n}\bigl[u\bigl(D(X)\bigr)\bigr]
    -\E_{X\sim Q}\bigl[v\bigl(D(X)\bigr)\bigr]\\
    &\qquad
    -\lambda(1-\varepsilon)\,\E_{X\sim Q}\bigl[\phi(X)^\top\Theta\,\phi(X)\bigr]
    -\lambda\log\Tr\bigl(e^{-\Theta}\bigr)\Bigr\},
\end{aligned}
\end{equation}
Note that the above has a single maximization over the joint adversary $(D,\Theta)$. For every fixed $Q$ the joint maximization decouples across the two components, and the $\Theta$-component is attained at the best response $\Theta^\star(Q)$ of Proposition~\ref{prop:spectral-minmax}.
\end{proposition}
\begin{proof}
We defer the proof to the Appendix.
\end{proof}

The proof, given in Appendix~\ref{app:training}, relies on a structural observation: the critic and the spectral adversary enter through suprema over independent variables, and such suprema combine additively. Hence, the identity holds pointwise in $Q$, and neither convexity of $\Pgen$ nor a minimax interchange is used; the interchange remains reserved for the ambient class (Remark~\ref{conv:ambient}, Appendix~\ref{app:prelim}). Algorithmically, \eqref{eq:IGA-gan} adds one adversary to standard GAN training, and this additional adversary is computationally inexpensive: while the critic $D$ is trained by gradient steps, the spectral player requires no training at all, since its best response is the centered log-spectrum $\Theta^\star(Q)$ and can be computed from an eigendecomposition of the minibatch covariance at $O(d^3)$ cost per refresh.

The two adversaries play complementary roles. The critic enforces the fidelity of individual samples by comparing generated examples against data, whereas $\Theta$ acts on the generated distribution as a whole and pays the generator the spectral novelty reward $\lambda\Gen{Q}(x)$, up to a sample-independent constant, for occupying directions that the current generated law neglects (Lemma~\ref{lem:best-response}). We also highlight that freezing $\Theta$ at its best response is not a heuristic: the gradients of the generator parameters through the frozen payoff are \emph{exactly} the gradients of $\lambda\He(Q_\vartheta)$ (Proposition~\ref{prop:entropy-grad}, Appendix~\ref{app:training}). This envelope-type identity removes the need to differentiate through the eigendecomposition.

\section{The Entropy Wall: Spectral Entropy of Real Data as the Boundary between Imitation and Imagination}
\label{sec:wall}

\begin{figure*}[h]
    \centering
    \begin{minipage}[t]{0.485\textwidth}
        \centering
        \includegraphics[width=\linewidth]
            {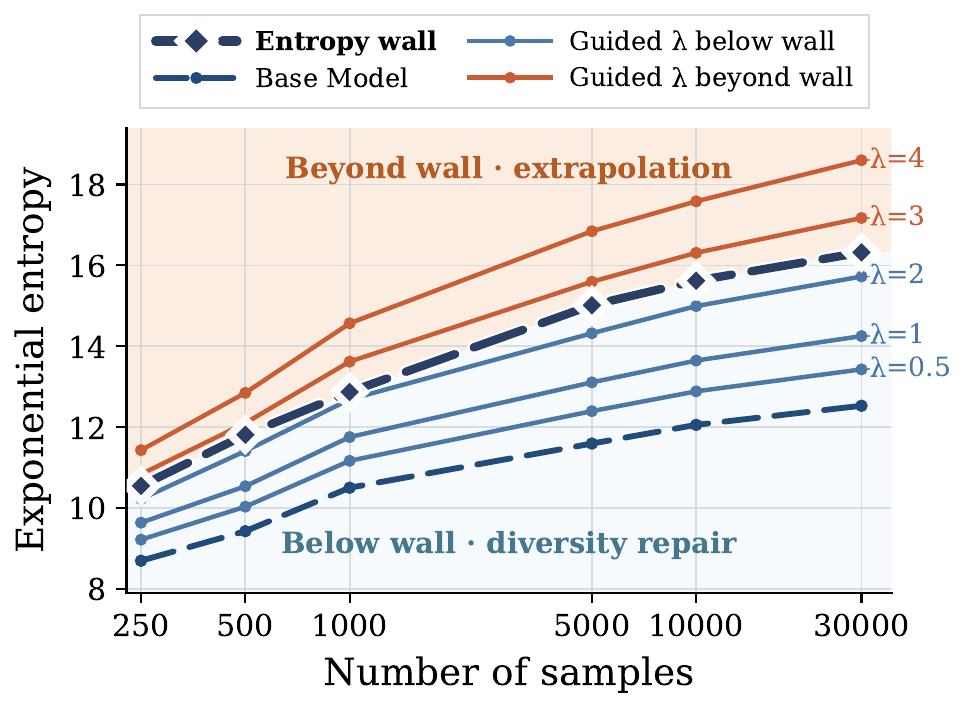}\\[-1mm]
        \small (a) CelebA-HQ
    \end{minipage}
    \hfill
    \begin{minipage}[t]{0.485\textwidth}
        \centering
        \includegraphics[width=\linewidth]
            {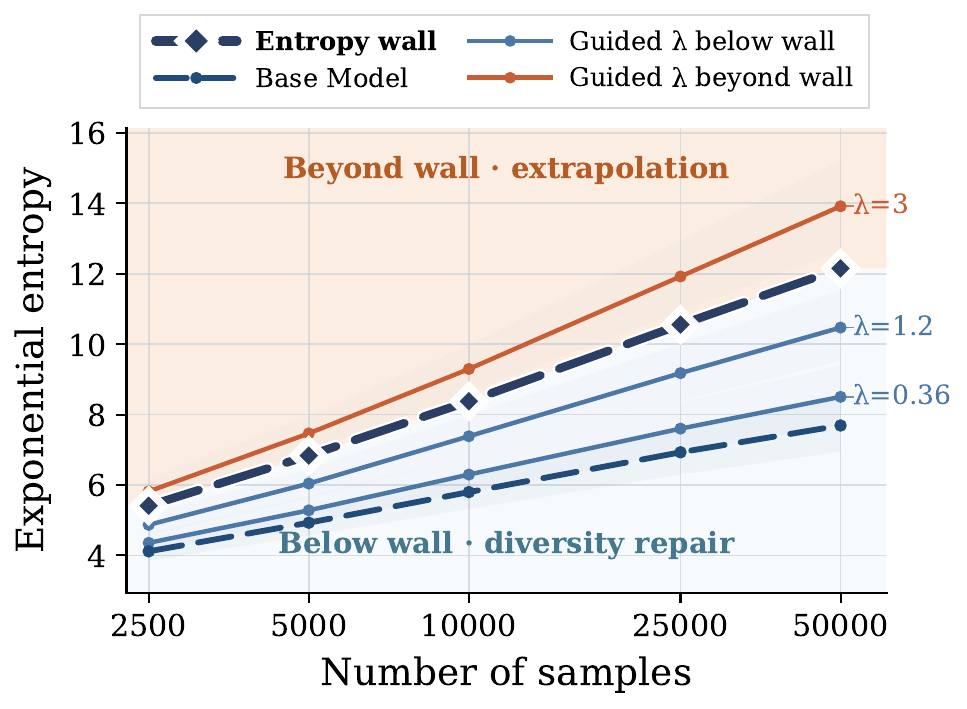}\\[-1mm]
        \small (b) ImageNet
    \end{minipage}
    \caption{
        \textbf{The Entropy walls on CelebA-HQ and ImageNet.}
        Real and generated entropy estimates use matched sample sizes.
        Both base models remain below the data wall; increasing
        $\text{IGA }\lambda$ closes the deficit and eventually crosses into the
        imagination regime.
    }
    \label{fig:vendi-sample-size}
\end{figure*}

As we discussed in the introduction, the entropy wall is the spectral diversity of the data itself, measured in the chosen representation. We note that the concept of the entropy wall is implied by the discussion in \citep{farnia2026exposing}, in which the authors reveal the spectral entropy gap between the standard generative models and their target underlying data distributions. In this section, we formalize the concept and propose the term \textit{"Entropy Wall"} to highlight the spectral entropy level of the underlying real data distribution.

Particularly, we highlight that the definition of entropy wall separates two qualitatively different uses of the IGA regularization framework: as long as the required spectral diversity level in IGA stays at or below what the data exhibits, increasing diversity can be read as \emph{repairing} a deficiency of the learned generator; once the request exceeds it, the data distribution is no longer feasible for~\eqref{eq:Crho}, and pushing further is deliberate \emph{extrapolation} beyond the data. Like every diversity statement in this paper, the wall's location depends on the fixed pair $(\phi,k)$ and is therefore representation-relative.

\begin{definition}[Entropy Wall]\label{def:wall}
For spectral entropy function $H$, the entropy wall is
\[
    \rho_{\star,H}\defeq H(\Pdata)
\]
Based on the above definition, a distribution $Q$ is below, on, or beyond the entropy wall according as $H(Q)$ is $<$, $=$, or $>\rho_{\star,H}$.
\end{definition}

Reading the regularization path of Theorem~\ref{thm:mono} through the wall gives it the statistical interpretation promised in the introduction. Note that we use the notation $Q_\lambda$ for the optimal solution to the problem with Lagrangian coefficient $\lambda$.

\paragraph{Below the wall: diversity repair.}
When $H(Q_\lambda)\le H(\Pdata)$, increasing $\lambda$ moves $Q_\lambda$ toward higher diversity, and this movement is provably safe for a base-anchored, correctly oriented Bregman objective: every below-wall point on the path is no farther from $\Pdata$, in the anchoring divergence, than the base model is (Theorem~\ref{thm:repair}, Appendix~\ref{app:wall}). This is the precise sense in which the sub-wall path performs \emph{repair}, counteracting the spectral contraction reported in modern generators~\citep{farnia2026exposing}. It does not identify $Q_\lambda$ with $\Pdata$, nor does it guarantee that every induced semantic change recovers a genuine data mode.

\paragraph{Beyond the wall: spectral extrapolation.}
When $H(Q_\lambda)>H(\Pdata)$, the same monotone increase in $\lambda$ means something different: $Q_\lambda$ is no longer estimating $\Pdata$ but performing representation-relative extrapolation, spreading its mass across directions of the representation more broadly than the data does. The transition is a change of statistical interpretation, not a geometric barrier: the wall can be crossed at arbitrarily small discrepancy whenever a higher-entropy direction exists in $\cP$ and the discrepancy is continuous along the mixture path toward it (Proposition~\ref{prop:crossing}, Appendix~\ref{app:wall}). Beyond the wall we claim no improved estimation of the data distribution; the regime is evaluated as controlled, representation-relative extrapolation.

\section{IGA Guidance for Score-based and Diffusion Models}
\label{sec:guidance}

\begin{figure}
    \centering
    \includegraphics[width=0.99\linewidth]{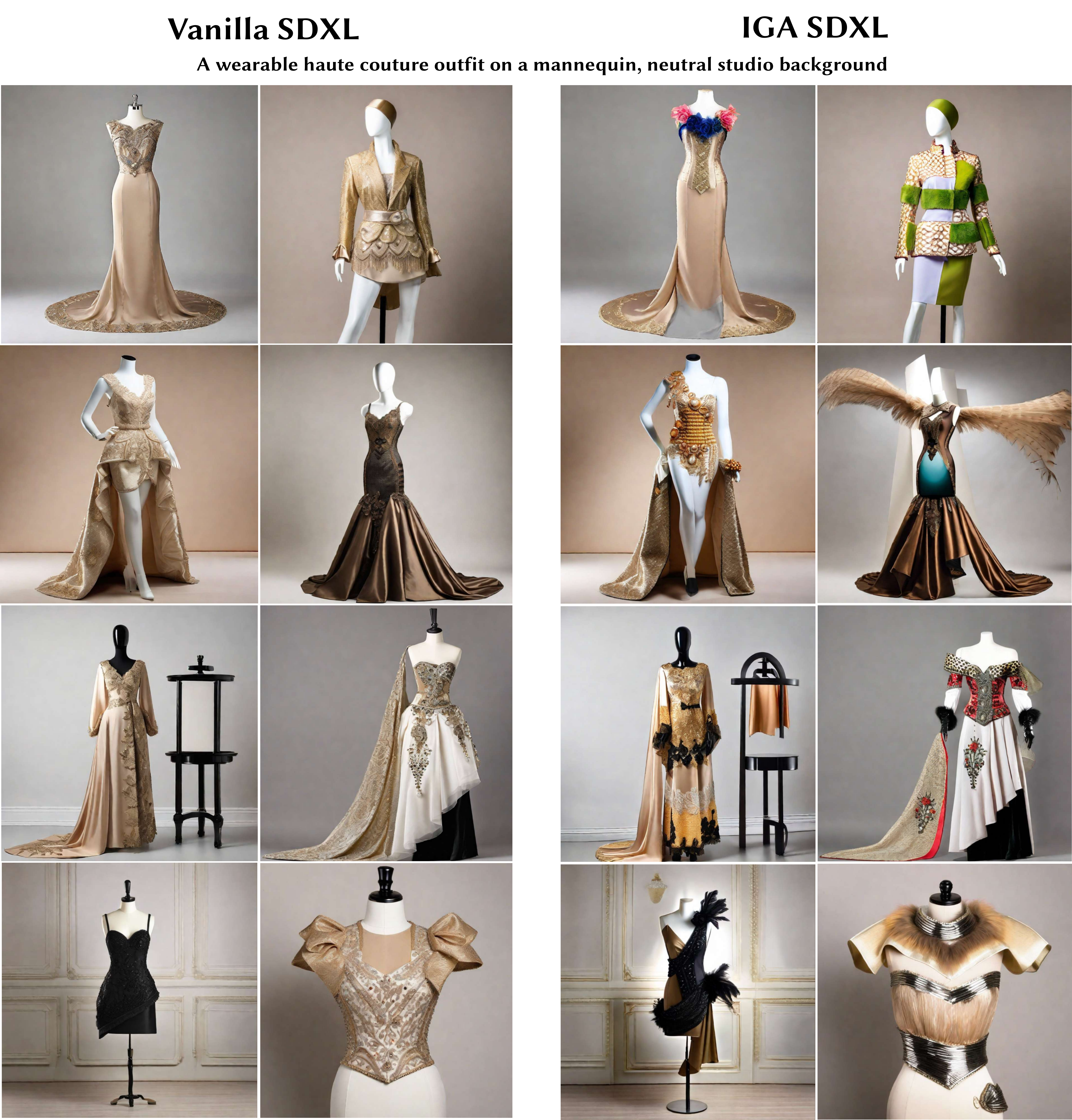}
    \caption{
        \textbf{Fashion design with vanilla and IGA SDXL.}
        Left: vanilla SDXL. Right: IGA SDXL. Corresponding cells use the same
        initial noise seed. Vanilla samples cluster around beige and gold
        eveningwear with familiar gown and tailored shapes. IGA adds bright
        color blocking, asymmetric cuts, mixed materials, and large sculptural
        or feathered elements.
    }
    \label{fig:IGA-sdxl-fashion}
\end{figure}

\begin{figure}
    \centering
    \includegraphics[width=0.99\linewidth]{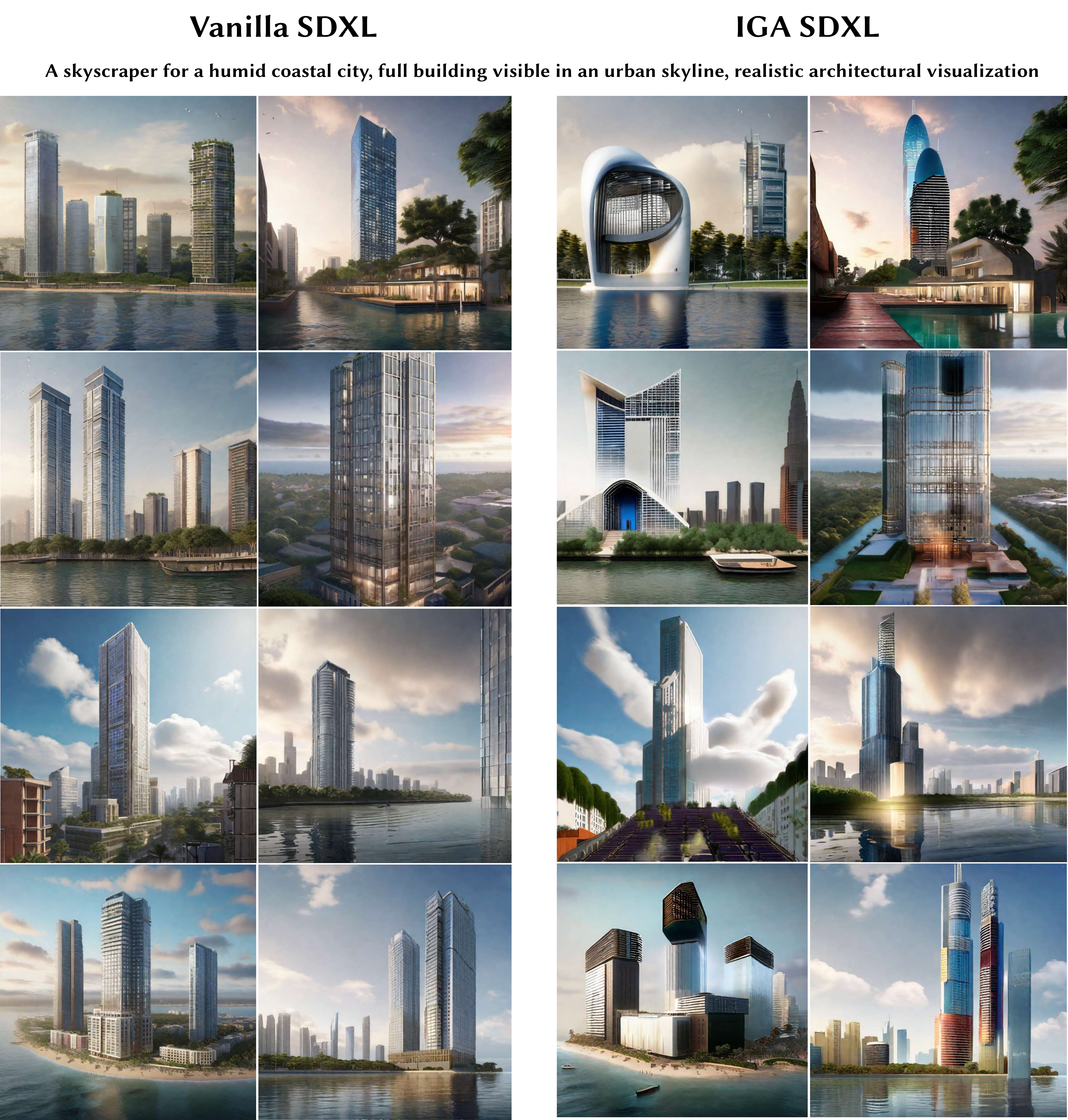}
    \caption{
        \textbf{Architectural design with vanilla and IGA SDXL.}
        Left: vanilla SDXL. Right: IGA SDXL. Corresponding cells use the same
        initial noise seed. Vanilla SDXL mainly produces straight glass towers
        with similar overall forms. IGA introduces curved shells, open frames,
        split tops, stacked blocks, and larger changes in color and proportion,
        while preserving a clear full-building view.
    }
    \label{fig:IGA-sdxl-sky}
\end{figure}

\begin{figure}
    \centering
    \includegraphics[width=0.99\linewidth]
        {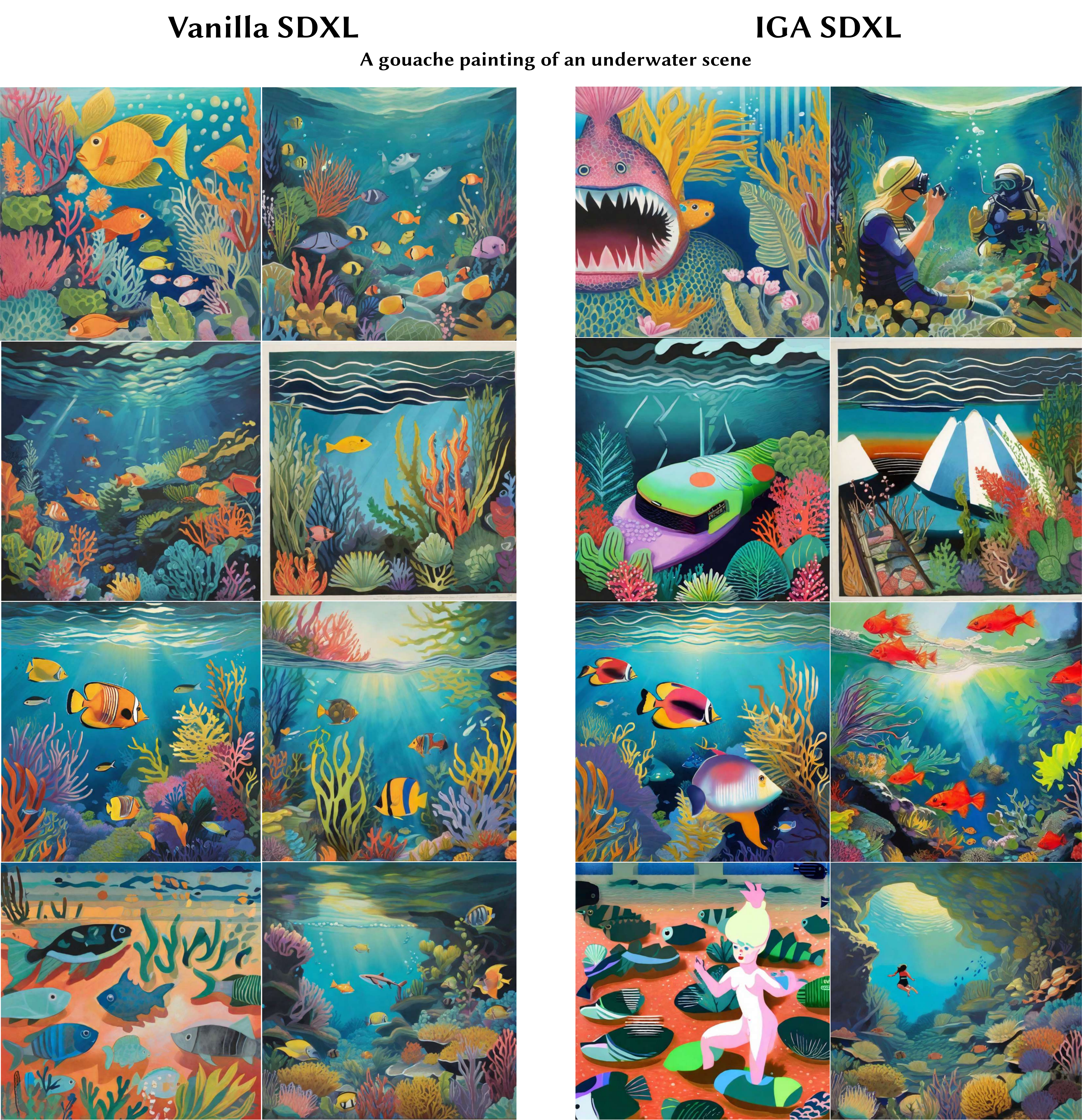}
    \caption{
        \textbf{Stylized underwater scenes with vanilla and IGA SDXL.}
        Left: vanilla SDXL. Right: IGA SDXL. Corresponding cells use the same
        initial noise seed. Vanilla SDXL mostly depicts coral reefs and schools
        of fish. IGA expands the scene content to divers, large creatures,
        vehicles, built structures, and cave-like spaces, while keeping the
        gouache rendering style.
    }
    \label{fig:IGA-sdxl-underwater}
\end{figure}

\begin{figure}
    \centering
    \includegraphics[width=0.99\linewidth]{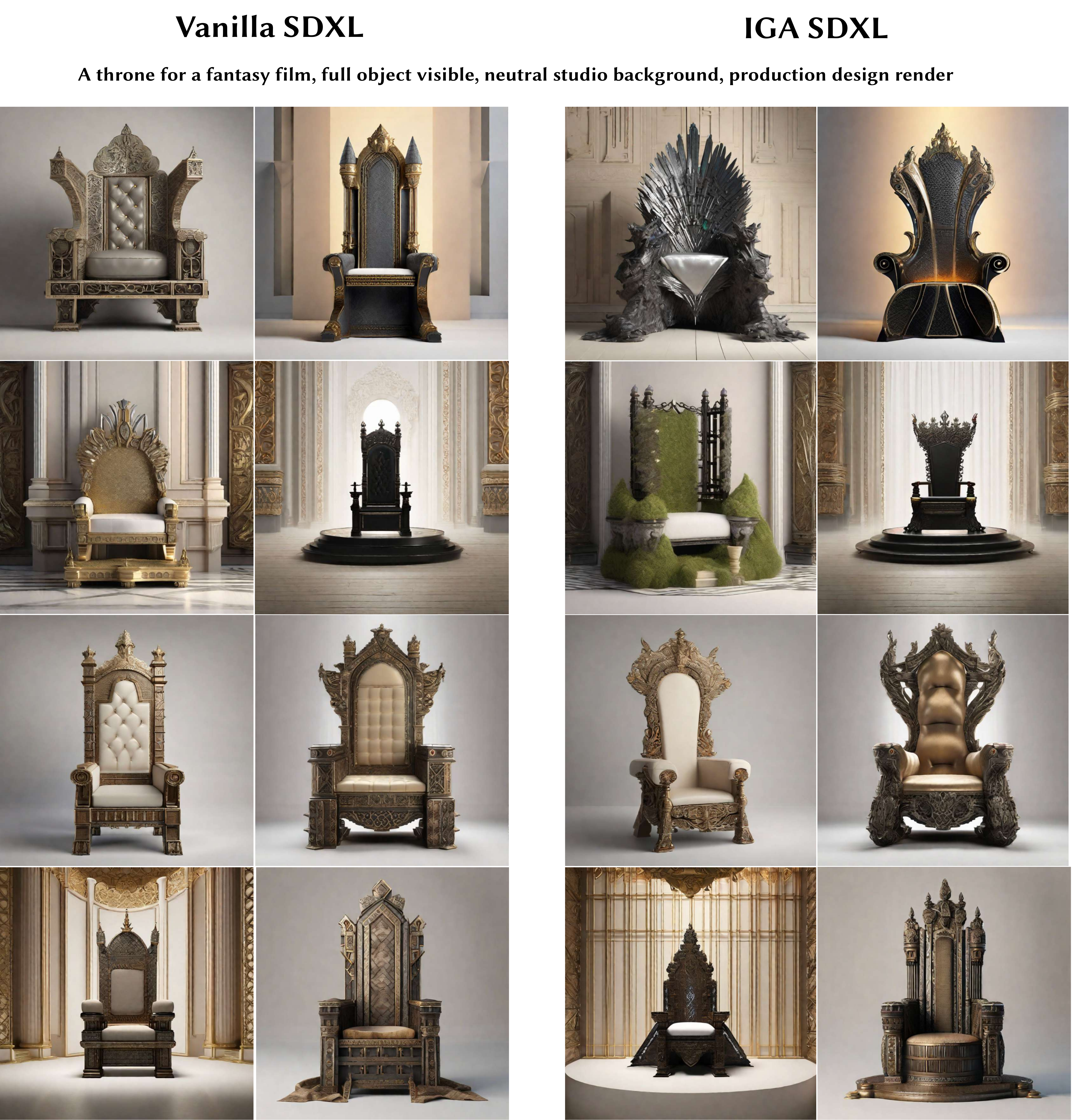}
    \caption{
        \textbf{Fantasy throne design with vanilla and IGA SDXL.}
        Left: vanilla SDXL. Right: IGA SDXL. Corresponding cells use the same
        initial noise seed. Vanilla SDXL mostly returns ornate high-backed
        chairs with similar carved frames. IGA introduces spiked metal forms,
        curved black shells, moss-covered structures, and larger changes in the
        seat and back, while keeping the throne centered and fully visible.
    }
    \label{fig:IGA-sdxl-throne}
\end{figure}

\begin{figure}
    \centering
    \includegraphics[width=0.99\linewidth]{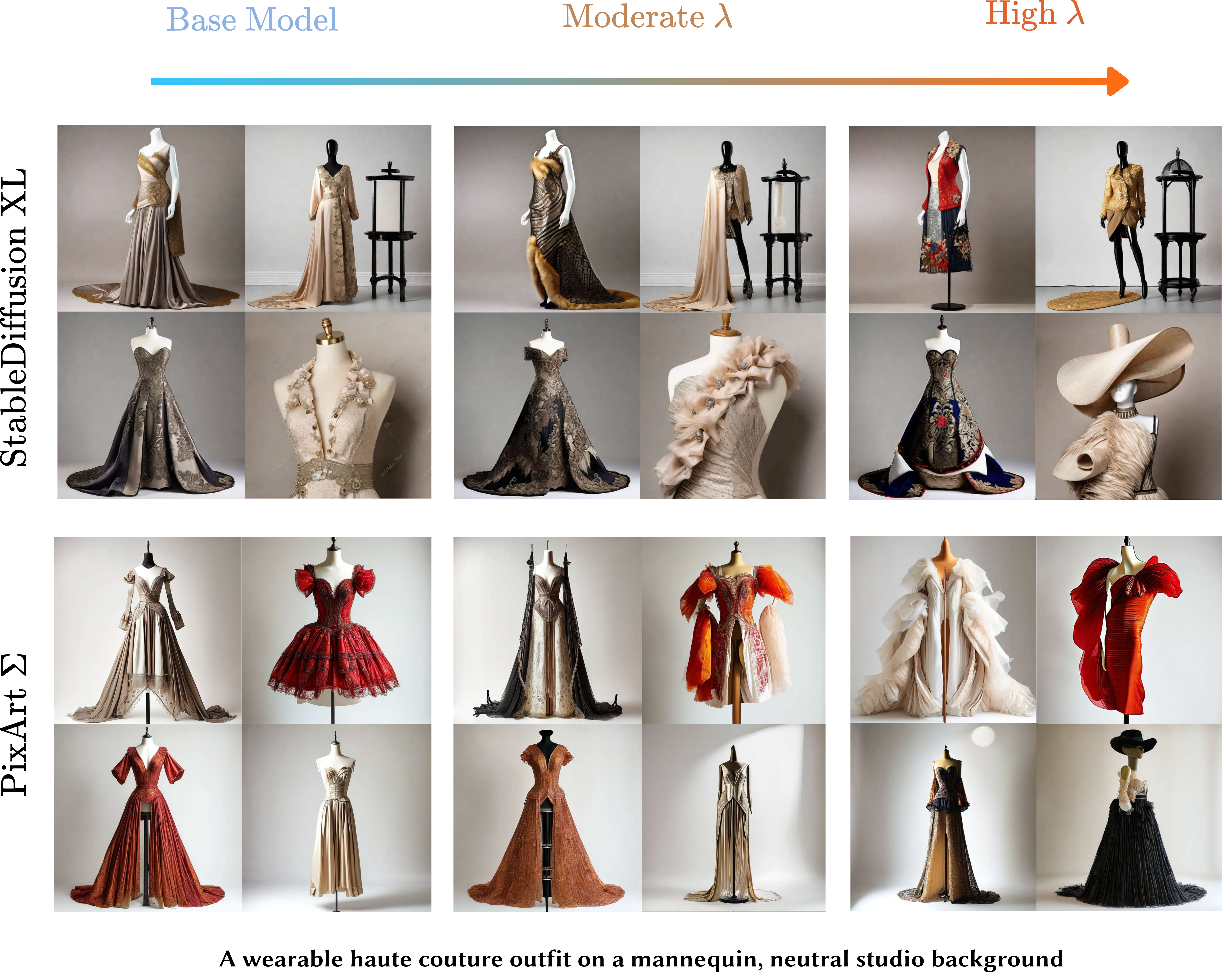}
    \caption{
        \textbf{IGA $\lambda$-sweep across SDXL and PixArt-$\Sigma$.}
        Top: SDXL with $\lambda\in\{0,4,8\}$.
        Bottom: PixArt-$\Sigma$ with $\lambda\in\{0,10,20\}$.
        Within each model, corresponding positions use the same initial noise
        seed. As $\lambda$ increases, both models move from familiar dress
        shapes toward stronger asymmetry, larger volumes and accessories, and
        wider material and color choices, while staying consistent with the
        prompt.
    }
    \label{fig:IGA-text-model-sweep}
\end{figure}

Here, we focus on diffusion models and apply sampling-time IGA to pre-trained score-based and diffusion models. The developments in this section largely build on the score-function characterization established in Corollary~\ref{cor:score-tilt}. We first establish existence and uniqueness of the tilted target $Q^\star$, then propagate the tilt through the forward noising process, and finally derive retraining-free approximations for score-based, DDPM, and DDIM samplers.

\subsection{The target distribution and exact guidance}
\label{subsec:target}

\begin{assumption}[Sampling-time setup]\label{ass:guidance}
Given the sample space $\cX$, the representation map $\phi$ is continuous, hence bounded by the unit-norm normalization. $P_\theta$ is the original distribution of the pretrained generator. Also, as stated in previous sections, we suppose the hyperparameters satisfy $\lambda\ge0$, $\varepsilon\in(0,1)$.
\end{assumption}

Under Assumption~\ref{ass:guidance}, we minimize the functional $F$ of~\eqref{eq:IGA-sampling-kl} over $\cP$. The KL anchor makes the problem tractable for two reasons: the effective domain $\{Q\ll P_\theta\}$ is convex, and the first variation of $\KL(\cdot\| P_\theta)$ along a density perturbation is simply the log-density ratio. The main difficulty is that $\He$ is highly nonlinear in $\Sigma_Q$: although $\Sigma_Q$ is affine in $Q$, the matrix logarithm couples all eigenvalues. The following lemma resolves this difficulty and justifies the role assigned to the entropy energy at its definition~\eqref{eq:energy-def}; it is the analytical core of Proposition~\ref{prop:kl-tilt} and of the existence theorem below.

\begin{lemma}[First variation of the smoothed entropy]\label{lem:firstvar}
Let $\phi$ be bounded and measurable with $\norm{\phi(x)}_2=1$, and let $\Gen{Q}$ be the entropy energy~\eqref{eq:energy-def}. Then for every distribution $Q$ and every finite signed measure $\nu$ with $\nu(\cX)=0$ such that $Q+t\nu$ is a probability measure for all sufficiently small $t>0$,
\[
    \frac{d}{dt}\,\He(Q+t\nu)\Big|_{t=0^+}
    =\int \Gen{Q}(x)\,d\nu(x).
\]
\end{lemma}

The lemma therefore identifies the total reward $R_Q$ of~\eqref{eq:energy-reward} as the first variation of the reward part $\lambda\He(Q)$ of the objective (proof and further discussion in Appendix~\ref{app:guidance}).

\begin{theorem}\label{thm:tilt}
Under Assumption~\ref{ass:guidance}, the following hold for the functional $F$ of~\eqref{eq:IGA-sampling-kl}.
\begin{enumerate}[leftmargin=*]
    \item[\textup{(i)}] \textbf{Existence and uniqueness.} $F$ has a unique minimizer $Q^\star$, and $Q^\star$ is mutually absolutely continuous with $P_\theta$.
    \item[\textup{(ii)}] \textbf{Self-consistent exponential tilt.} $Q^\star$ satisfies the tilt characterization~\eqref{eq:kl-tilt}.
    \item[\textup{(iii)}] \textbf{Boundedness.} $R_{Q^\star}$ is uniformly bounded; consequently $dQ^\star/dP_\theta$ is bounded above and below by positive constants, and $Q^\star$ has the same $P_\theta$-essential support as $P_\theta$.
    \end{enumerate}
\end{theorem}
\begin{proof}
We defer the proof to the Appendix.
\end{proof}

\eqref{eq:kl-tilt} is exactly the tilt announced in~\eqref{eq:intro-exponential-tilt}, with $G_{Q^\star}=\Gen{Q^\star}$. The theorem verifies the finite-minimum hypothesis of Proposition~\ref{prop:kl-tilt} rather than assuming it: the proof, given in Appendix~\ref{app:guidance}, establishes existence by the tightness of KL sublevel sets on a Polish space together with weak lower semicontinuity, and then inherits uniqueness, mutual absolute continuity, and the tilt formula from that proposition. What the theorem adds is that the target is well defined without any attainment hypothesis and that the reward is uniformly bounded, so the tilt redistributes mass within the $P_\theta$-essential support and creates none (Remark~\ref{rem:tilt-scope}).

The tilt~\eqref{eq:kl-tilt} concerns the clean distribution over $x_0$, but diffusion samplers generate $x_0$ as the endpoint of a denoising process that starts from noise at time $T$, so guidance must be injected at every noise level $t$. Then, we let $w(x_0)=\exp(R_{Q^\star}(x_0))$ and $Z=\E_{P_\theta}w(X_0)$, so that $Q^\star(dx_0)=\frac{1}Z w(x_0)P_\theta(dx_0)$, and let $K_t(dx_t\,\vert\,  x_0)$ denote the forward noising kernel with time-$t$ marginals $p_t,q_t^\star$ under $P_\theta,Q^\star$.

\begin{theorem}[Derivation of the exact guidance field]\label{thm:twist}
Let $h_t(x_t)=\E_{P_\theta}[w(X_0)\mid X_t=x_t]$. Then $\frac{dq_t^\star}{dp_t}=\frac{1}Z h_t$, and if $p_t,q_t^\star$ admit positive differentiable densities, then we have
\begin{equation}
    \nabla\log q_t^\star(x_t)=\nabla\log p_t(x_t)+u_t(x_t),
    \qquad
    u_t(x_t)\defeq \nabla_{x_t}\log h_t(x_t).
\end{equation}
\end{theorem}

Note that the function $h_t$ averages the clean-sample reward $w$ over all origins that the base posterior regards as plausible for $x_t$. We emphasize that $h_t$ is \emph{not} the clean tilt evaluated at a denoised point estimate, and this distinction is what makes the identity exact: applying the tilt before the noising process does not commute with applying it afterward. Adding $u_t=\nabla\log h_t$ to the base score yields a reverse process whose marginals match $q_t^\star$ at \emph{every} noise level (Theorem~\ref{thm:reverse}, Appendix~\ref{app:guidance}); for continuous-time samplers, the guided score can be used directly in the reverse SDE or the probability-flow ODE.

One caveat accompanies this exactness. The exact reverse process must be initialized at $q_T^\star$, which is not directly samplable, whereas practical samplers initialize from $p_T$ (typically Gaussian noise). The two distributions coincide only when $h_T$ is constant, i.e.\ when the terminal noise level has erased all reward information. The resulting mismatch enters the end-to-end bound of Theorem~\ref{thm:endpoint} as the initialization term $\KL(p_T\| q_T^\star)$, and its magnitude is quantified in Remark~\ref{rem:init}.

\subsection{Practical diffusion guidance}
\label{subsec:practical}

There exist three approximation items that separate the discussed theoretical framework from an implementable sampler. In the following, we make each one explicit and discuss how to address it.

\paragraph{Plug-in guidance fields.}
The exact field $u_t=\nabla\log h_t$ requires the gradient of a conditional log-moment-generating function under the base posterior $P_\theta(dx_0\mid x_t)$, which is generally intractable. What is available at every noise level is a denoiser $\widehat x_0(x_t,t)\approx\E[X_0\mid X_t=x_t]$, and the \emph{plug-in} approximation substitutes this point estimate for the posterior average. Two variants differ in how the reward is turned into a vector field: the \emph{chain-rule} variant differentiates $x_t\mapsto R_{Q^\star}(\widehat x_0(x_t,t))$ through the denoiser Jacobian $J_{\widehat x_0}$, while the cheaper \emph{direct-injection} variant reuses the clean-space gradient as a direction in noisy-sample space:
\[
    \widetilde u_t^{\mathrm{chain}}(x_t)
    =\omega_t J_{\widehat x_0}(x_t,t)^\top
     \nabla_x R_{Q^\star}\bigl(\widehat x_0(x_t,t)\bigr),
    \qquad
    \widetilde u_t^{\mathrm{dir}}(x_t)
    =\omega_t\nabla_x R_{Q^\star}\bigl(\widehat x_0(x_t,t)\bigr),
\]
with guidance scale $\omega_t\ge0$. Both are heuristics without a general error bound: $R_{Q^\star}$ is nonlinear, and neither conditional expectation nor differentiation commutes with a point-mass substitution (Definition~\ref{def:plugin}, Remark~\ref{rem:plugin}). We always report which variant is used.

\paragraph{Estimating the self-referential reward.}
The reward $R_{Q^\star}$ depends on the unknown covariance $S_{Q^\star}^\varepsilon$, so a practical sampler replaces it by an estimate, and how the estimate is maintained determines the statistical status of the outputs. If the covariance is \emph{frozen}, i.e., computed once from a pilot batch and used to define a single estimated potential $\widehat R$ for all subsequent trajectories, the draws are conditionally i.i.d.\ from the frozen-potential law. If instead the covariance is recomputed on the fly from the batch being generated, each particle's drift depends on the others, and the outputs form an exchangeable but non-i.i.d.\ interacting particle system (Remark~\ref{rem:iid}). Frozen-potential estimation is therefore the setting in which IGA guidance can be described as sampling from a well-defined target distribution.

\paragraph{Discrete sampler updates.}
Once a guidance field $\widetilde u_t$ is chosen, it is converted to a correction on the noise prediction. In the $\varepsilon$-prediction parameterization, with the standard
noise--score convention
$\nabla\log p_t(x_t)=-\varepsilon_\theta(x_t,t)/
\sqrt{1-\bar\alpha_t}$~\citep{song2021score}, guiding the score by $+\widetilde u_t$ corresponds to
\begin{equation}\label{eq:eps-IGA}
    \varepsilon_\theta^{\mathrm{IGA}}(x_t,t)
    =\varepsilon_\theta(x_t,t)-\sqrt{1-\bar\alpha_t}\,\widetilde u_t(x_t).
\end{equation}
Substituting~\eqref{eq:eps-IGA} into the DDPM posterior
mean~\citep{ho2020denoising}
\[
\mu_\theta=
\frac{1}{\sqrt{\alpha_t}}
\left(
x_t-\frac{\beta_t}{\sqrt{1-\bar\alpha_t}}
\varepsilon_\theta
\right)
\]
and the DDIM update~\citep{song2021ddim} gives
\begin{equation}\label{eq:ddpm-ddim}
\begin{gathered}
    \mu_\theta^{\mathrm{IGA}}(x_t,t)
    =\mu_\theta(x_t,t)+\frac{\beta_t}{\sqrt{\alpha_t}}\,\widetilde u_t(x_t),
    \\
    x_{t-1}=\sqrt{\bar\alpha_{t-1}}\,\widehat x_0^{\mathrm{IGA}}
    +\sqrt{1-\bar\alpha_{t-1}-\sigma_t^2}\,\varepsilon_\theta^{\mathrm{IGA}}
    +\sigma_t z,
\end{gathered}
\end{equation}
with $\widehat x_0^{\mathrm{IGA}}=(x_t-\sqrt{1-\bar\alpha_t}\,\varepsilon_\theta^{\mathrm{IGA}})/\sqrt{\bar\alpha_t}$ and $\sigma_t=0$ for deterministic DDIM; the $\sqrt{1-\bar\alpha_t}$ in~\eqref{eq:eps-IGA} and the $\beta_t/\sqrt{\alpha_t}$ in~\eqref{eq:ddpm-ddim} cancel algebraically, so the DDPM mean correction is exactly $+(\beta_t/\sqrt{\alpha_t})\widetilde u_t$. A model trained with $v$- or $x_0$-prediction is first converted to an equivalent $\varepsilon_\theta$ in the standard way (e.g.\ $\varepsilon_\theta=\sqrt{\bar\alpha_t}\,v_\theta+\sqrt{1-\bar\alpha_t}\,x_t$ for $v$-prediction). These discrete updates are implementations inspired by Theorem~\ref{thm:twist}, and they do not exactly sample $Q^\star$ even when $\widetilde u_t=u_t$, because the reverse kernels are discretized (Remark~\ref{rem:ddpm-ddim}). The end-to-end guarantee is given by Theorem~\ref{thm:endpoint} in Appendix~\ref{app:guidance}. This theorem bounds $\KL(\widehat Q\| Q^\star)$ and $\TV(\widehat Q,Q^\star)$ for the deployed continuous-time process in terms of the initialization mismatch, the score error, and the guidance error; a separate discretization term is required for the implemented sampler.

\section{Numerical Evaluation}
\begin{figure*}[!t]
  \centering
  \captionsetup[subfigure]{font=normalsize,labelfont=bf,skip=3pt}

  \begin{subfigure}[t]{0.235\textwidth}
    \centering
    \caption{Population}
    \label{fig:toy-population}
    \includegraphics[width=\linewidth]{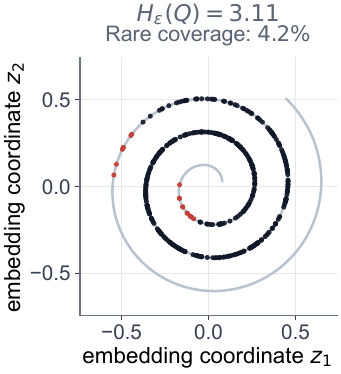}
  \end{subfigure}\hfill
  \begin{subfigure}[t]{0.235\textwidth}
    \centering
    \caption{Fitted DDPM}
    \label{fig:toy-fitted}
    \includegraphics[width=\linewidth]{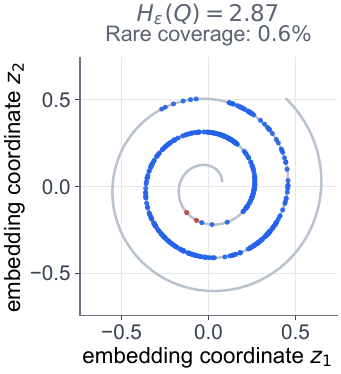}
  \end{subfigure}\hfill
  \begin{subfigure}[t]{0.235\textwidth}
    \centering
    \caption{At the entropy wall}
    \label{fig:toy-wall}
    \includegraphics[width=\linewidth]{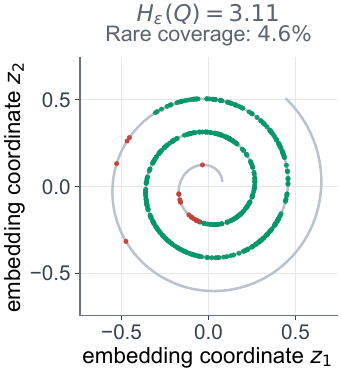}
  \end{subfigure}\hfill
  \begin{subfigure}[t]{0.235\textwidth}
    \centering
    \caption{Beyond the wall}
    \label{fig:toy-beyond}
    \includegraphics[width=\linewidth]{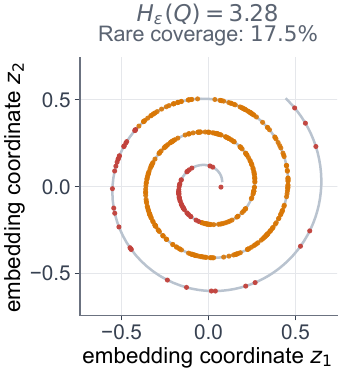}
  \end{subfigure}

  \vspace{0.35em}
  \makebox[\textwidth][c]{\small
    \toyfiglegenddot{population}\hspace{0.25em}Population\hspace{0.9em}
    \toyfiglegenddot{fitted}\hspace{0.25em}Fitted DDPM\hspace{0.9em}
    \toyfiglegenddot{entropywall}\hspace{0.25em}Entropy wall\hspace{0.9em}
    \toyfiglegenddot{beyondwall}\hspace{0.25em}Beyond wall\hspace{0.9em}
    \toyfiglegenddot{rareregion}\hspace{0.25em}Rare-region samples}
  \vspace{0.85em}

  \begin{subfigure}[t]{0.475\textwidth}
    \centering
    \caption{Distributional repair}
    \label{fig:toy-repair}
    \includegraphics[width=\linewidth]{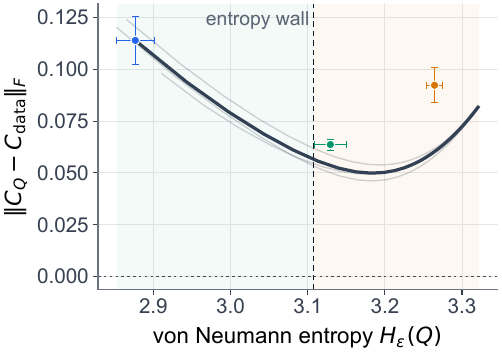}
  \end{subfigure}\hfill
  \begin{subfigure}[t]{0.475\textwidth}
    \centering
    \caption{Rare-region coverage}
    \label{fig:toy-coverage}
    \includegraphics[width=\linewidth]{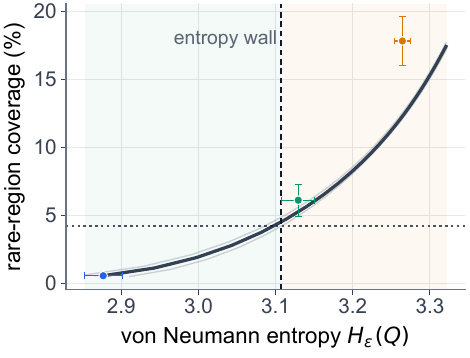}
  \end{subfigure}

  \caption{
    \textbf{Repair and extrapolation on a nonlinear manifold.}
    The fitted DDPM underrepresents the spiral endpoints. IGA restores
    population-level entropy and rare-region coverage near the wall and
    increases endpoint exploration beyond it. Panels (e,f) trace the
    covariance discrepancy and rare-region coverage; thin curves denote
    individual seeds and thick curves their mean.
  }
  \label{fig:entropy-wall-toy}
\end{figure*}

Throughout our numerical study, we aim to empirically address the following questions:

\begin{enumerate}[leftmargin=*]
    \item Do widely used pretrained generative models show a diversity deficit and existence of the entropy wall?

    \item Does our proposed framework address this deficit, and are there values of $\lambda$ that reach and go beyond the entropy wall? If so, as predicted by Theorem~\ref{thm:repair}, is there an initial repair region in which the pretrained model moves closer to the data distribution it was meant to imitate?

    \item Does crossing the wall produce structured, novel, and imaginative variations across different pretrained generative models, including text-conditional models? We test our theory and hypothesis on CelebA-HQ and ImageNet, and then ask what the resulting variation looks like in a large text-conditioned model.
    
\end{enumerate}

\subsection{Experimental Protocol}

Our experiments cover sampling-time IGA on controlled synthetic distributions and real-world image benchmarks, a training-time study on MNIST, and qualitative text-conditioned generation with SDXL. Because these settings use different models and evaluation criteria, we state only the shared experimental conventions here and introduce the setting-specific configurations in the corresponding subsections.

\begin{figure*}[!tp]
    \centering
    \includegraphics[width=\textwidth]{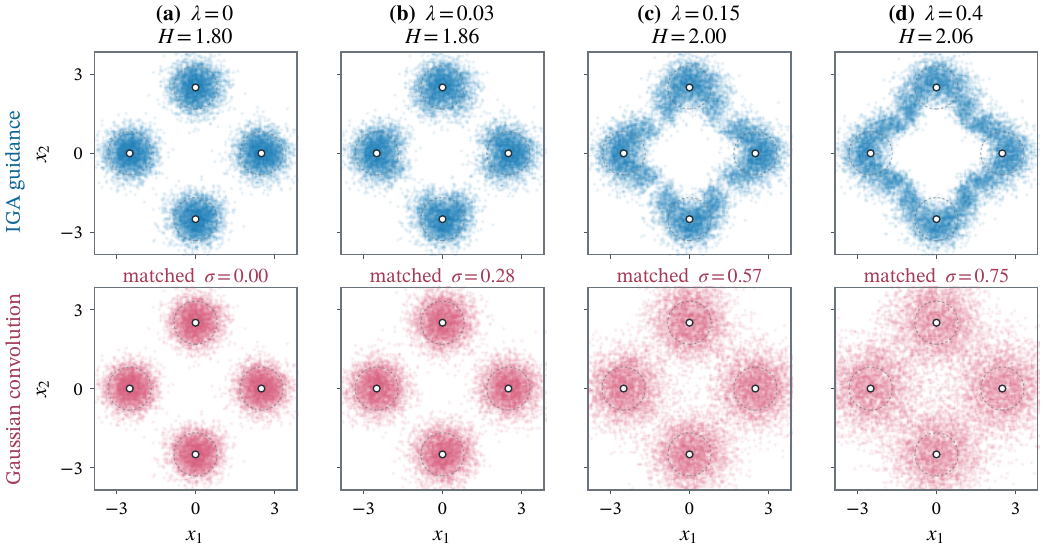}
    \caption{
        \textbf{IGA fills underrepresented inter-mode regions more
        coherently than entropy-matched noise.}
        Top: base and IGA-guided DDIM samples as $\lambda$ increases.
        Bottom: Gaussian-convolved samples with $\sigma$ selected to
        match the entropy of the corresponding IGA distribution.
        IGA connects the gaps between modes while preserving the
        original modal structure; Gaussian convolution broadens each
        mode isotropically.
    }
    \label{fig:ega-extrapolation}
\end{figure*}

\begin{figure*}[!b]
    \centering
    \includegraphics[width=0.9\textwidth]{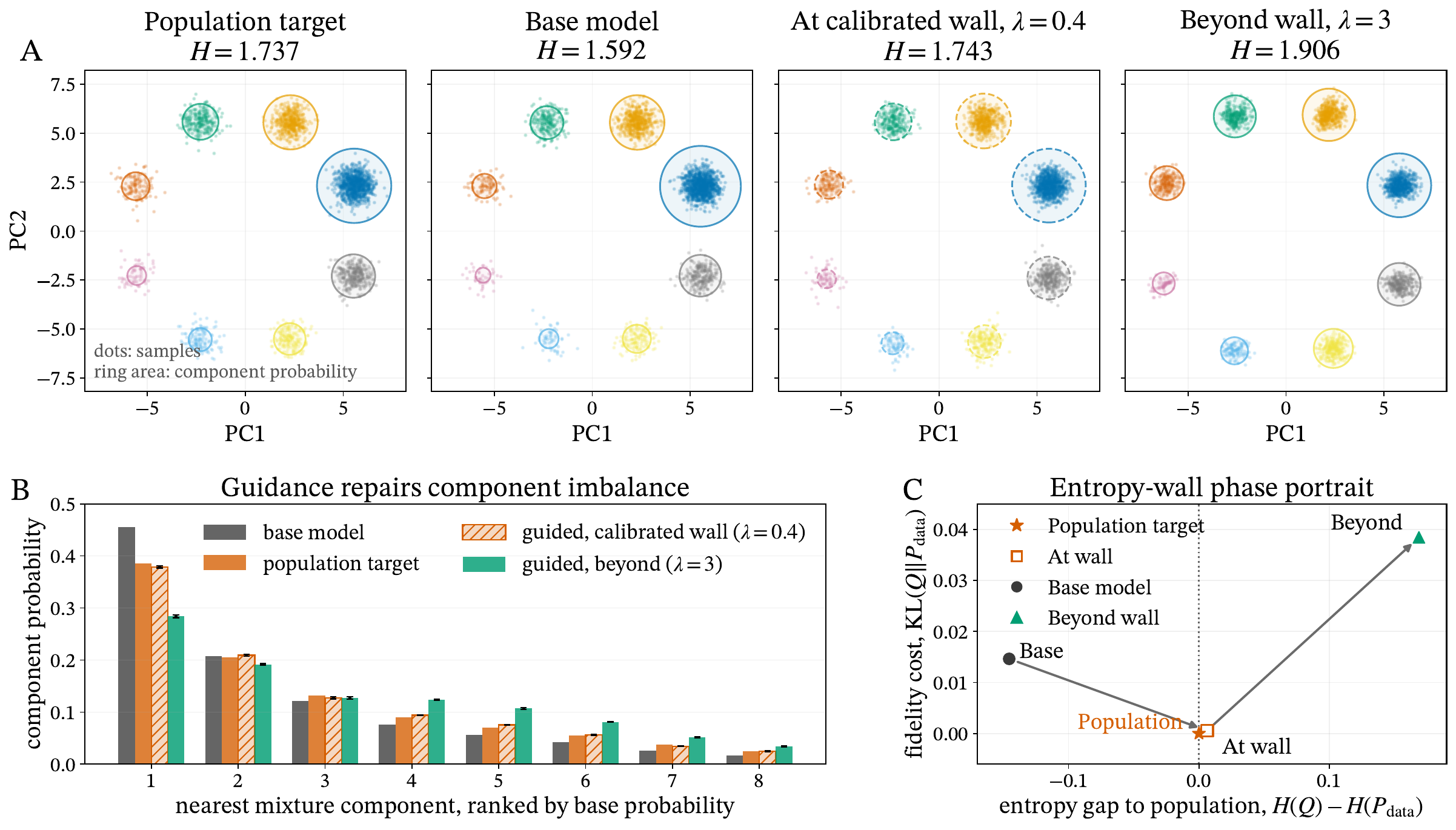}
    \caption{
        \textbf{The entropy wall separates repair from extrapolation
        in a controlled mixture.}
        \textbf{(A)} Population, base, calibrated-wall, and beyond-wall
        distributions. \textbf{(B)} IGA repairs the component imbalance
        at the wall and produces more uniform weights beyond it.
        \textbf{(C)} Population KL first decreases and then increases
        as the path crosses the entropy wall.
    }
    \label{fig:mixture-entropy-wall}
\end{figure*}

\paragraph{Entropy representation and evaluation.}
Across all sampling-time experiments, we set the spectral floor
of~\eqref{eq:smoothed-def} to $\varepsilon=10^{-3}$, for which Lemma~3 bounds the gap between $H_\varepsilon$ and $H_0$ by $0.015$ nats, several times smaller than the smallest entropy difference we report. IGA guidance uses the smoothed entropy $H_\varepsilon$, while, unless stated otherwise, we report the unsmoothed von Neumann entropy $H_0$ and its exponential, $\exp(H_0)$, corresponding to the Vendi score~\citep{friedman2023vendi}. In all sampling-time experiments, we fix the coefficient of the IGA score correction to its theoretically prescribed value of one and vary only $\lambda$, so each point along the reported path corresponds to a different target $Q_\lambda^\star$, rather than to a different guidance strength.

\noindent For the real-image benchmarks, we use the CLS-token embeddings of DINOv2 ViT-B/14~\citep{oquab2023dinov2} and approximate an RBF kernel on these embeddings using 1024 random Fourier features~\citep{rahimi2007random}, with the RBF bandwidth selected by the median heuristic. Because empirical spectral entropy depends on the number of samples, we compare real and generated distributions using matched sample sizes when locating the empirical entropy wall. In DINOv2 feature space, we report Fr\'echet distance, kernel distance, and recall. As an evaluation independent of the guidance representation, we additionally report FID and KID in Inception-v3 feature space~\citep{heusel2017gans,binkowski2018demystifying,kynkaanniemi2019improved}. Synthetic and training-time experiments use the problem-specific metrics introduced in their respective subsections.

\paragraph{Pilot estimation.}
In all sampling-time experiments, we use the chain-rule plug-in guidance field. The covariance entering the IGA potential is estimated from an independent pilot batch of 2048 samples drawn from the unguided base model and then frozen during subsequent sampling. Thus, conditioned on the frozen pilot estimate, individual sampling trajectories are independent.

\subsection{Numerical Application of IGA in Post-hoc Sampling-time Mode}

\subsubsection{Synthetic Experiments with Known Groundtruth Model}

We first study three controlled settings in which the population distribution is known. These experiments allow us to evaluate whether the IGA path approaches the population below the entropy wall and departs from it beyond the wall. The three settings provide complementary evidence: a nonlinear manifold illustrates rare-region repair, an entropy-matched control distinguishes IGA from isotropic noise for increasing entropy, and a finite mixture exposes the redistribution of probability mass across modes.

\paragraph{Rare-region repair on a nonlinear manifold.}
Figure~\ref{fig:entropy-wall-toy} considers a
DDPM~\citep{ho2020denoising} trained on a one-dimensional population embedded in $\mathbb{R}^{128}$. The fitted model captures the dominant central portion of the manifold but substantially underrepresents its endpoints, leading to lower entropy and reduced rare-region coverage. Here, rare regions are defined as the portions outside the central $70\%$ of the normalized manifold coordinate.

\noindent Increasing $\lambda$ initially corrects this contraction. Near the entropy wall, IGA recovers both the population entropy and the missing endpoint mass. Panels~\ref{fig:toy-repair}
and~\ref{fig:toy-coverage} show the corresponding transition:
the covariance discrepancy decreases as the path approaches the wall, while rare-region coverage increases. Beyond the wall, coverage
continues to grow, but the discrepancy to the population turns upward. Thus, the same path first repairs variation lost by the fitted model and then promotes exploration beyond the population level.

\paragraph{Structured coverage versus entropy-matched noise.}
The spiral experiment shows that IGA directs probability toward
underrepresented regions. To determine whether this behavior could be reproduced by simply adding noise, Figure~\ref{fig:ega-extrapolation}
compares IGA with an entropy-matched Gaussian-convolution baseline applied to a multimodal DDIM model~\citep{song2021ddim}. For each IGA setting, the convolution scale $\sigma$ is selected by bisection
so that
\(
    Q_\sigma
    =
    P_{\mathrm{DDIM}} * \mathcal{N}(0,\sigma^2 I)
\)
attains the same representation-space von Neumann entropy.

\noindent Despite matching entropy, the two methods distribute their additional mass differently. Gaussian convolution broadens every mode approximately isotropically, producing increasingly diffuse clouds around the original modal centers. IGA instead selectively fills the underrepresented regions between neighboring modes. As $\lambda$ increases, these inter-mode regions form a coherent ring while the original modes remain visible. The entropy increase produced by IGA therefore reflects structure-aware redistribution rather than an undirected increase in noise.

\paragraph{Population-level confirmation in a controlled mixture.}
Figure~\ref{fig:mixture-entropy-wall} provides a complementary view using an eight-component mixture with known population weights. The base distribution overweights its most frequent components and underrepresents the remaining modes, resulting in lower entropy than the population. Increasing $\lambda$ initially corrects this imbalance: at the calibrated wall, the guided distribution approximately recovers both the population entropy and its component probabilities. Beyond the wall, the component probabilities become more uniform than those of the population.

\noindent Panel C of Figure~\ref{fig:mixture-entropy-wall} makes the change in regime explicit. Along the below-wall portion of the path, the
population KL decreases as IGA repairs the component imbalance. After the wall is crossed, entropy continues to increase while KL turns upward. The path therefore first approaches the population through diversity repair and subsequently departs from it through deliberate extrapolation.

\noindent Together, these controlled experiments show that IGA restores underrepresented population structure below the entropy wall and enters an extrapolative regime beyond it. They further show that the increase in diversity arises from selective redistribution toward underrepresented regions rather than isotropic perturbation. We next examine whether the same progression appears in pretrained diffusion models on real-world image benchmarks.

\subsubsection{Real-World Image Distribution Benchmarks}

Having established the repair-to-extrapolation transition in
controlled settings, we next ask whether the same progression appears
in pretrained diffusion models on real-world image distributions. We
evaluate sampling-time IGA on unconditional CelebA-HQ and
class-conditional ImageNet generation.

\paragraph{Benchmark settings.}
On CelebA-HQ~\citep{karras2018progressive}, we guide the pretrained
\texttt{google/ddpm-ema-celebahq-256}
DDPM~\citep{ho2020denoising} at $256\times256$, using deterministic
DDIM sampling~\citep{song2021ddim} for $100$ steps. On ImageNet, we use the standard ImageNet-100 subset introduced by Tian et al.~\citep{tian2020contrastive}, consisting of their fixed 100-class subset of ILSVRC-2012~\citep{russakovsky2015imagenet}. We guide the class-conditional
\texttt{facebook/DiT-XL-2-256}
model~\citep{peebles2023scalable} for $50$ DDIM steps, using
classifier-free guidance~\citep{ho2022classifierfree} at scale $2.0$.

\begin{figure*}[!t]
    \centering

    \begin{subfigure}[t]{0.615\textwidth}
        \vspace{0pt}
        \centering
        \includegraphics[width=\linewidth]
            {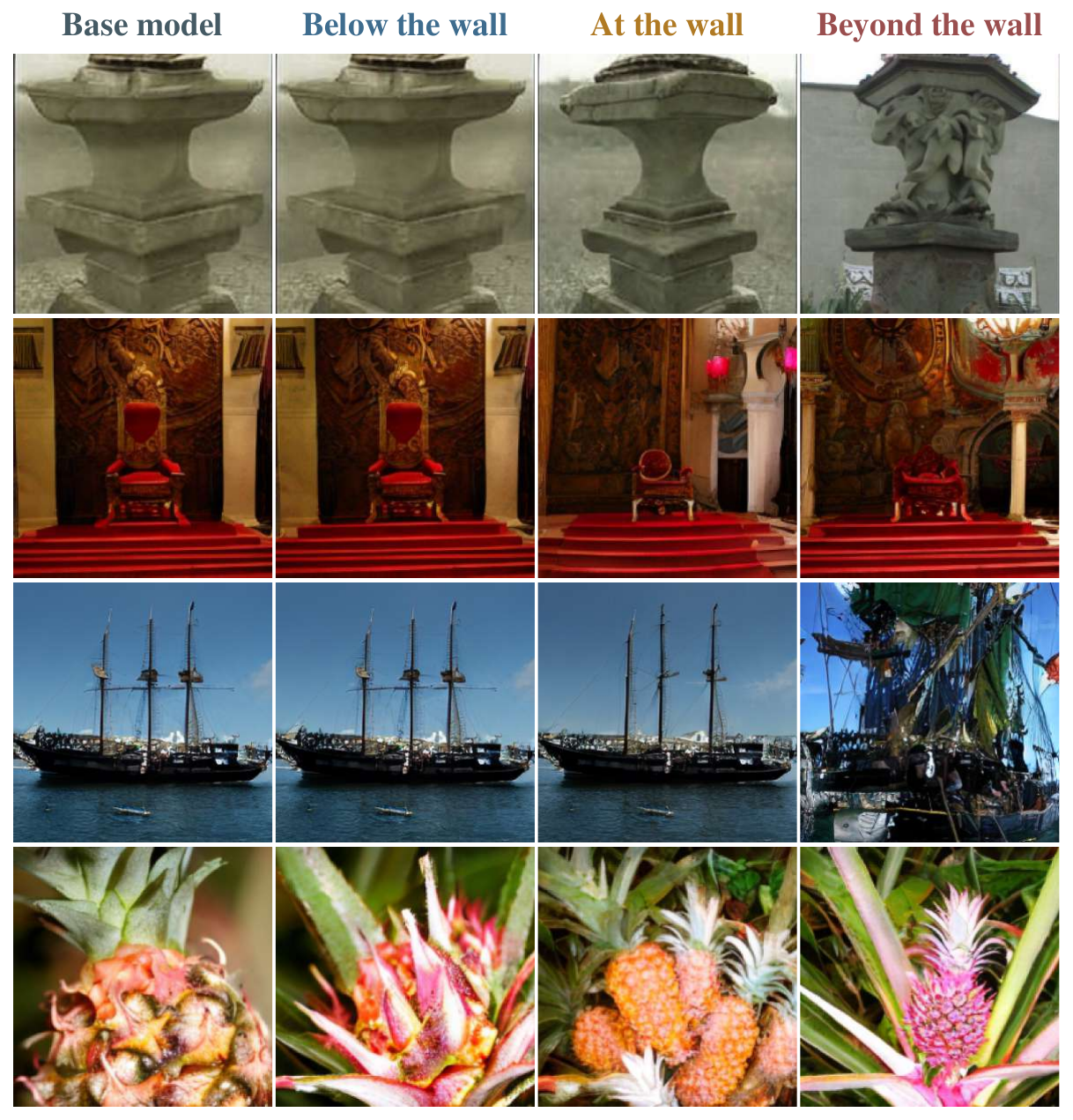}
        \caption{Qualitative transition across the entropy wall.}
        \label{fig:imagenet-wall-samples}
    \end{subfigure}
    \hfill
    \begin{minipage}[t]{0.365\textwidth}
        \vspace{0pt}
        \centering

        \begin{subfigure}[t]{\linewidth}
            \centering
            \includegraphics[width=\linewidth]
                {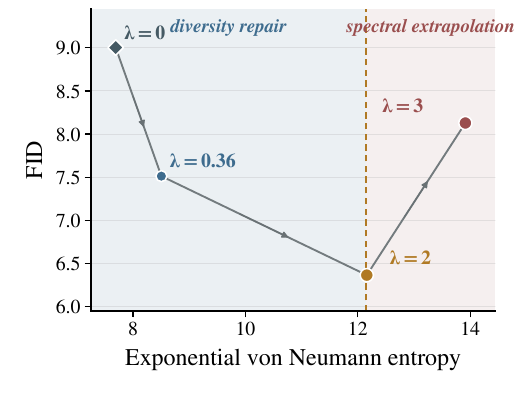}
            \caption{FID--entropy phase portrait.}
            \label{fig:imagenet-wall-fid}
        \end{subfigure}

        \vspace{0.7em}

        \begin{subfigure}[t]{\linewidth}
            \centering
            \includegraphics[width=\linewidth]
                {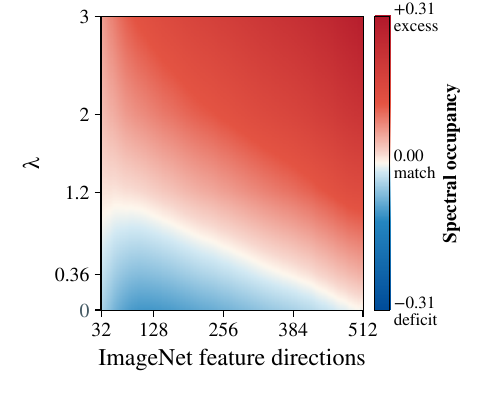}
            \caption{Spectral occupancy.}
            \label{fig:imagenet-wall-spectrum}
        \end{subfigure}
    \end{minipage}

    \caption{
        \textbf{IGA across the ImageNet entropy wall.}
        \textbf{(a)} Matched samples along the IGA target path as
        $\lambda$ increases.
        \textbf{(b)} FID initially decreases as entropy approaches
        the empirical wall and turns beyond it.
        \textbf{(c)} Cumulative spectral occupancy relative to the
        data, across feature directions ordered from dominant to
        rare.
    }
    \label{fig:imagenet-wall-main}
\end{figure*}

The coefficient on the IGA score correction is fixed to one, as
prescribed by the sampling-time construction. We vary only the entropy
multiplier $\lambda$. Each point along the reported path therefore
corresponds to a different entropy-regularized target, rather than to
a stronger or weaker application of the same guidance field. The
DDIM sampler implements the corresponding unit-scale plug-in
correction at each denoising step.

\paragraph{Diversity deficit and wall crossing.}
Because empirical spectral entropy depends on sample size, Figure~\ref{fig:vendi-sample-size} compares
real and generated distributions using matched numbers of samples.
On both datasets, the base model remains below the corresponding empirical data wall throughout the evaluated sample-size range.
Increasing $\lambda$ progressively closes this deficit, reaches the wall at an intermediate point, and crosses it for larger values. The tested path therefore spans three interpretable regimes: a diversity-deficient base model, below-wall repair, and beyond-wall imagination.

\paragraph{The ImageNet path across the wall.}
Figure~\ref{fig:imagenet-wall-main} summarizes the progression on
ImageNet. The matched samples in
Figure~\ref{fig:imagenet-wall-samples} show the transition from the
base model through below-wall repair and into beyond-wall
extrapolation. The FID--entropy phase portrait in
Figure~\ref{fig:imagenet-wall-fid} shows the corresponding
distributional trend: FID initially decreases as the entropy deficit
is repaired and turns after the target approaches and crosses the
empirical wall. To examine how the additional entropy is obtained, we define the
cumulative spectral-occupancy ratio
\[
    T_{\lambda}(r)
    =
    \log
    \frac{\sum_{i=r}^{d} v_i^{\top}S_{\lambda}v_i}
         {\sum_{i=r}^{d} v_i^{\top}S_{\mathrm{data}}v_i},
\]
where $S_{\lambda}$ and $S_{\mathrm{data}}$ denote the generated and data covariance matrices in DINOv2 feature space, respectively, and the data-covariance eigenvectors $v_i$ are ordered from dominant to rare. Negative values indicate an occupancy deficit relative to
the data, whereas positive values indicate excess occupancy.
Figure~\ref{fig:imagenet-wall-spectrum} shows that increasing
$\lambda$ progressively closes the deficit across underrepresented
directions and produces excess occupancy after the wall is crossed.
IGA therefore gains entropy by allocating more probability to
directions that the base generator covers insufficiently.

\begin{figure*}[!t]
    \centering
    \captionsetup[subfigure]{font=normalsize,labelfont=bf,skip=3pt}
    \begin{subfigure}[t]{0.48\textwidth}
        \centering
        \includegraphics[width=\linewidth]
            {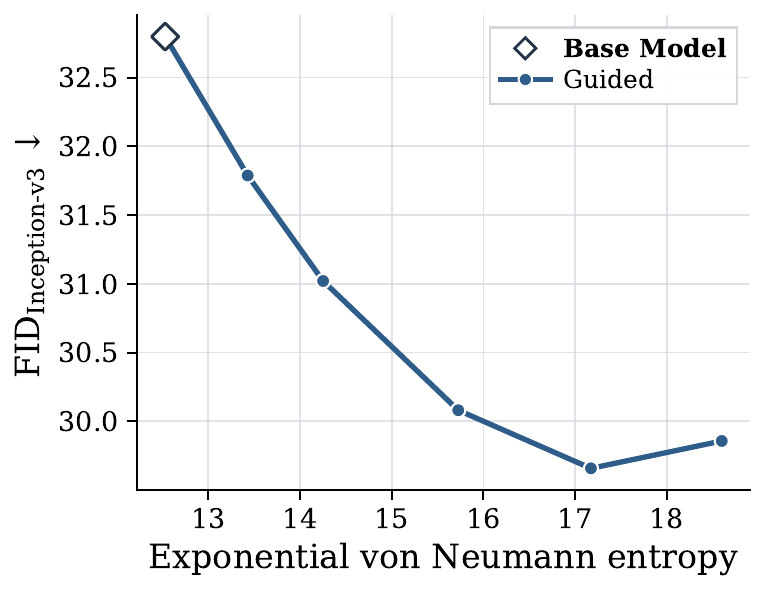}
        \caption{CelebA-HQ: FID.}
        \label{fig:celeba-inception-fid}
    \end{subfigure}
    \hfill
    \begin{subfigure}[t]{0.48\textwidth}
        \centering
        \includegraphics[width=\linewidth]
            {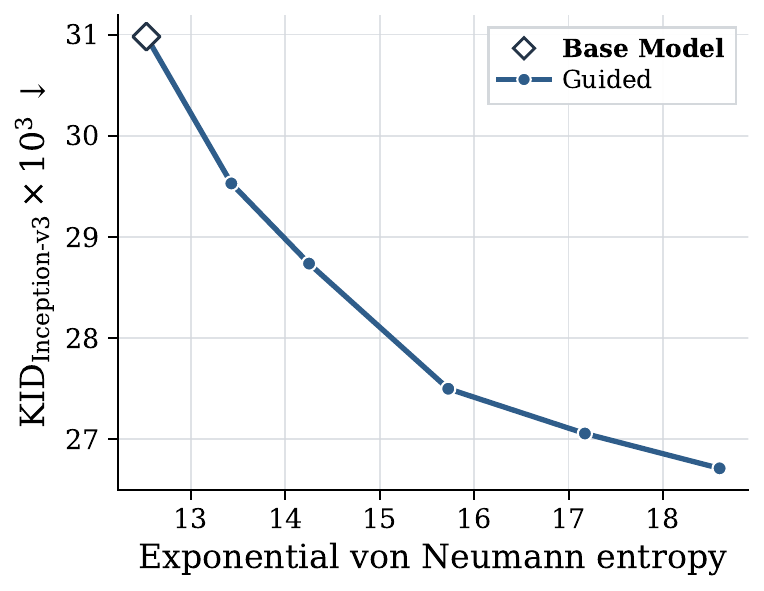}
        \caption{CelebA-HQ: KID.}
        \label{fig:celeba-inception-kid}
    \end{subfigure}

    \vspace{0.6em}

    \begin{subfigure}[t]{0.48\textwidth}
        \centering
        \includegraphics[width=\linewidth]
            {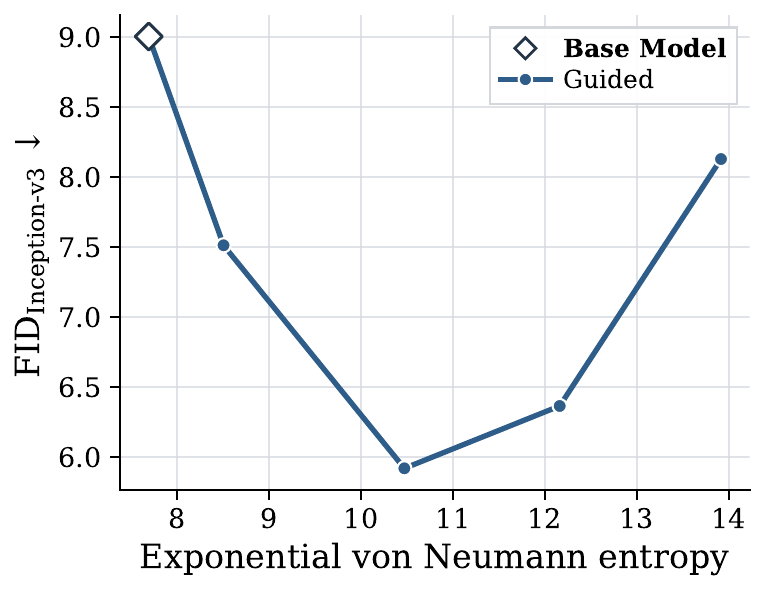}
        \caption{ImageNet: FID.}
        \label{fig:imagenet-inception-fid}
    \end{subfigure}
    \hfill
    \begin{subfigure}[t]{0.48\textwidth}
        \centering
        \includegraphics[width=\linewidth]
            {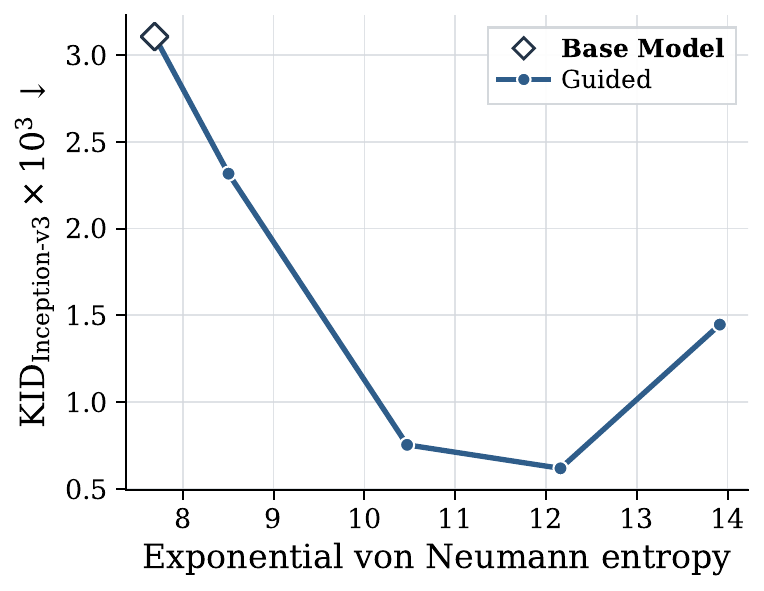}
        \caption{ImageNet: KID.}
        \label{fig:imagenet-inception-kid}
    \end{subfigure}

    \caption{
        \textbf{Independent Inception-v3 evaluation along the IGA
        target path.}
        FID and KID are measured in Inception-v3 feature space,
        independently of the DINOv2 representation used to define
        spectral diversity and the empirical entropy wall.
        On both CelebA-HQ and ImageNet, the initial below-wall portion
        of the path improves distributional agreement with the data
        while diversity increases. At larger values of $\lambda$,
        the behavior becomes metric- and dataset-dependent, with
        distributional distances eventually flattening or turning as
        the target enters the extrapolative regime.
        Corresponding DINOv2-space FD and KD curves are reported in
        Appendix~\ref{app:additional}.
    }
    \label{fig:inception-tradeoffs}
\end{figure*}

\begin{table*}[!t]
    \centering
    \small
    \setlength{\tabcolsep}{4.2pt}
    \renewcommand{\arraystretch}{1.08}

    \caption{
        \textbf{Comparison with diversity-guidance methods.}
        Distributional distances, coverage, and spectral diversity on
        CelebA-HQ and ImageNet. Shaded rows trace the IGA target
        path as $\lambda$ increases. Bold indicates the lowest
        distributional distance or highest recall within each dataset
        block.
    }
    \label{tab:main-results}

    \begin{tabular}{@{}clccccccc@{}}
        \toprule

        & & &
        \multicolumn{2}{c}{Inception-v3}
        & \multicolumn{2}{c}{DINOv2}
        & \multicolumn{2}{c}{Coverage \& diversity} \\

        \cmidrule(lr){4-5}
        \cmidrule(lr){6-7}
        \cmidrule(l){8-9}

        & \multicolumn{1}{l}{Method}
        & Vendi
        & KID\,$\times 10^{3}\downarrow$
        & FID\,$\downarrow$
        & KD\,$\times 10^{2}\downarrow$
        & FD\,$\downarrow$
        & Recall\,$\uparrow$
        & $H_0$ \\

        \midrule

        \multirow{7}{*}{
            \rotatebox[origin=c]{90}{\textbf{CelebA-HQ}}
        }
        & Base
        & 12.5
        & 30.98
        & 32.80
        & 12.331
        & 158.7
        & 0.611
        & 2.528 \\

        & SPARKE~\citep{jalali2025sparke}
        & 23.0
        & 50.314
        & 54.77
        & 30.294
        & 792.5
        & 0.479
        & 3.134 \\

        \addlinespace[2pt]

        & \igacell{IGA, $\lambda=0.5$}
        & \igacell{13.4}
        & \igacell{29.53}
        & \igacell{31.79}
        & \igacell{12.018}
        & \igacell{152.5}
        & \igacell{0.627}
        & \igacell{2.597} \\

        & \igacell{IGA, $\lambda=1$}
        & \igacell{14.3}
        & \igacell{28.74}
        & \igacell{31.02}
        & \igacell{\textbf{11.876}}
        & \igacell{\textbf{151.0}}
        & \igacell{0.649}
        & \igacell{2.657} \\

        & \igacell{IGA, $\lambda=2$}
        & \igacell{15.7}
        & \igacell{27.50}
        & \igacell{30.08}
        & \igacell{12.049}
        & \igacell{156.4}
        & \igacell{0.685}
        & \igacell{2.755} \\

        & \igacell{IGA, $\lambda=3$}
        & \igacell{17.2}
        & \igacell{27.06}
        & \igacell{\textbf{29.66}}
        & \igacell{12.616}
        & \igacell{169.8}
        & \igacell{0.713}
        & \igacell{2.843} \\

        & \igacell{IGA, $\lambda=4$}
        & \igacell{18.6}
        & \igacell{\textbf{26.71}}
        & \igacell{29.86}
        & \igacell{13.272}
        & \igacell{192.1}
        & \igacell{\textbf{0.732}}
        & \igacell{2.923} \\

        \midrule

        \multirow{7}{*}{
            \rotatebox[origin=c]{90}{\textbf{ImageNet}}
        }
        & Base
        & 7.7
        & 3.107
        & 9.00
        & 4.595
        & 118.3
        & 0.598
        & 2.040 \\

        & CADS~\citep{sadat2024cads}
        & 13.1
        & 1.260
        & 6.05
        & 4.518
        & 123.3
        & 0.645
        & 2.571 \\

        & SPARKE~\citep{jalali2025sparke}
        & 7.9
        & 2.912
        & 7.89
        & 4.593
        & 115.2
        & 0.470
        & 2.061 \\

        \addlinespace[2pt]

        & \igacell{IGA, $\lambda=0.36$}
        & \igacell{8.5}
        & \igacell{2.317}
        & \igacell{7.51}
        & \igacell{4.572}
        & \igacell{\textbf{114.3}}
        & \igacell{0.636}
        & \igacell{2.141} \\

        & \igacell{IGA, $\lambda=1.2$}
        & \igacell{10.5}
        & \igacell{0.753}
        & \igacell{\textbf{5.92}}
        & \igacell{4.550}
        & \igacell{123.0}
        & \igacell{0.696}
        & \igacell{2.349} \\

        & \igacell{IGA, $\lambda=2$}
        & \igacell{12.2}
        & \igacell{\textbf{0.617}}
        & \igacell{6.37}
        & \igacell{\textbf{4.546}}
        & \igacell{143.5}
        & \igacell{0.727}
        & \igacell{2.498} \\

        & \igacell{IGA, $\lambda=3$}
        & \igacell{13.9}
        & \igacell{1.446}
        & \igacell{8.13}
        & \igacell{4.548}
        & \igacell{170.7}
        & \igacell{\textbf{0.760}}
        & \igacell{2.633} \\

        \bottomrule
    \end{tabular}
\end{table*}

\paragraph{Repair below the wall.}
Figure~\ref{fig:inception-tradeoffs}
shows a consistent initial repair regime
on both benchmarks. At the beginning of the below-wall path, spectral
diversity and recall increase while the reported feature-space
distances decrease relative to the base model. IGA therefore recovers
variation missing from the pretrained generator while improving its
agreement with the data under both DINOv2 and Inception-v3
representations. This behavior is consistent with the repair result
of Theorem~\ref{thm:repair}. As the target approaches the empirical wall, entropy and recall continue to increase, while the different distances attain their minima at nearby but nonidentical values of $\lambda$. This is expected: Inception-v3 and DINOv2 encode different properties of image distributions and need not identify the same target as closest to the data.

\paragraph{Imagination beyond the wall.}
Once $\lambda$ moves the target beyond the wall, entropy and recall
continue to increase, but the feature-space distances no longer
decrease uniformly. Their turning points depend on the dataset,
metric, and evaluation representation. This should not be interpreted
as a direct measurement of declining perceptual image quality. FID,
KID, FD, and KD measure distributional departure from the data in
particular feature spaces. Beyond the wall, their increase instead
indicates that the generated distribution is moving farther from the
data reference while occupying additional feature directions.

\noindent The two benchmarks therefore exhibit the same overall progression:
IGA first repairs a measurable diversity deficit and then enters a
different statistical regime after crossing the wall. The central
result is not a single optimal value of $\lambda$, but an
interpretable target path whose meaning changes from distributional
repair to deliberate spectral extrapolation.

\paragraph{Reference diversity-guidance methods.}
Table~\ref{tab:main-results} includes CADS~\citep{sadat2024cads} and SPARKE~\citep{jalali2025sparke} as reference points rather than
like-for-like baselines. CADS perturbs the conditioning signal and is therefore reported only on class-conditional ImageNet, while the
evaluated SPARKE configuration uses joint batch guidance and produces coupled samples. Neither method defines its operating point relative to the data entropy or distinguishes below-wall repair from beyond-wall extrapolation. In contrast, IGA traces a wall-calibrated family of target distributions and, once its potential is estimated and frozen, applies the same guidance independently to each sampling trajectory. The table therefore provides numerical context under common evaluation metrics rather than a comparison of identical objectives or guarantees.

\noindent Across both benchmarks, the empirical picture is consistent. The pretrained model begins below the entropy wall; below-wall values of $\lambda$ repair part of this deficit while increasing coverage and reducing distributional distances; and larger values cross the wall, where additional coverage is accompanied by a representation-dependent departure from the data distribution. IGA therefore exposes an interpretable target path from diversity repair to controlled imagination.

\subsection{Numerical Application of Training-Time IGA}
\label{subsec:training-time-iga}

\noindent
The main empirical focus of this paper is sampling-time IGA, which can be applied to a pretrained generator without retraining. Nevertheless, the same distribution-level regularization principle also extends naturally to model training. Section~\ref{subsec:training} formulates
training-time IGA as~\eqref{eq:IGA-training-population}, where the model is trained to balance fidelity to the empirical data distribution with the spectral entropy of its generated distribution. The corresponding MNIST results are reported in Appendix~\ref{app:mnist-training-results}.

\noindent
We evaluate this training-time realization using a GAN on MNIST. Following Proposition~\ref{prop:IGA-gan}, the entropy reward is added to the adversarial objective through the joint-adversary formulation in~\eqref{eq:IGA-gan}. For each generated minibatch, we compute the spectral adversary at its closed-form best response and hold it fixed
during the generator update. By Proposition~\ref{prop:entropy-grad}, this frozen payoff gives the exact gradient of the minibatch entropy at the refresh point.

\paragraph{Implementation details.}
We use a convolutional GAN with a 64-dimensional latent and batch size 128, trained for 20 epochs with Adam using learning rate $2\times10^{-4}$ and $(\beta_1,\beta_2)=(0.5,0.999)$ for both generator and discriminator. We use the non-saturating logistic generator objective and one discriminator update per generator update. The IGA representation is the unit-normalized 64-dimensional embedding of a frozen MNIST classifier, while an architecturally distinct frozen classifier with a 96-dimensional embedding is used for independent evaluation. We set $\varepsilon=0.05$ and report means and standard errors over five random seeds.

\noindent
The baseline GAN exhibits a noticeable imbalance in generated digit frequencies despite being trained on the nearly balanced MNIST distribution. Figure~\ref{fig:mnist-IGA-repair} in Appendix~\ref{app:mnist-training-results} shows that moderate IGA regularization redistributes probability mass toward digit classes underrepresented by the baseline generator. At the best intermediate settings, the total variation distance between the generated and
empirical class distributions decreases by $37.1\%$, while the Fr\'echet distance measured in the feature space of a separate evaluator network decreases by $40.1\%$. The improvement in both metrics indicates that the effect is not limited to the class-frequency statistic. Their nonmonotone dependence on the IGA multiplier also illustrates the tradeoff between the GAN fidelity objective and the distribution-level entropy reward.

\subsection{IGA Application to Prompt-Conditioned Generative Models}

We finally test sampling-time IGA on two text-to-image models:
Stable Diffusion XL (SDXL)~\citep{podell2024sdxl} and PixArt-$\Sigma$~\citep{chen2024pixartsigma}. We use \texttt{stable-diffusion-xl-base-1.0} at $768\times768$ and \texttt{PixArt-Sigma-XL-2-1024-MS} at $1024\times1024$, with deterministic DDIM sampling for $50$ steps. The IGA score correction is fixed at unit scale, and only $\lambda$ is varied.
Figure~\ref{fig:IGA-text-model-sweep} shows the resulting path, using $\lambda\in\{0,4,8\}$ for SDXL and $\lambda\in\{0,10,20\}$ for PixArt-$\Sigma$. Within each model, the initial noise seeds are matched across the sweep. Increasing $\lambda$ leads to larger changes in garment shape, volume, material, and color
while remaining consistent with the prompt.

\noindent Figures~\ref{fig:IGA-sdxl-fashion}--\ref{fig:IGA-sdxl-throne} compare vanilla and IGA SDXL across fashion, architecture, underwater painting, and throne design. The base samples tend to stay near familiar forms, whereas IGA produces sculptural garments, curved and stacked towers, underwater scenes with divers and vehicles, and more varied throne structures. Across these examples, the main changes are in shape, structure, and scene composition rather than only color or texture. Additional qualitative results for PixArt-$\Sigma$ are provided in Appendix~\ref{app:additional_viz}.

\section{Conclusion and Discussion}
Generative modeling is typically formulated as distributional imitation, i.e., the ultimate goal is to generate fresh samples from the underlying distribution of real training samples. However, as shown in \cite{farnia2026exposing}, such an approach can empirically lead to a model that generates high-quality samples while remaining systematically less diverse than the target real distribution. Our work introduces \emph{Imaginative Generative AI (IGA)}, a distribution-level framework that incorporates spectral entropy as an explicit and controllable diversity component of the generative modeling objective. The real data distribution's spectral entropy establishes an \emph{Entropy Wall} in the application of IGA: below this wall, IGA entropy regularization repairs diversity lost during training while remaining compatible with the diversity of the data; beyond the wall, the generated distribution intentionally attains greater representation-relative spectral diversity than the real data.

We note that the IGA regularization principle can be applied to both the training of a generative model and post-hoc sampling from a pretrained model. Therefore, IGA provides a general framework for diversity regularization and imaginative generation. Beginning with improving the imitation regime, the approach first counteracts spectral-diversity deficits and encourages the recovery of variation underrepresented by the learned generator. Upon reaching the Entropy Wall, additional regularization transitions into a controlled extrapolative regime, balancing increased representation-relative diversity with closeness to the reference distribution. During sampling, once the IGA guidance potential is fixed, the resulting target enables independent and identically distributed generation, eliminating the need for an interacting batch.

Our numerical results support the application of IGA for both diversity repair and imaginative data generation. Pretrained diffusion models demonstrate a measurable entropy deficit compared with matched real-data samples; moderate IGA guidance addresses this deficit, enhancing diversity and, in several cases, distributional fidelity. Stronger guidance crosses the Entropy Wall and generates structured variation beyond the data reference level, including qualitatively novel and imaginative changes in large text-to-image models. These findings indicate that diversity enhancement does not need to be treated as an architecture-specific heuristic or as an uncontrolled deviation from quality. Instead, IGA offers a general regularization framework for systematically transitioning from imitation, through diversity repair, to controlled imaginative extrapolation.

The notion of imagination in IGA is intentionally representation-relative: exceeding the Entropy Wall means exceeding the spectral diversity of the data in a specified embedding space, rather than satisfying a representation-independent notion of creativity or novelty. Consequently, the choice of representation, the reference distribution, and the fidelity discrepancy remain important modeling decisions. Subject to these choices, the Entropy Wall provides an explicit and measurable boundary between improving imitation and deliberately moving beyond it, making the transition from imitation to imagination mathematically well-defined and controllable.

\bibliographystyle{unsrt}
\bibliography{ref}

\clearpage

\appendix

\section{Related Work}
\label{sec:related-work}

\paragraph{Diversity in generative models.}
Diversity loss is a persistent problem across generative modeling. In generative adversarial networks (GANs), mode collapse and limited effective support motivated minibatch discrimination and support-size diagnostics \citep{salimans2016improved,arora2018gans}. Language models likewise tend toward generic, repetitive, or homogeneous outputs, motivating diversity-aware objectives and decoding strategies \citep{li2016diversity,holtzman2020degeneration,jiang2025hivemind}. More recently, R\'enyi Kernel Entropy (RKE) evaluations have shown that modern generators can produce high-quality samples while still missing modes \citep{jalali2023information}. Related deficits have been documented in image diffusion models, together with a systematic gap between real and generated diversity for which finite-sample entropy underestimation is one statistical source \citep{dombrowski2025image,farnia2026exposing}. These findings motivate IGA's premise: diversity should be specified as a property of the target distribution rather than left as a by-product of distribution fitting.

\paragraph{Measuring novelty and diversity.}
Reference-free measures assess variation within a distribution through similarities among its samples. The Vendi Score uses the von Neumann entropy of a normalized kernel matrix \citep{friedman2023vendi}, while RKE provides a tractable order-two counterpart with mode-count interpretations \citep{jalali2023information}. Kernel-based Entropic Novelty compares which modes are more strongly expressed than in a reference distribution \citep{zhang2024interpretable}. Conditional Vendi and Scendi extend diversity evaluation to prompt-conditioned generators \citep{jalali2026conditional,ospanov2025scendi}, while scalable and truncated variants address computational cost and finite-sample estimation \citep{ospanov2024scalable,ospanov2025truncated}. A complementary line folds sample quality directly into the diversity score itself, yielding quality-weighted Vendi scores \citep{nguyen2024quality}. IGA moves this spectral perspective from post-hoc evaluation into the generative objective itself. Population-data entropy then defines an entropy wall that separates recovery of lost diversity from deliberate extrapolation beyond the data.

\paragraph{Promoting novelty and diversity.}
Existing interventions typically specialize either training or generation. At training time, diversity has been promoted through reinforcement learning with explicit diversity rewards \citep{miao2024diverse} and through diversity-aware diffusion modules \citep{dombrowski2025image}. At inference, CADS anneals noise in the conditioning signal \citep{sadat2024cads}; c-VSG and SPARKE guide generation using contextualized Vendi and conditional RKE, respectively \citep{askarihemmat2024cvsg,jalali2025sparke}; and STRIDE perturbs intermediate features in distilled one- and few-step generators \citep{yadav2026stride}. Particle Guidance instead evolves an interacting set under a pairwise diversity potential and is explicitly non-i.i.d.\ \citep{corso2024particle}, while SPELL repels trajectories from protected, concurrent, or previously generated images \citep{kirchhof2025shielded}.

\noindent IGA instead provides a single distribution-level regularizer with both training- and sampling-time realizations. It does not define diversity through the active batch or generation history: during sampling, its potential is estimated beforehand, held fixed, and applied independently to each initialized trajectory, yielding i.i.d.\ samples from the approximated IGA target. The entropy wall additionally identifies whether regularization repairs diversity lost during learning or intentionally moves beyond the population data. IGA thus unifies diversity control across end-to-end learning and post-hoc sampling while providing a principled transition from imitation to extrapolation.

\paragraph{Creative generation.}
Creative generation has been pursued through deviation from learned styles, novel composition, rare-region sampling, and repulsion from exemplars. Creative Adversarial Networks depart from established artistic styles \citep{elgammal2017can}; DoodlerGAN recombines object parts into unseen sketches \citep{ge2021creative}; and ConceptLab searches for new category members \citep{richardson2024conceptlab}. For diffusion models, evidence of training-data replication sharpens the distinction between creativity and reproduction \citep{somepalli2023forgery}. Low-density sampling explores rare regions of the learned distribution \citep{sehwag2022lowdensity}, whereas ProCreate pushes generations away from reference images \citep{lu2024procreate}. IGA instead gives creativity a distributional interpretation: crossing the entropy wall produces a target whose representation-relative spectral diversity exceeds that of the population data, while the fidelity term controls departure from the reference. The same definition governs both training and sampling.

\section{Additional Qualitative Results}
We provide additional qualitative comparisons on PixArt-$\Sigma$.
Unless stated otherwise, IGA uses $\lambda=20$. Across prompts, the vanilla model often concentrates on a narrow set of familiar forms or compositions, while IGA produces broader
structural and semantic variation. This is particularly visible in the underwater example, where vanilla PixArt repeatedly generates very similar coral-reef scenes, whereas IGA explores substantially different subjects and
layouts while retaining the requested rendering style.

\label{app:additional_viz}

\begin{figure}[!h]
    \centering
    \includegraphics[width=0.99\linewidth]{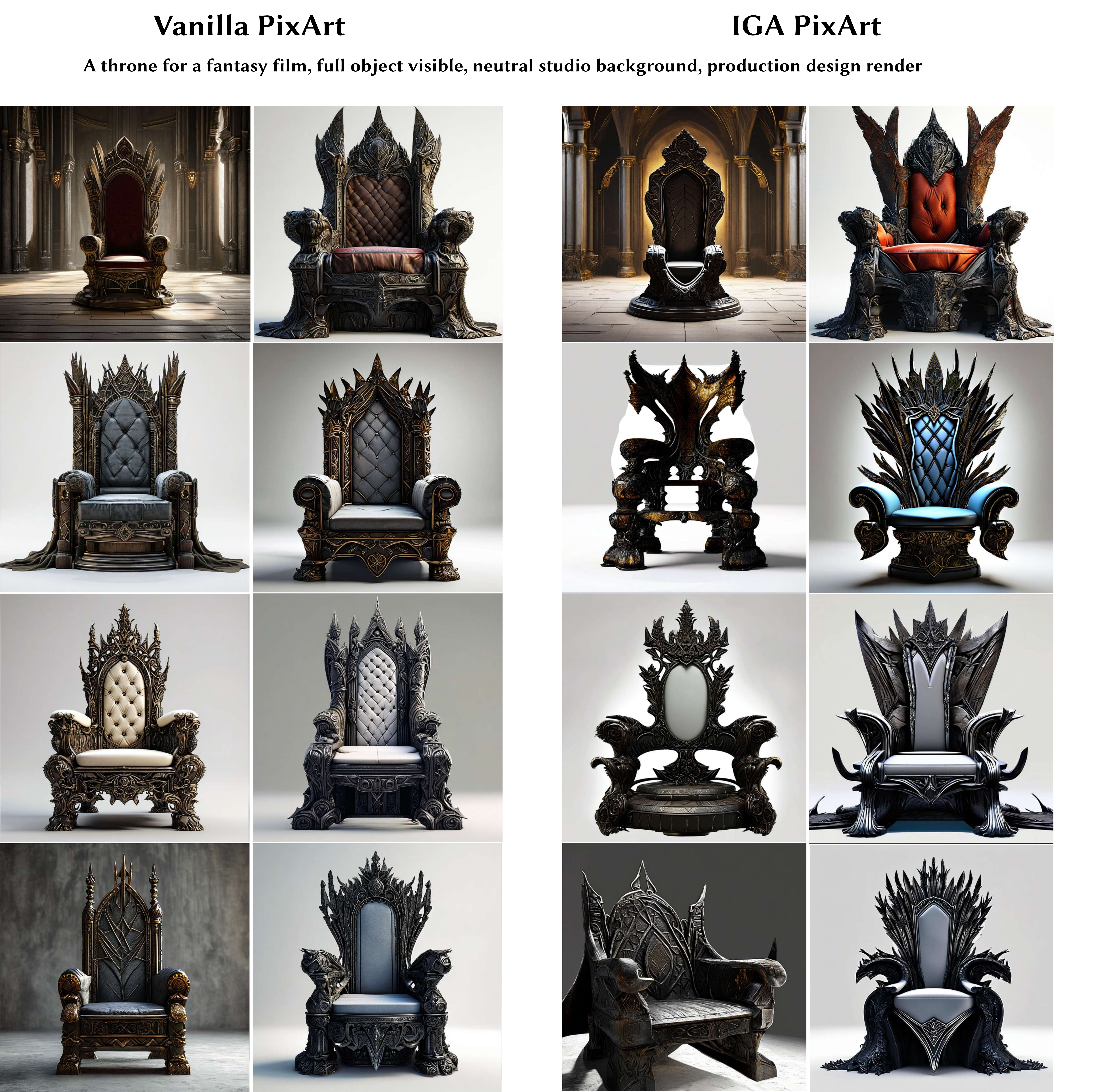}
    \caption{\textbf{Fantasy throne design with PixArt-$\Sigma$.}
    Vanilla PixArt (left) and IGA with $\lambda=20$ (right).
    IGA produces broader variation in silhouette and structure.}
\label{fig:app:throne}
\end{figure}
\begin{figure}
    \centering
    \includegraphics[width=0.99\linewidth]{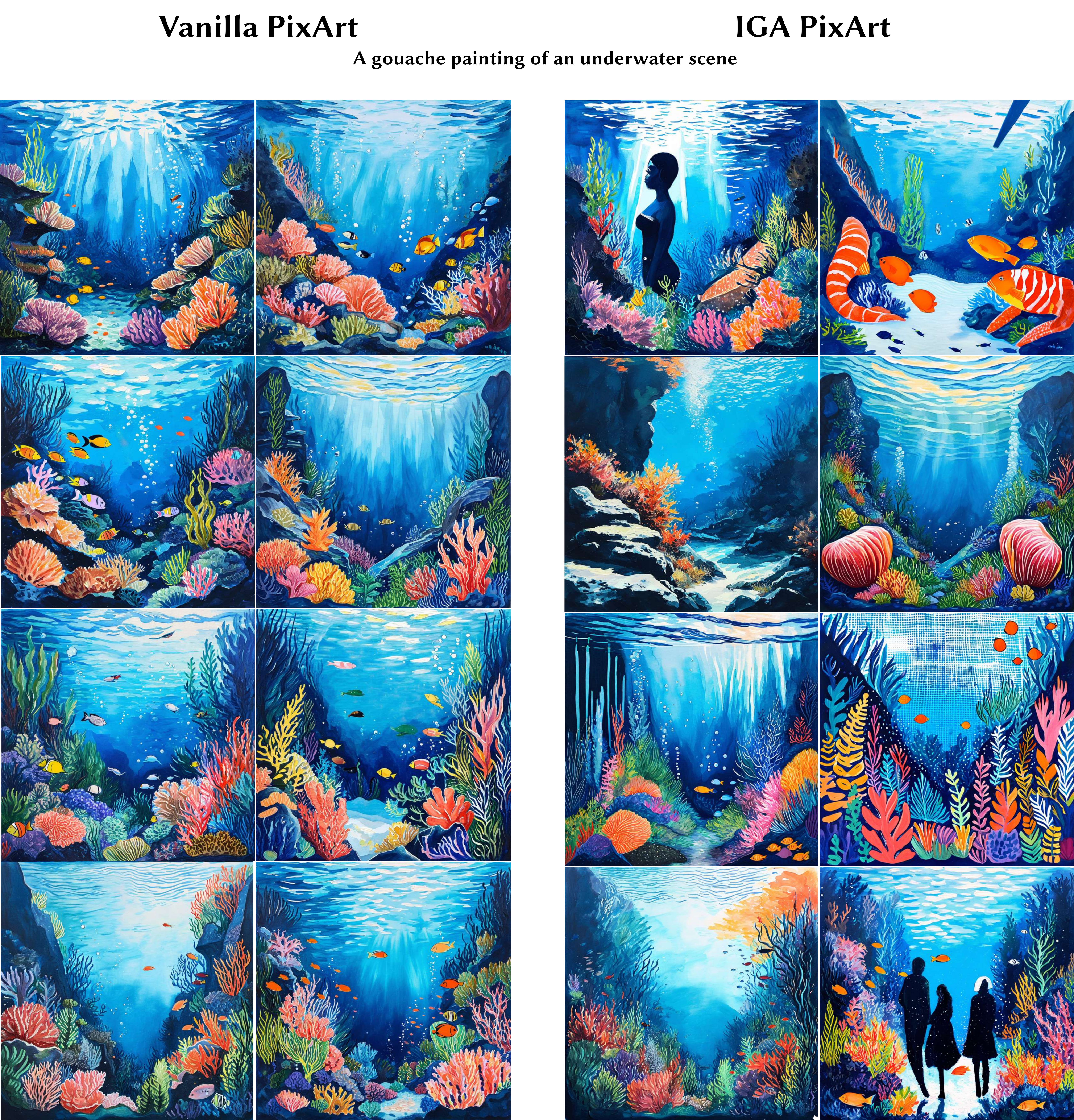}
     \caption{\textbf{Underwater scenes with PixArt-$\Sigma$.}
    Vanilla PixArt (left) largely collapses to the same coral-reef scene across
    seeds. IGA with $\lambda=40$ (right) produces substantially broader variation
    in subjects and composition while preserving the gouache style.}
     \label{fig:app:underwater}
\end{figure}

\begin{figure}
    \centering
    \includegraphics[width=0.99\linewidth]{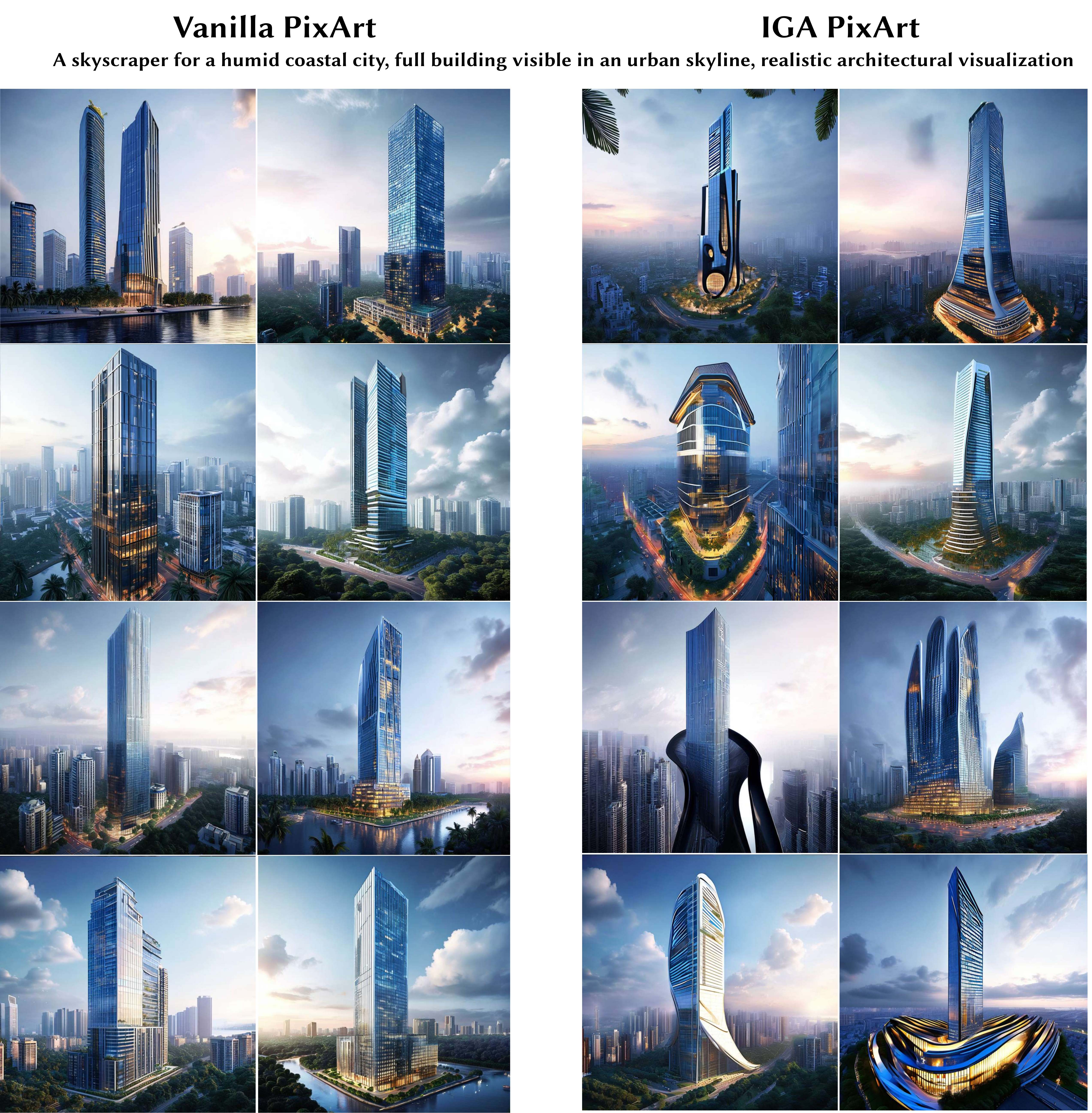}
    \caption{\textbf{Architectural design with PixArt-$\Sigma$.}
    Vanilla PixArt (left) and IGA with $\lambda=20$ (right).
    IGA introduces stronger geometric and structural variation while retaining
    the skyscraper concept.}    
\label{fig:app:sky}
\end{figure}

\begin{figure}
    \centering
    \includegraphics[width=0.99\linewidth]{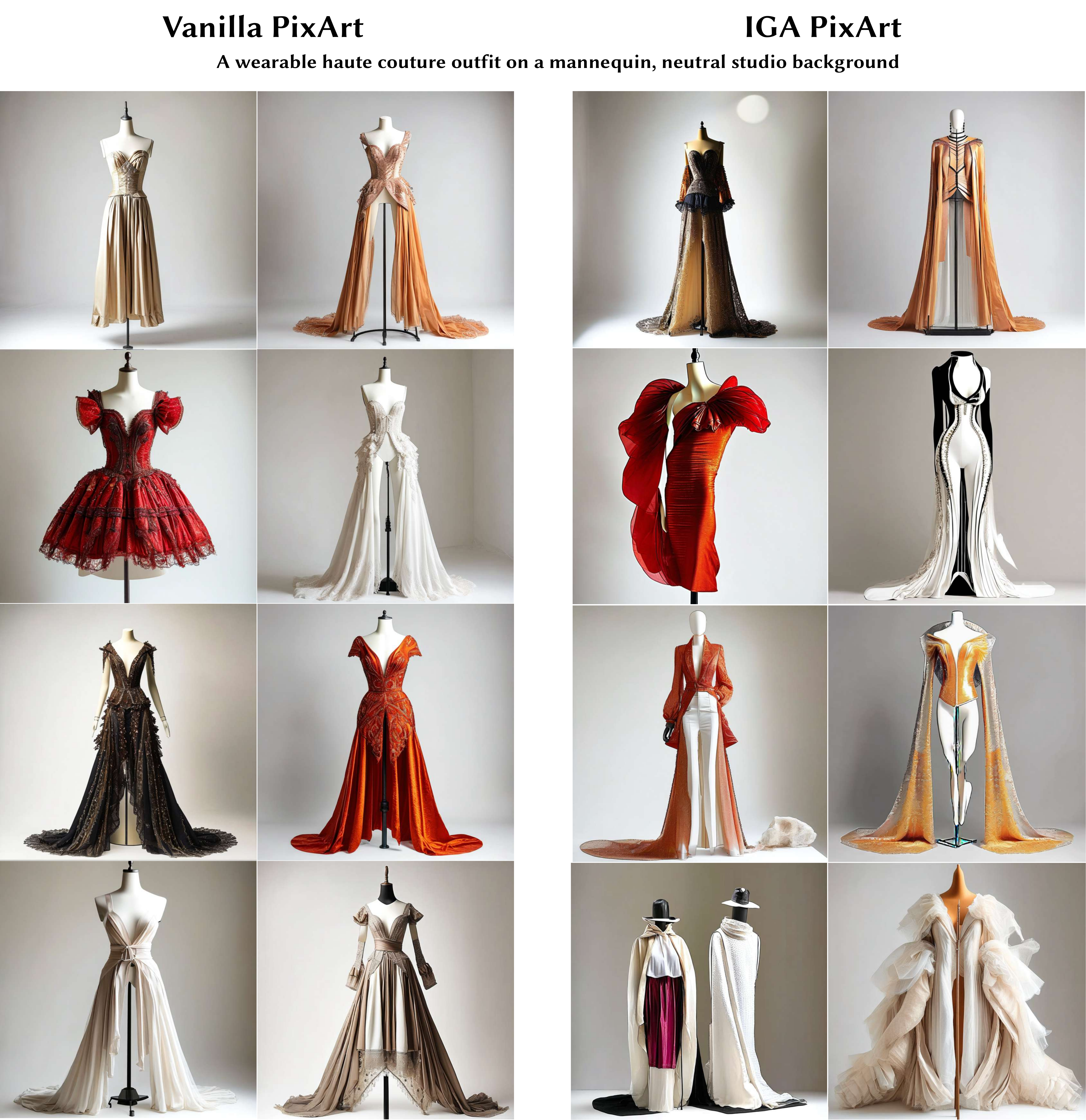}
    \caption{\textbf{Fashion design with PixArt-$\Sigma$.}
    Vanilla PixArt (left) and IGA with $\lambda=20$ (right).
    IGA produces broader variation in garment silhouette, volume, material, and color, while remaining consistent with the wearable haute-couture prompt.}
\label{fig:app:fashion}
\end{figure}

\section{Extended Preliminaries and Conventions}
\label{app:prelim}

We begin by stating two conventions that are used throughout the paper. We would like to further clarify that some of the theoretical lemmas and basic statements are also discussed in \cite{farnia2026exposing}.

\begin{remark}\label{conv:orient}
For any discrepancy $\Dref$ we write $\Dref(Q;\Pref)$ with the optimized distribution $Q$ first and the reference $\Pref$ second. For symmetric discrepancies this is cosmetic. However, for KL and general Bregman divergences the ordering cannot be generally swapped. 

In our notation, the KL anchor is always $\KL(Q\|\Pref)$, and a Bregman anchor is always $D_\Phi(Q,\Pref)$. A data-first object such as $\KL(\widehat P_n\|Q)$ or $D_\Phi(\widehat P_n,Q)$ is a different problem and does not inherit the guarantees below (see also Remark~\ref{rem:repair-scope} and, for the maximum-likelihood setting where the data-first orientation is forced, Proposition~\ref{prop:mle-vae}).
\end{remark}

\begin{remark}\label{conv:ambient}
$\cP$ is assumed to be the convex ambient class of probability measures on $\cX$, used for the convex-analytic theory (concavity of entropy, convex duality, Bregman projection). $\Pgen$ is the possibly nonconvex class realizable by a fixed architecture, used for training. Convexity, strong duality, and Pythagorean statements are proved on $\cP$ and never silently transferred to $\Pgen$; a trained or guided model reaches the ambient optimum only in an approximation-theoretic sense.
\end{remark}

\begin{assumption}[Finite-dimensional embedding]\label{ass:finite-dim}
Unless stated otherwise, $\phi:\cX\to\R^d$ with $\norm{\phi(x)}_2=1$ is measurable, and $\Sigma_Q\in\R^{d\times d}$ with $\Sigma_Q\succeq0$ and $\Tr(\Sigma_Q)=1$. When weak continuity of $Q\mapsto\Sigma_Q$ is invoked, we further assume $\cX$ is Polish and $\phi$ is bounded and continuous.
\end{assumption}

\begin{remark}[Use of entropy symbol $H$]\label{rem:two-H}
When a result holds for either entropy functional we let $H$ denote a fixed but arbitrary choice from $\{\Hz,\He\}$ (as in \eqref{eq:Crho}, \eqref{eq:Plambda}, Theorem~\ref{thm:mono}, and Definition~\ref{def:wall}); $H$ is never used to mix the two within a single statement. When differentiability of the entropy is needed we specialize to the smoothed $\He$ (Section~\ref{sec:guidance}), and when defining the population wall in its main statistical interpretation we use the unsmoothed $\Hz$ (Section~\ref{sec:wall}). Lemma~\ref{lem:smoothing} controls the gap between the two.
\end{remark}

\begin{lemma}[Concavity of $\Hz$ and $\He$ in $Q$]\label{lem:concave}
If $Q\mapsto\Sigma_Q$ is affine, then $Q\mapsto\Hz(Q)$ and $Q\mapsto\He(Q)$ are concave.
\end{lemma}
\begin{proof}
Note that $S\mapsto-\Tr(S\log S)$ is a concave functional of density matrices (unit-trace PSD matrices), and a concave function composed with an affine map is concave. This proves concavity of $\Hz$. 

For $\He$, the map $Q\mapsto S_Q^\varepsilon=(1-\varepsilon)\Sigma_Q+\varepsilon \tfrac{1}{d}I_d$ is also affine in $Q$, and therefore the same argument applies.
\end{proof}

\begin{lemma}[Bounding Smoothing Gap of Spectral Entropy]\label{lem:smoothing}
Let $d\ge2$, let $\Sigma$ be a $d\times d$ density matrix, and set $S=(1-\varepsilon)\Sigma+\varepsilon \tfrac{1}{d}I_d$ for $\varepsilon\in[0,1]$. Writing $\tau=\varepsilon(1-\frac{1}{d})$ and $h_2(t)=-t\log t-(1-t)\log(1-t)$,
\[
    \bigl|\mHe(S)-\mHe(\Sigma)\bigr|\le\tau\log(d-1)+h_2(\tau).
\]
Since $\He(Q)=\mHe(S_Q^\varepsilon)$ and $\Hz(Q)=\mHe(\Sigma_Q)$, this implies the following for every $Q$:
\[
    \bigl|\He(Q)-\Hz(Q)\bigr|\le\tau\log(d-1)+h_2(\tau).
\]
In particular, $|\rho_{\star,\He}-\rho_\star|$ admits the same bound. The bound is achieved when $\Sigma$ is rank one: then $\tfrac12\norm{S-\Sigma}_1=\tau$ exactly, $S$ has eigenvalues $1-\tau,\ \tfrac{\tau}{d-1},\dots,\tfrac{\tau}{d-1}$, and $\bigl|\mHe(S)-\mHe(\Sigma)\bigr|=\tau\log(d-1)+h_2(\tau)$.
\end{lemma}
\begin{proof}
First, we bound the trace distance between $S$ and $\Sigma$. Since $S-\Sigma=\varepsilon(\tfrac{1}{d}I_d-\Sigma)$, we have $\tfrac12\norm{S-\Sigma}_1=\varepsilon\cdot\tfrac12\norm{\tfrac{1}{d}I_d-\Sigma}_1$. If $\Sigma$ has eigenvalues $p_i$, then $\tfrac12\norm{\tfrac{1}{d}I_d-\Sigma}_1=\tfrac12\sum_i|p_i-\frac{1}{d}|$, which over the probability simplex is maximized at a vertex $p=e_j$, giving $1-\frac{1}{d}$. Hence
\[
    \tfrac12\norm{S-\Sigma}_1\le\varepsilon(1-\frac{1}{d})=\tau.
\]
Next, we apply the Fannes--Audenaert inequality, showing that for $d\times d$ density matrices $A,B$ with $\tfrac12\norm{A-B}_1\le t\le 1-\frac{1}{d}$,
\[
    |\mHe(A)-\mHe(B)|\le t\log(d-1)+h_2(t).
\]
The right-hand side is non-decreasing in $t$ on $[0,1-\frac{1}{d}]$: its derivative $\log(d-1)+\log\frac{1-t}{t}$ is non-negative there, vanishing only at $t=1-\frac{1}{d}$. Applying the inequality at $t=\tfrac12\norm{S-\Sigma}_1\le\tau\le1-\frac{1}{d}$ (here $d\ge2$ ensures $\log(d-1)\ge0$) gives the claim with $A=S$ and $B=\Sigma$. Finally, the consequence for $\He,\Hz$ follows by the stated identities, and the wall bound follows by taking $Q=\Pdata$. For rank-one $\Sigma=vv^\top$, the eigenvalues of $\tfrac{1}{d}I_d-\Sigma$ are $\tfrac1d-1$ (once) and $\tfrac1d$ (with multiplicity $d-1$), so $\tfrac12\norm{S-\Sigma}_1=\tau$ exactly; the spectrum of $S$ is then $\bigl(1-\tau,\ \tfrac{\tau}{d-1},\dots,\tfrac{\tau}{d-1}\bigr)$, whence $\mHe(\Sigma)=0$ and $\mHe(S)=\tau\log(d-1)+h_2(\tau)$, so equality holds.
\end{proof}

The last ingredient is an elementary identity for exponential tilts. For a \emph{fixed} bounded reward $G$, it identifies the minimizer of the KL-anchored linear objective in closed form; the sampling-time tilt of Theorem~\ref{thm:tilt} is its self-consistent analogue, in which $G$ is the entropy energy evaluated at the optimum itself.

\begin{lemma}[Elementary Gibbs identity]\label{lem:gibbs}
Let $P$ be a probability law and $G$ measurable with $Z_G=\E_P[e^{G}]<\infty$; define $P^{G}(dx)=Z_G^{-1}e^{G(x)}P(dx)$. For every $Q\ll P$,
\begin{equation}\label{eq:gibbs-decomposition}
    \KL(Q\|P)-\E_Q[G]=\KL(Q\|P^{G})-\log Z_G,
\end{equation}
so $P^{G}$ is the unique minimizer over $\{Q\ll P\}$ of the left-hand side whenever it is finite.
\end{lemma}
\begin{proof}
On $\{Q\ll P\}$ we have $\log(dQ/dP^{G})=\log(dQ/dP)-G+\log Z_G$. Integrating against $Q$ gives
\[
    \KL(Q\|P^{G})=\KL(Q\|P)-\E_Q[G]+\log Z_G,
\]
which is \eqref{eq:gibbs-decomposition}. The left-hand side equals $\KL(Q\|P^{G})-\log Z_G$, minimized (over $Q\ll P$, equivalently $Q\ll P^{G}$ since the two are equivalent) uniquely at $Q=P^{G}$, where $\KL=0$.
\end{proof}

\section{The Constrained--Penalized Correspondence, the Spectral Game, and Training}
\label{app:formulation}

This appendix proves the results of Section~\ref{sec:formulation}. We first state and prove the constrained--penalized correspondence invoked in Section~\ref{subsec:objectives}, together with the monotone regularization path; Appendix~\ref{app:spectral-minmax} then proves the spectral min--max representation, and Appendix~\ref{app:training} the training-time instantiations of Section~\ref{subsec:training}.

\begin{theorem}[Constrained--penalized correspondence]\label{thm:dual}
Suppose that:
\begin{enumerate}[leftmargin=*]
    \item[\textup{(i)}] $\cP$ is a nonempty compact convex subset of a locally convex Hausdorff space of finite signed measures;
    \item[\textup{(ii)}] $Q\mapsto\Dref(Q;\Pref)$ is proper, convex, and lower semicontinuous on $\cP$;
    \item[\textup{(iii)}] $Q\mapsto\Sigma_Q$ is affine and continuous, so that $H\in\{\Hz,\He\}$ is concave (Lemma~\ref{lem:concave}) and upper semicontinuous;
    \item[\textup{(iv)}] \emph{(Slater condition)} there exists $\bar Q\in\cP$ with $\Dref(\bar Q;\Pref)<\infty$ and $H(\bar Q)>\rho$.
\end{enumerate}
Then the following hold.
\begin{enumerate}[leftmargin=*]
    \item[\textup{(a)}] \emph{Attainment.} The feasible set $\cP\cap\{H\ge\rho\}$ is nonempty and compact, and both the constrained minimum in \eqref{eq:Crho} and the inner minimum defining $F_\lambda$ are attained.
    \item[\textup{(b)}] \emph{Strong duality.}
    \begin{equation}\label{eq:dual}
        \min_{\substack{Q\in\cP\\H(Q)\ge\rho}}\Dref(Q;\Pref)
        \;=\;
        \max_{\lambda\ge0}\;
        \Bigl\{\min_{Q\in\cP}\bigl\{\Dref(Q;\Pref)-\lambda H(Q)\bigr\}
        +\lambda\rho\Bigr\}.
    \end{equation}
    \item[\textup{(c)}] \emph{Optimal multiplier.} There exists $\lambda^\star\ge0$ such that every solution $Q^\star$ of \eqref{eq:Crho} minimizes $F_{\lambda^\star}$ and satisfies the complementary-slackness identity $\lambda^\star\bigl(\rho-H(Q^\star)\bigr)=0$.
\end{enumerate}
\end{theorem}

Theorem~\ref{thm:dual} justifies replacing \eqref{eq:Crho} by \eqref{eq:Plambda} at the specific multiplier $\lambda^\star$ dual to $\rho$; it does \emph{not} claim that every $\lambda\ge0$ corresponds to a user-chosen target level. The proof proceeds through the value function of the constrained problem, after recording an unconditional min--max identity (Proposition~\ref{prop:primal-minmax}).

\begin{proposition}[Exact primal min--max identity]\label{prop:primal-minmax}
For any feasible set $\cQ$ and arbitrary functionals $J,H$,
\[
    \inf_{\substack{Q\in\cQ\\H(Q)\ge\rho}}J(Q)
    =
    \inf_{Q\in\cQ}\sup_{\lambda\ge0}
    \bigl\{J(Q)+\lambda(\rho-H(Q))\bigr\}.
\]
\end{proposition}
\begin{proof}
For fixed $Q$, $\sup_{\lambda\ge0}\{J(Q)+\lambda(\rho-H(Q))\}$ equals $J(Q)$ if $H(Q)\ge\rho$ (the coefficient of $\lambda$ is nonpositive, so the supremum is at $\lambda=0$) and $+\infty$ if $H(Q)<\rho$ (the coefficient is positive, so the expression diverges as $\lambda\to\infty$). Taking the infimum over $Q\in\cQ$ retains only feasible $Q$ and reproduces the constrained value.
\end{proof}

\begin{proof}[Proof of Theorem~\ref{thm:dual}]
By hypotheses (ii)--(iii), $J(Q)=\Dref(Q;\Pref)$ is proper, convex, l.s.c.\ on the convex set $\cP$, and $H\in\{\Hz,\He\}$ is concave (Lemma~\ref{lem:concave}) and u.s.c., so the feasible set $\cP\cap\{H\ge\rho\}$ is convex and closed.

First, we verify feasibility and attainment. By Slater (iv) the feasible set contains $\bar Q$, hence is nonempty; it is a closed subset of the compact $\cP$ (i), hence compact. A l.s.c.\ function attains its minimum on a nonempty compact set, so the constrained minimum in \eqref{eq:Crho} is attained; likewise, for each $\lambda\ge0$, $Q\mapsto J(Q)-\lambda H(Q)$ is l.s.c.\ on the compact $\cP$ and attains its minimum, so $F_\lambda$ has a minimizer and the displayed $\min$'s in \eqref{eq:dual} are justified.

The remainder of the proof runs through the value function of the constrained problem,
\[
    v(r)=\inf_{Q\in\cP}\bigl\{J(Q):H(Q)\ge r\bigr\},
\]
with $v(r)=+\infty$ if no feasible $Q$ exists; by the previous paragraph, $v(\rho)$ is finite and attained. Note also that $J$, being l.s.c.\ on the compact $\cP$, is bounded below on $\cP$, so $v(r)\ge\inf_{\cP}J>-\infty$ for every $r$.

Next, we establish the two structural properties of $v$. The value function is nondecreasing: if $r_1\le r_2$ then $\{H\ge r_2\}\subseteq\{H\ge r_1\}$, so the infimum over the smaller set is at least as large, i.e.\ $v(r_1)\le v(r_2)$. The value function is also convex. To see this, fix $r_1,r_2\in\R$, $\theta\in[0,1]$, and $\eta>0$, and choose feasible $Q_i$ (that is, $H(Q_i)\ge r_i$) with $J(Q_i)\le v(r_i)+\eta$. The mixture $Q_\theta=\theta Q_1+(1-\theta)Q_2\in\cP$ then satisfies $H(Q_\theta)\ge\theta r_1+(1-\theta)r_2$ by concavity of $H$, and, by convexity of $J$,
\[
    J(Q_\theta)\le\theta J(Q_1)+(1-\theta)J(Q_2)
    \le\theta v(r_1)+(1-\theta)v(r_2)+\eta.
\]
Hence $v(\theta r_1+(1-\theta)r_2)\le\theta v(r_1)+(1-\theta)v(r_2)+\eta$, and letting $\eta\downarrow0$ gives convexity.

Then, we show that $v$ is subdifferentiable at the target level $\rho$. By the Slater condition there is $\bar Q\in\cP$ with $J(\bar Q)<\infty$ and $H(\bar Q)>\rho$; hence $v(r)\le J(\bar Q)<\infty$ for all $r\le H(\bar Q)$. Combined with the lower bound above, $v$ is finite on $(-\infty,H(\bar Q)]$, an interval whose interior contains $\rho$. A finite convex function on an open interval is subdifferentiable at every interior point; pick $\lambda^\star\in\partial v(\rho)$. Since $v$ is nondecreasing, $\lambda^\star\ge0$.

With the multiplier $\lambda^\star$ in hand, we can prove strong duality. The subgradient inequality gives, for every $Q\in\cP$,
\[
    J(Q)\ge v(H(Q))\ge v(\rho)+\lambda^\star\bigl(H(Q)-\rho\bigr),
\]
hence $J(Q)-\lambda^\star H(Q)\ge v(\rho)-\lambda^\star\rho$; taking the infimum over $Q\in\cP$ yields $\inf_{Q\in\cP}\{J-\lambda^\star H\}+\lambda^\star\rho\ge v(\rho)$. Conversely, weak duality holds: for any $\lambda\ge0$ and any feasible $Q$ (that is, $H(Q)\ge\rho$),
\[
    J(Q)\ge J(Q)-\lambda\bigl(H(Q)-\rho\bigr)
    \ge\inf_{Q'\in\cP}\{J-\lambda H\}+\lambda\rho,
\]
and taking the infimum over feasible $Q$ gives $v(\rho)\ge\sup_{\lambda\ge0}\{\inf_{Q'}\{J-\lambda H\}+\lambda\rho\}$. The two inequalities together yield \eqref{eq:dual}, with the outer supremum attained at $\lambda^\star$.

Finally, we establish complementary slackness. Let $Q^\star$ solve \eqref{eq:Crho}. Feasibility gives $H(Q^\star)\ge\rho$, and by strong duality
\[
    J(Q^\star)=v(\rho)=\inf_{Q}\{J-\lambda^\star H\}+\lambda^\star\rho.
\]
On the one hand, $J(Q^\star)-\lambda^\star H(Q^\star)\ge\inf_Q\{J-\lambda^\star H\}=v(\rho)-\lambda^\star\rho$. On the other hand, feasibility and $\lambda^\star\ge0$ give $J(Q^\star)-\lambda^\star H(Q^\star)\le J(Q^\star)-\lambda^\star\rho=v(\rho)-\lambda^\star\rho$. The two bounds match, forcing $\lambda^\star\bigl(H(Q^\star)-\rho\bigr)=0$ and $J(Q^\star)-\lambda^\star H(Q^\star)=\inf_Q\{J-\lambda^\star H\}$, i.e.\ $Q^\star$ minimizes $F_{\lambda^\star}$.
\end{proof}

\begin{remark}[Which hypotheses do what]\label{rem:attainment}
The hypotheses of Theorem~\ref{thm:dual} play three separable roles. \emph{Duality.} The value-function argument shows that strong duality and the existence of an optimal multiplier $\lambda^\star\in\partial v(\rho)$ require only convexity of $\cP$, convex l.s.c.\ $J$, concave $H$, the Slater condition (iv), and the value function $v$ being proper and finite near $\rho$, i.e.\ $v(\rho)>-\infty$ (equivalently, $J$ bounded below on the feasible set), in addition to $v(\rho)<\infty$ from Slater. Convexity and Slater alone do not guarantee $v(\rho)>-\infty$: if $J$ is unbounded below on $\cP$ then $v\equiv-\infty$, no finite subgradient exists, and the duality statement is vacuous. In the compact setting of Theorem~\ref{thm:dual} this cannot happen, because an l.s.c.\ $J$ on the compact $\cP$ is bounded below; the properness caveat matters only in the noncompact variant. \emph{Attainment.} The compactness in (i), with l.s.c.\ $J$ and u.s.c.\ $H$, is otherwise used only to guarantee attainment of the constrained minimum, of the inner minima defining $F_\lambda$, and hence of the displayed $\min/\max$ in \eqref{eq:dual}. \emph{Noncompact classes.} In settings where $\{Q\ll P_\theta\}$ is convex but not compact (Section~\ref{sec:guidance}), existence and attainment are instead obtained by the direct method of Theorem~\ref{thm:tilt}, where $J=\KL(\cdot\|P_\theta)\ge0$ is automatically bounded below.
\end{remark}

\begin{proposition}[Saddle representation]\label{prop:saddle}
Under the hypotheses of Theorem~\ref{thm:dual}, with $Q^\star$ a solution of \eqref{eq:Crho} and $\lambda^\star$ the optimal multiplier, $(Q^\star,\lambda^\star)$ is a saddle point of $\mathcal L(Q,\lambda)=J(Q)+\lambda(\rho-H(Q))$ on $\cP\times[0,\infty)$:
\[
    \mathcal L(Q^\star,\lambda)\le\mathcal L(Q^\star,\lambda^\star)
    \le\mathcal L(Q,\lambda^\star)
    \qquad\forall Q\in\cP,\ \lambda\ge0,
\]
and consequently $\inf_{Q\in\cP}\sup_{\lambda\ge0}\mathcal L(Q,\lambda)=\sup_{\lambda\ge0}\inf_{Q\in\cP}\mathcal L(Q,\lambda)$.
\end{proposition}
\begin{proof}
Since $Q^\star$ minimizes $F_{\lambda^\star}=J-\lambda^\star H$ over $\cP$ (Theorem~\ref{thm:dual}) and $\mathcal L(\cdot,\lambda^\star)=F_{\lambda^\star}(\cdot)+\lambda^\star\rho$, the right inequality $\mathcal L(Q^\star,\lambda^\star)\le\mathcal L(Q,\lambda^\star)$ holds for all $Q\in\cP$. For the left inequality, $\mathcal L(Q^\star,\lambda)=J(Q^\star)+\lambda(\rho-H(Q^\star))$ is nonincreasing in $\lambda\ge0$ because $\rho-H(Q^\star)\le0$; together with $\lambda^\star(\rho-H(Q^\star))=0$ this gives $\mathcal L(Q^\star,\lambda)\le\mathcal L(Q^\star,\lambda^\star)$ for all $\lambda\ge0$. The equality of the two mixed extrema is the standard consequence of a saddle point.
\end{proof}

Even without convexity, the multiplier acts as a monotone control on global minimizers; the following statement, summarized in Section~\ref{subsec:objectives}, applies both to the ambient problem and to a nonconvex generator family.

\begin{theorem}[Monotone regularization path]\label{thm:mono}
Let $\cQ\in\{\cP,\Pgen\}$, write $J(Q)=\Dref(Q;\Pref)$, and suppose that a minimizer $Q_\lambda\in\argmin_{Q\in\cQ}F_\lambda(Q)$ exists for every $\lambda\ge0$. Then, for any $0\le\lambda_1<\lambda_2$ and any choices of minimizers $Q_{\lambda_1}$ and $Q_{\lambda_2}$,
\[
    H(Q_{\lambda_2})\ge H(Q_{\lambda_1}),
    \qquad
    J(Q_{\lambda_2})\ge J(Q_{\lambda_1}).
\]
\end{theorem}

\begin{proof}[Proof of Theorem~\ref{thm:mono}]
Let $0\le\lambda_1<\lambda_2$ and let $Q_{\lambda_1},Q_{\lambda_2}$ be any minimizers of $F_{\lambda_1},F_{\lambda_2}$ over $\cQ$. Abbreviate $H_i=H(Q_{\lambda_i})$, $J_i=J(Q_{\lambda_i})$. Optimality of $Q_{\lambda_1}$ at $\lambda_1$ and of $Q_{\lambda_2}$ at $\lambda_2$ gives
\[
    J_1-\lambda_1H_1\le J_2-\lambda_1H_2,
    \qquad
    J_2-\lambda_2H_2\le J_1-\lambda_2H_1.
\]
Adding these two inequalities cancels $J_1,J_2$ and yields $(\lambda_2-\lambda_1)(H_2-H_1)\ge0$, so $H_2\ge H_1$. Substituting into the first inequality, rearranged as $J_1-J_2\le\lambda_1(H_1-H_2)\le0$, gives $J_2\ge J_1$. No convexity, differentiability, uniqueness, or path continuity is used; the argument is valid for any selection of minimizers, so the ordering holds even when minimizers are nonunique.
\end{proof}

\begin{proposition}[Every penalized optimizer is a constrained optimizer at its attained level]\label{prop:attained}
Under the hypotheses of Theorem~\ref{thm:mono}, set $\rho_\lambda\defeq H(Q_\lambda)$. Then $Q_\lambda\in\underset{Q\in\cQ}{\arg\!\min}\:\bigl\{J(Q):H(Q)\ge\rho_\lambda\bigr\}$.
\end{proposition}
\begin{proof}
Take any $Q\in\cQ$ with $H(Q)\ge\rho_\lambda$. Penalized optimality gives $J(Q_\lambda)-\lambda H(Q_\lambda)\le J(Q)-\lambda H(Q)$, so
\[
    J(Q_\lambda)\le J(Q)+\lambda(H(Q_\lambda)-H(Q))
    =J(Q)-\lambda(H(Q)-\rho_\lambda)\le J(Q),
\]
using $\lambda\ge0$ and $H(Q)\ge\rho_\lambda$.
\end{proof}

\begin{remark}[Constrained and penalized problems are not interchangeable in general]\label{rem:nonequiv}
Outside the convex setting the two problems need not share solutions for a prescribed $\lambda$ or $\rho$. Proposition~\ref{prop:attained} matches each $Q_\lambda$ to its \emph{own} attained level $\rho_\lambda$, while Theorem~\ref{thm:dual} recovers a \emph{prescribed} level $\rho$ only under its convexity and Slater hypotheses. On a nonconvex $\Pgen$ the saddle representation (Proposition~\ref{prop:saddle}) can fail with a positive duality gap; we therefore do not claim per-$\lambda$ constrained--penalized equivalence for $\Pgen$.
\end{remark}

\subsection{The spectral min--max representation: proofs}
\label{app:spectral-minmax}

This appendix proves Proposition~\ref{prop:spectral-minmax} and the best-response identity quoted in Sections~\ref{subsec:objectives} and~\ref{subsec:training}. The main tool is the Gibbs variational principle for matrix entropy, which we derive from Klein's inequality; both results are proved in full.

\begin{lemma}[Klein's inequality for matrix relative entropy]\label{lem:klein}
Let $S,T$ be $d\times d$ density matrices with $T\succ0$. Then
\[
    \Tr\bigl(S\log S-S\log T\bigr)\ge\Tr(S)-\Tr(T)=0,
\]
with equality if and only if $S=T$.
\end{lemma}
\begin{proof}
Write spectral decompositions $S=\sum_i\alpha_iu_iu_i^\top$ and $T=\sum_j\beta_jv_jv_j^\top$ with orthonormal bases $(u_i),(v_j)$, eigenvalues $\alpha_i\ge0$, $\beta_j>0$, and set $c_{ij}\defeq(u_i^\top v_j)^2$. The matrix $(c_{ij})$ is doubly stochastic: $\sum_jc_{ij}=\norm{u_i}^2=1$ and $\sum_ic_{ij}=\norm{v_j}^2=1$, since each basis is orthonormal. Expanding the traces in these bases,
\[
    \Tr(S\log S)=\sum_i\alpha_i\log\alpha_i
    =\sum_{i,j}c_{ij}\,\alpha_i\log\alpha_i,
    \qquad
    \Tr(S\log T)=\sum_{i,j}c_{ij}\,\alpha_i\log\beta_j,
\]
using $\sum_jc_{ij}=1$ for the first identity and $u_i^\top(\log T)u_i=\sum_jc_{ij}\log\beta_j$ for the second. The scalar inequality $x\log x-x\log y\ge x-y$, valid for $x\ge0$, $y>0$ (with $0\log0=0$; it is the tangent-line inequality for the convex function $x\mapsto x\log x$ at $y$), holds with equality if and only if $x=y$: for $x>0$ this is strict convexity, and at $x=0$ the inequality reads $0\ge-y$, strict since $y>0$. Applying it termwise,
\[
    \Tr\bigl(S\log S-S\log T\bigr)
    =\sum_{i,j}c_{ij}\bigl(\alpha_i\log\alpha_i-\alpha_i\log\beta_j\bigr)
    \ge\sum_{i,j}c_{ij}(\alpha_i-\beta_j)
    =\sum_i\alpha_i-\sum_j\beta_j=0,
\]
where the last step again uses double stochasticity. If equality holds, then every pair $(i,j)$ with $c_{ij}>0$ satisfies $\alpha_i=\beta_j$. Fix $j$ and expand $v_j=\sum_i(u_i^\top v_j)u_i$; then
\[
    Sv_j=\sum_i\alpha_i(u_i^\top v_j)\,u_i
    =\sum_i\beta_j(u_i^\top v_j)\,u_i=\beta_jv_j,
\]
since every index $i$ contributing a nonzero coefficient has $c_{ij}>0$, hence $\alpha_i=\beta_j$. Thus $S$ acts as $\beta_j$ on each $v_j$, so $S=\sum_j\beta_jv_jv_j^\top=T$. Conversely $S=T$ gives equality trivially.
\end{proof}

\begin{lemma}[Gibbs variational principle for matrix entropy]\label{lem:gibbs-matrix}
For every $d\times d$ density matrix $S$,
\[
    \sup_{\Theta=\Theta^\top}
    \bigl\{-\Tr(S\Theta)-\log\Tr(e^{-\Theta})\bigr\}
    =-\mHe(S),
\]
where the supremum runs over all symmetric $d\times d$ matrices. The objective is invariant under $\Theta\mapsto\Theta+cI_d$ for $c\in\R$. If $S\succ0$, the supremum is attained exactly at the family $\Theta=-\log S+cI_d$, $c\in\R$; if $S$ is singular, the supremum is not attained, but is approached along $\Theta_\delta=-\log(S+\delta I_d)$ as $\delta\downarrow0$, whose objective value is $\sum_i\alpha_i\log(\alpha_i+\delta)-\log(1+\delta d)$ in terms of the eigenvalues $\alpha_i$ of $S$.
\end{lemma}
\begin{proof}
First, the invariance: replacing $\Theta$ by $\Theta+cI_d$ changes $-\Tr(S\Theta)$ by $-c\,\Tr(S)=-c$ and changes $-\log\Tr(e^{-\Theta-cI})=-\log(e^{-c}\Tr(e^{-\Theta}))$ by $+c$, so the objective is unchanged.

Next, the upper bound. For symmetric $\Theta$, let $R_\Theta\defeq e^{-\Theta}/\Tr(e^{-\Theta})$, a positive-definite density matrix with $\log R_\Theta=-\Theta-\log\Tr(e^{-\Theta})\,I_d$. Klein's inequality (Lemma~\ref{lem:klein}) with $T=R_\Theta$ gives $\Tr(S\log S)-\Tr(S\log R_\Theta)\ge0$, which expands to
\[
    0\le\Tr(S\log S)+\Tr(S\Theta)+\log\Tr(e^{-\Theta}),
\]
i.e.\ $-\Tr(S\Theta)-\log\Tr(e^{-\Theta})\le-\mHe(S)$, with equality if and only if $R_\Theta=S$.

Then, attainment. If $S\succ0$, the equation $R_\Theta=S$ has the solutions $\Theta=-\log S+cI_d$, $c\in\R$, and no others: $R_\Theta=S$ forces $-\Theta=\log S+\log\Tr(e^{-\Theta})I_d$. If $S$ is singular, no symmetric $\Theta$ satisfies $R_\Theta=S$, since $R_\Theta\succ0$ always; hence the supremum is not attained. Finally, evaluating the objective at $\Theta_\delta=-\log(S+\delta I_d)$ gives
\begin{align*}
    -\Tr(S\Theta_\delta)-\log\Tr(e^{-\Theta_\delta})
    &=\Tr\bigl(S\log(S+\delta I_d)\bigr)-\log\Tr(S+\delta I_d)\\
    &=\sum_i\alpha_i\log(\alpha_i+\delta)-\log(1+\delta d),
\end{align*}
which converges to $\sum_i\alpha_i\log\alpha_i=-\mHe(S)$ as $\delta\downarrow0$ (the terms with $\alpha_i=0$ contribute $0\cdot\log\delta=0$). Hence the supremum equals $-\mHe(S)$ in all cases.
\end{proof}

\begin{remark}[Smoothed versus unsmoothed dual, and concavity as a byproduct]\label{rem:hz-dual}
Lemma~\ref{lem:gibbs-matrix} explains why the min--max form is stated for the \emph{smoothed} entropy. For the unsmoothed $\Hz$, the covariance $\Sigma_Q$ can be singular, in which case the supremum is not attained, and the near-maximizers $\Theta_\delta=-\log(\Sigma_Q+\delta I_d)$ have operator norm growing like $\log(1/\delta)$: no compact adversary class captures the supremum uniformly over all $Q$. The spectral floor $S_Q^\varepsilon\succeq(\varepsilon/d)I_d$ removes both obstructions: it confines the best response to the compact class $\Tset$ of Proposition~\ref{prop:spectral-minmax} and guarantees attainment. The lemma also yields an independent proof of fact (a) of the organization paragraph: it displays $-\mHe(S)$ as a supremum of affine functions of $S$, hence convex, so $\mHe$ is concave, which also yields Lemma~\ref{lem:concave}.
\end{remark}

\begin{remark}[Further readings of the min--max form]\label{rem:minmax-readings}
Two structural readings of \eqref{eq:IGA-minmax} complement the closed-form best response discussed in Section~\ref{subsec:objectives}. \emph{Linearization:} for fixed $\Theta$, the inner objective depends on $Q$ only through the expectation of the per-sample payoff $\phi(x)^\top\Theta\,\phi(x)$, so the distribution-level reward $-\lambda\He(Q)$ becomes an ordinary expected loss at the cost of one $d\times d$ symmetric adversarial variable; Section~\ref{subsec:training} exploits this directly, where the spectral player joins the discriminator as a second adversary that admits a closed-form best response. \emph{Why smoothing:} for the unsmoothed entropy the supremum runs over an unbounded matrix class and is not attained at rank-deficient covariances (Remark~\ref{rem:hz-dual}), whereas the spectral floor $\varepsilon \tfrac{1}{d}I_d$ confines the adversary to the compact class $\Tset$ and guarantees attainment; the interchange in part (iii) then follows from Sion's minimax theorem, whose compactness requirement is satisfied by the $\Theta$-side alone.
\end{remark}

\begin{proof}[Proof of Proposition~\ref{prop:spectral-minmax}]
First, part (i). Since $S_Q^\varepsilon\succeq(\varepsilon/d)I_d\succ0$, Lemma~\ref{lem:gibbs-matrix} gives
\[
    -\He(Q)=-\mHe(S_Q^\varepsilon)
    =\sup_{\Theta=\Theta^\top}
    \bigl\{-\Tr(S_Q^\varepsilon\Theta)-\log\Tr(e^{-\Theta})\bigr\},
\]
attained exactly at the family $-\log S_Q^\varepsilon+cI_d$. Imposing $\Tr(\Theta)=0$ pins the constant at $c=\tfrac1d\Tr(\log S_Q^\varepsilon)$, which is the matrix $\Theta^\star(Q)$ of the statement; it is the unique traceless maximizer, since attainment forces membership in the family. For the operator-norm bound, the eigenvalues of $S_Q^\varepsilon$ lie in $[\varepsilon/d,\,1-\varepsilon(1-\tfrac1d)]\subseteq[\varepsilon/d,1]$, so the eigenvalues $\ell_1,\dots,\ell_d$ of $-\log S_Q^\varepsilon$ lie in $[0,\log(d/\varepsilon)]$; centering replaces $\ell_i$ by $\ell_i-\bar\ell$ with $\bar\ell=\tfrac1d\sum_j\ell_j\in[0,\log(d/\varepsilon)]$, so each centered eigenvalue satisfies $|\ell_i-\bar\ell|\le\max_j\ell_j-\min_j\ell_j\le\log(d/\varepsilon)$. Hence $\Theta^\star(Q)\in\Tset$, and restricting the supremum to $\Tset$ preserves both the value and the attainment, upgrading $\sup$ to $\max$. It remains to pass to the per-sample form: for traceless $\Theta$,
\[
    \Tr(S_Q^\varepsilon\Theta)
    =(1-\varepsilon)\Tr(\Sigma_Q\Theta)+\frac{\varepsilon}{d}\Tr(\Theta)
    =(1-\varepsilon)\,\E_{X\sim Q}\bigl[\phi(X)^\top\Theta\,\phi(X)\bigr],
\]
using $\Tr(\Sigma_Q\Theta)=\E_Q[\Tr(\phi\phi^\top\Theta)]=\E_Q[\phi^\top\Theta\phi]$. This is \eqref{eq:vne-dual}.

Next, part (ii). Multiplying \eqref{eq:vne-dual} by $\lambda\ge0$ preserves the maximum (for $\lambda=0$ both sides of the resulting identity vanish identically on $\Tset$, since $-\lambda\He(Q)=0$ and the $\Theta$-dependent terms carry the factor $\lambda$), and adding $\Dref(Q;\Pref)$, which does not depend on $\Theta$, gives the pointwise identity \eqref{eq:IGA-minmax}. Since the two sides agree as functions of $Q$, the problem \eqref{eq:Plambda} of minimizing the left side over $\cP$ is the two-player game of minimizing the right side, as claimed.

Finally, part (iii). $\Tset$ is convex and compact. For fixed $\Theta$, $Q\mapsto\mathcal A_\lambda(Q,\Theta)$ is convex and l.s.c.: the expectation term is affine in $Q$ (and weakly continuous when the continuity clause of Assumption~\ref{ass:finite-dim} is in force, $\phi$ being bounded and continuous), and $\Dref(\cdot;\Pref)$ is convex l.s.c.\ by hypothesis. For fixed $Q$, $\Theta\mapsto\mathcal A_\lambda(Q,\Theta)$ is concave and continuous: the expectation term is linear in $\Theta$, and $-\lambda\log\Tr(e^{-\Theta})$ is concave, since Lemma~\ref{lem:gibbs-matrix} exhibits $\Theta\mapsto-\log\Tr(e^{-\Theta})$ as an infimum over $S$ of affine functions of $\Theta$ (namely $-\log\Tr(e^{-\Theta})=\inf_S\{\Tr(S\Theta)-\mHe(S)\}$, the dual reading of the same variational identity). Sion's minimax theorem requires compactness of only one side, here the $\Theta$-side $\Tset$, so no compactness of $\cP$ is needed, and the interchange holds as stated.
\end{proof}

\begin{lemma}[Best response and the entropy energy]\label{lem:best-response}
For every distribution $Q$ and every $x\in\cX$,
\[
    \lambda(1-\varepsilon)\,\phi(x)^\top\Theta^\star(Q)\,\phi(x)
    =\lambda\,\Gen{Q}(x)-c_Q,
\]
where $c_Q\defeq\frac{\lambda(1-\varepsilon)}{d}\Tr\bigl(-\log S_Q^\varepsilon\bigr)\in\bigl[0,\ \lambda(1-\varepsilon)\log(d/\varepsilon)\bigr]$ is a constant independent of $x$.
In particular, the per-sample payoff of the best-responding spectral adversary equals the entropy energy \eqref{eq:energy-def}, up to an additive constant independent of $x$.
\end{lemma}
\begin{proof}
By definition $\Theta^\star(Q)=-\log S_Q^\varepsilon+\tfrac1d\Tr(\log S_Q^\varepsilon)I_d$, so
\begin{align*}
    \lambda(1-\varepsilon)\,\phi(x)^\top\Theta^\star(Q)\,\phi(x)
    &=\lambda(1-\varepsilon)\,\phi(x)^\top\bigl(-\log S_Q^\varepsilon\bigr)\phi(x)\\
    &\qquad+\frac{\lambda(1-\varepsilon)}{d}\Tr\bigl(\log S_Q^\varepsilon\bigr)\norm{\phi(x)}_2^2.
\end{align*}
The first term is $\lambda\,\Gen{Q}(x)$ by \eqref{eq:energy-def}, and since $\norm{\phi(x)}_2=1$ the second term is the constant $-c_Q$. The range of $c_Q$ follows because the eigenvalues of $-\log S_Q^\varepsilon$ lie in $[0,\log(d/\varepsilon)]$, so their average lies in the same interval.
\end{proof}

\subsection{Training-time instantiations: adversarial and maximum-likelihood models}
\label{app:training}

This appendix proves the results invoked in Section~\ref{subsec:training}: the joint-adversary identity for adversarially trained generators (Proposition~\ref{prop:IGA-gan}), the exactness of gradients computed through the frozen spectral adversary (Proposition~\ref{prop:entropy-grad}), and the likelihood--KL--ELBO relations underlying the maximum-likelihood instantiation (Proposition~\ref{prop:mle-vae}).

\begin{proof}[Proof of Proposition~\ref{prop:IGA-gan}]
Fix $Q\in\cQ$. By the critic representation \eqref{eq:critic-fidelity} and the entropy dual \eqref{eq:vne-dual},
\[
    \Dref(Q;\widehat P_n)-\lambda\He(Q)
    =\sup_{D\in\cD}A_Q(D)+\max_{\Theta\in\Tset}B_Q(\Theta),
\]
where $A_Q(D)=\E_{\widehat P_n}[u(D)]-\E_Q[v(D)]$ and $B_Q(\Theta)=-\lambda(1-\varepsilon)\E_Q[\phi^\top\Theta\phi]-\lambda\log\Tr(e^{-\Theta})$; here the passage from \eqref{eq:vne-dual} to $-\lambda\He(Q)=\max_\Theta B_Q(\Theta)$ is part (ii) of the proof of Proposition~\ref{prop:spectral-minmax} (multiplication by $\lambda\ge0$, with the degenerate case $\lambda=0$ giving $B_Q\equiv0=-\lambda\He(Q)$). Since $D$ and $\Theta$ range over independent sets and the two objectives share no variable, the suprema add:
\[
    \sup_{D\in\cD}A_Q(D)+\max_{\Theta\in\Tset}B_Q(\Theta)
    =\sup_{D\in\cD}\,\max_{\Theta\in\Tset}\,\bigl\{A_Q(D)+B_Q(\Theta)\bigr\},
\]
and the right-hand side is the inner expression of \eqref{eq:IGA-gan}. The identity therefore holds \emph{pointwise} in $Q$, and taking the infimum over an arbitrary class $\cQ$, convex or not, preserves it. Attainment of the $\Theta$-maximum at $\Theta^\star(Q)$ is part (i) of Proposition~\ref{prop:spectral-minmax}. No convexity of $\cQ$ was used and no minimax interchange was performed.
\end{proof}

\begin{proposition}[Frozen spectral adversary yields exact entropy gradients]\label{prop:entropy-grad}
Let $Z\sim P_Z$ on a latent space $\mathcal Z$, let $g_\vartheta:\mathcal Z\to\cX$ be measurable for each $\vartheta\in\R^p$ and differentiable in $\vartheta$ at $P_Z$-a.e.\ $z$, let $\phi$ be differentiable on an open set containing the relevant ranges with $\norm{\phi}_2\equiv1$, and let $Q_\vartheta$ denote the law of $g_\vartheta(Z)$. Suppose there are a neighborhood $U$ of $\vartheta_0$ and $L\in L^1(P_Z)$ such that
\[
    \bigl\|\nabla_\vartheta\bigl[\phi(g_\vartheta(z))\phi(g_\vartheta(z))^\top\bigr]\bigr\|\le L(z)
    \qquad\text{for all }\vartheta\in U\text{ and }P_Z\text{-a.e.\ }z.
\]
Then $\vartheta\mapsto\He(Q_\vartheta)$ is differentiable at $\vartheta_0$ and
\[
    \nabla_\vartheta\,\He(Q_\vartheta)\Big|_{\vartheta_0}
    =(1-\varepsilon)\,\E_{Z}\Bigl[\nabla_\vartheta\,
    \phi\bigl(g_\vartheta(Z)\bigr)^\top\,\Theta^\star(Q_{\vartheta_0})\,
    \phi\bigl(g_\vartheta(Z)\bigr)\Big|_{\vartheta_0}\Bigr],
\]
i.e.\ the exact gradient of the entropy coincides with the gradient of the expected per-sample payoff in which the spectral adversary is \emph{frozen} at its best response $\Theta^\star(Q_{\vartheta_0})$.
\end{proposition}
\begin{proof}
Write $S(\vartheta)\defeq S^\varepsilon_{Q_\vartheta}=(1-\varepsilon)\,\E_Z\bigl[\phi(g_\vartheta(Z))\phi(g_\vartheta(Z))^\top\bigr]+\varepsilon \tfrac{1}{d}I_d$. First, the domination hypothesis justifies differentiation under the expectation: $\vartheta\mapsto S(\vartheta)$ is differentiable at $\vartheta_0$ with $\partial_{\vartheta_k}S(\vartheta_0)=(1-\varepsilon)\,\E_Z\bigl[\partial_{\vartheta_k}\bigl(\phi(g_\vartheta(Z))\phi(g_\vartheta(Z))^\top\bigr)\big|_{\vartheta_0}\bigr]$. Next, as in the proof of Lemma~\ref{lem:firstvar}, the matrix entropy $\mHe$ is Fr\'echet differentiable at every positive-definite matrix with $D\mHe(S)[B]=-\Tr\bigl(B(\log S+I_d)\bigr)$ for symmetric $B$, and this applies at $S(\vartheta_0)\succeq(\varepsilon/d)I_d\succ0$. By the chain rule,
\[
    \partial_{\vartheta_k}\He(Q_\vartheta)\Big|_{\vartheta_0}
    =-\Tr\Bigl(\partial_{\vartheta_k}S(\vartheta_0)\,
    \bigl(\log S(\vartheta_0)+I_d\bigr)\Bigr).
\]
Then, the trace term drops out: since $\norm{\phi}_2^2\equiv1$, we have $\Tr\bigl(\partial_{\vartheta_k}S(\vartheta_0)\bigr)=(1-\varepsilon)\,\partial_{\vartheta_k}\E_Z\bigl[\norm{\phi(g_\vartheta(Z))}_2^2\bigr]=\partial_{\vartheta_k}(1-\varepsilon)=0$, so the $I_d$ contribution vanishes and
\[
    \partial_{\vartheta_k}\He(Q_\vartheta)\Big|_{\vartheta_0}
    =\Tr\Bigl(\partial_{\vartheta_k}S(\vartheta_0)\,\bigl(-\log S(\vartheta_0)\bigr)\Bigr)
    =(1-\varepsilon)\,\E_Z\Bigl[\partial_{\vartheta_k}\,
    \phi^\top\bigl(-\log S(\vartheta_0)\bigr)\phi\Big|_{\vartheta_0}\Bigr],
\]
where the matrix $-\log S(\vartheta_0)$ is held fixed under the derivative. Finally, replacing $-\log S(\vartheta_0)$ by $\Theta^\star(Q_{\vartheta_0})=-\log S(\vartheta_0)+\tfrac1d\Tr(\log S(\vartheta_0))I_d$ changes the per-sample payoff by a multiple of $\norm{\phi}_2^2\equiv1$, whose $\vartheta$-gradient is zero; hence the displayed identity. In the language of the min--max game \eqref{eq:IGA-minmax}, this is an envelope (Danskin-type) statement: the inner maximum is attained at the unique $\Theta^\star(Q_{\vartheta_0})$, and differentiating the value equals differentiating at the frozen maximizer. The direct computation above proves the identity without invoking any general envelope theorem.
\end{proof}

\paragraph{Maximum-likelihood training and VAEs.}
Deep maximum-likelihood models fit an explicit density $q_\vartheta$ by minimizing the empirical negative log-likelihood, which is the per-sample empirical proxy for the \emph{data-first} divergence: up to an additive constant independent of $\vartheta$, $\E_{\Pdata}[-\log q_\vartheta(X)]=\KL(\Pdata\|Q_\vartheta)+\mathrm{const}$ (Proposition~\ref{prop:mle-vae} below). This orientation is forced at training time: the model-first quantity $\KL(Q_\vartheta\|\widehat P_n)$ is typically infinite for a continuously supported model against an atomic empirical reference, whereas the likelihood is finite and estimable sample by sample. The IGA-regularized maximum-likelihood objective is
\begin{equation}\label{eq:IGA-mle}
    \min_\vartheta\;
    \E_{X\sim\widehat P_n}\bigl[-\log q_\vartheta(X)\bigr]
    -\lambda\,\He(Q_\vartheta),
\end{equation}
and when the likelihood is intractable, as in variational autoencoders, the negative evidence lower bound takes its place~\citep{kingma2014auto,rezende2014stochastic}:
\begin{equation}\label{eq:IGA-vae}
    \min_{\vartheta,\eta}\;
    \E_{X\sim\widehat P_n}\bigl[-\mathrm{ELBO}(X;\vartheta,\eta)\bigr]
    -\lambda\,\He(Q_\vartheta).
\end{equation}
Proposition~\ref{prop:mle-vae} records the exact relation between the two: the negative ELBO exceeds the negative log-likelihood by the encoder-posterior gap $\KL(r_\eta(\cdot\mid x)\,\|\,p_\vartheta(\cdot\mid x))\ge0$, a $\lambda$-independent quantity, so \eqref{eq:IGA-vae} is \eqref{eq:IGA-mle} plus a nonnegative gap that only the encoder parameters $\eta$ tighten; the same reading applies to diffusion models trained through variational bounds~\citep{ho2020denoising}.

We highlight two features of this combination. First, the entropy regularizer is \emph{likelihood-free}: evaluating $\He(Q_\vartheta)$ requires only samples from the decoder, never density values, so it applies to any latent-variable model whose sampler is differentiable, alongside a fidelity term that does require likelihoods. Second, the orientation caveat of Remark~\ref{rem:repair-scope} applies: the repair guarantee of Theorem~\ref{thm:repair} is proved for base-anchored objectives and does not transfer to this data-first geometry, while the monotone path of Theorem~\ref{thm:mono}, which is orientation- and convexity-agnostic, continues to describe the global minimizers of \eqref{eq:IGA-mle} and \eqref{eq:IGA-vae} as $\lambda$ grows. When the reference is instead a smooth law, such as a pretrained teacher in fine-tuning rather than $\widehat P_n$, the model-first KL anchor becomes directly usable.

\begin{proposition}[Likelihood, data-first KL, and the ELBO]\label{prop:mle-vae}
Let $\nu$ be a $\sigma$-finite measure on $\cX$ and let each model law $Q_\vartheta$ have $\nu$-density $q_\vartheta$.
\begin{enumerate}[leftmargin=*]
    \item[\textup{(i)}] \emph{Likelihood is data-first KL.} Suppose $\Pdata\ll\nu$ with density $p_0$ and $\E_{\Pdata}|\log p_0(X)|<\infty$. Then, for every $\vartheta$,
    \[
        \E_{\Pdata}\bigl[-\log q_\vartheta(X)\bigr]
        =\KL(\Pdata\|Q_\vartheta)+h_\nu(\Pdata),
    \]
    where $h_\nu(\Pdata)\defeq-\E_{\Pdata}\bigl[\log p_0(X)\bigr]$ is finite and independent of $\vartheta$, and the two sides are finite or $+\infty$ together.
    \item[\textup{(ii)}] \emph{Structure of the data-first divergence.} For fixed $P$, the map $Q\mapsto\KL(P\|Q)$ is convex on $\cP$, and if $\cX$ is Polish it is weakly lower semicontinuous. Consequently Theorem~\ref{thm:mono} applies to the objectives \eqref{eq:IGA-mle} and \eqref{eq:IGA-vae} whenever global minimizers exist, while Theorem~\ref{thm:repair}, proved for the base-anchored orientation, does not transfer (Remark~\ref{rem:repair-scope}).
    \item[\textup{(iii)}] \emph{ELBO gap.} Let $q_\vartheta(x)=\int p_\vartheta(x\mid z)\,p_Z(dz)$ be a latent-variable model and $r_\eta(\cdot\mid x)$ an encoder with $r_\eta(\cdot\mid x)\ll p_\vartheta(\cdot\mid x)$, where $p_\vartheta(\cdot\mid x)$ is the model posterior. Then, for every $x$ with $q_\vartheta(x)\in(0,\infty)$,
    \[
        -\mathrm{ELBO}(x;\vartheta,\eta)
        =-\log q_\vartheta(x)
        +\KL\bigl(r_\eta(\cdot\mid x)\,\big\|\,p_\vartheta(\cdot\mid x)\bigr)
        \ \ge\ -\log q_\vartheta(x),
    \]
    with equality if and only if the encoder matches the model posterior at $x$.
\end{enumerate}
\end{proposition}
\begin{proof}
First, part (i). Decompose $-\log q_\vartheta=\log(p_0/q_\vartheta)-\log p_0$ on the set $\{p_0>0\}$, which carries full $\Pdata$-mass. The second term integrates to $h_\nu(\Pdata)$, finite by hypothesis. For the first term, set $r\defeq q_\vartheta/p_0$ on $\{p_0>0\}$; the positive part of $\log r$ is $\Pdata$-integrable, since $\log r\le r-1$ gives $\E_{\Pdata}[(\log r)_+]\le\E_{\Pdata}[r]=\int_{\{p_0>0\}}q_\vartheta\,d\nu\le1$, so $\E_{\Pdata}[-\log r]=\E_{\Pdata}[\log(p_0/q_\vartheta)]$ is well defined in $(-\infty,+\infty]$. If $\Pdata\ll Q_\vartheta$, then $d\Pdata/dQ_\vartheta=p_0/q_\vartheta$ holds $\Pdata$-a.s.\ and $\E_{\Pdata}[\log(p_0/q_\vartheta)]=\KL(\Pdata\|Q_\vartheta)$ by definition. If $\Pdata\not\ll Q_\vartheta$, pick $A$ with $Q_\vartheta(A)=0<\Pdata(A)$; then $q_\vartheta=0$ $\nu$-a.e.\ on $A$, so $\log(p_0/q_\vartheta)=+\infty$ on a set of positive $\Pdata$-measure and, the negative part being integrable, $\E_{\Pdata}[\log(p_0/q_\vartheta)]=+\infty=\KL(\Pdata\|Q_\vartheta)$ under the extended-value convention. In both cases the displayed identity holds, with both sides finite or $+\infty$ together since $h_\nu(\Pdata)$ is finite.

Next, part (ii). For convexity, fix $Q_0,Q_1$ and $\theta\in(0,1)$, and let $\nu'$ be a $\sigma$-finite measure dominating $P$, $Q_0$, and $Q_1$ (for instance $P+Q_0+Q_1$), with densities $p,q_0,q_1$; the mixture $Q_\theta=(1-\theta)Q_0+\theta Q_1$ has density $q_\theta=(1-\theta)q_0+\theta q_1$. For fixed $x$ with $p(x)>0$, the map $q\mapsto p(x)\log(p(x)/q)$ is convex in $q>0$ (as $-\log$ is convex), and extends convexly to $q\ge0$ with value $+\infty$ at $q=0$; composing with the affine $\theta\mapsto q_\theta(x)$ and integrating $d\nu'$ preserves convexity, giving $\KL(P\|Q_\theta)\le(1-\theta)\KL(P\|Q_0)+\theta\KL(P\|Q_1)$. For lower semicontinuity, we invoke the Donsker--Varadhan variational formula, a standard fact: for probability measures on a Polish space,
\[
    \KL(P\|Q)=\sup_{f\in C_b(\cX)}
    \bigl\{\E_P[f]-\log\E_Q[e^f]\bigr\}.
\]
For each fixed $f\in C_b(\cX)$, the map $Q\mapsto\E_P[f]-\log\E_Q[e^f]$ is weakly continuous: $e^f$ is bounded continuous, so $Q\mapsto\E_Q[e^f]$ is weakly continuous with values in the compact interval $[e^{-\norm f_\infty},e^{\norm f_\infty}]\subset(0,\infty)$, on which $\log$ is continuous. A supremum of weakly continuous functions is weakly lower semicontinuous, which proves the claim. The consequences for Theorems~\ref{thm:mono} and~\ref{thm:repair} are as stated: the former uses only the existence of global minimizers and is agnostic to orientation and convexity, while the latter's three-point argument differentiates the Bregman divergence in its \emph{first} argument and is unavailable in the data-first orientation.

Finally, part (iii). Write the ELBO with encoder $r_\eta$:
\[
    \mathrm{ELBO}(x;\vartheta,\eta)
    =\E_{Z\sim r_\eta(\cdot\mid x)}
    \bigl[\log p_\vartheta(x\mid Z)+\log p_Z(Z)-\log r_\eta(Z\mid x)\bigr],
\]
with $\log p_Z$ understood as the density of the prior with respect to the latent reference measure. By Bayes' rule, $p_\vartheta(z\mid x)=p_\vartheta(x\mid z)\,p_Z(z)/q_\vartheta(x)$ for $q_\vartheta(x)\in(0,\infty)$, so $\log p_\vartheta(x\mid z)+\log p_Z(z)=\log p_\vartheta(z\mid x)+\log q_\vartheta(x)$, and substituting,
\begin{align*}
    \mathrm{ELBO}(x;\vartheta,\eta)
    &=\log q_\vartheta(x)
    -\E_{Z\sim r_\eta(\cdot\mid x)}
    \Bigl[\log\frac{r_\eta(Z\mid x)}{p_\vartheta(Z\mid x)}\Bigr]\\
    &=\log q_\vartheta(x)
    -\KL\bigl(r_\eta(\cdot\mid x)\,\big\|\,p_\vartheta(\cdot\mid x)\bigr).
\end{align*}
Negating gives the display; nonnegativity of KL gives the inequality, with equality if and only if $r_\eta(\cdot\mid x)=p_\vartheta(\cdot\mid x)$.
\end{proof}

\section{Proofs for Section~\ref{sec:wall}}
\label{app:wall}

This appendix proves the results of Section~\ref{sec:wall}: the bias and consistency of the empirical entropy wall in Theorem~\ref{thm:bias}, the sub-wall repair guarantee, and the wall-crossing statement (Proposition~\ref{prop:crossing}).

\begin{theorem}[Bias and consistency of the empirical wall]\label{thm:bias}
Let $X_1,X_2,\ldots$ be i.i.d.\ from $\Pdata$, let $\phi:\cX\to\R^d$ satisfy $\norm{\phi(x)}_2=1$, and let $\widehat P_N=\tfrac1N\sum_{i=1}^N\delta_{X_i}$. Then, for either $H\in\{\Hz,\He\}$ and every $N\ge1$:
\begin{enumerate}[leftmargin=*]
    \item[\textup{(i)}] \emph{downward bias:}\quad
        $\E\,H(\widehat P_N)\le H(\Pdata)$;
    \item[\textup{(ii)}] \emph{monotonicity in the sample size:}\quad
        $\E\,H(\widehat P_{N+1})\ge\E\,H(\widehat P_N)$;
    \item[\textup{(iii)}] \emph{consistency:}\quad
        $\E\,H(\widehat P_N)\longrightarrow H(\Pdata)$ as $N\to\infty$.
\end{enumerate}
\end{theorem}

\begin{proof}[Proof of Theorem~\ref{thm:bias}]
We prove the three parts in turn.

First, for part (i), recall from Lemma~\ref{lem:concave} that $H\in\{\Hz,\He\}$ is concave in its distribution argument through the affine map $Q\mapsto\Sigma_Q$. The empirical covariance is unbiased:
\[
    \E\,\Sigma_{\widehat P_N}
    =\frac1N\sum_{i=1}^N\E\bigl[\phi(X_i)\phi(X_i)^\top\bigr]
    =\Sigma_{\Pdata}.
\]
Jensen's inequality for the concave map $\Sigma\mapsto H$ then gives
\[
    \E\,H(\widehat P_N)\le H\bigl(\E\,\Sigma_{\widehat P_N}\bigr)=H(\Pdata).
\]

Next, for part (ii), we use a leave-one-out averaging identity. Fix $N+1$ samples and, for $j=1,\dots,N+1$, let $\widehat P_N^{(-j)}=\tfrac1N\sum_{i\ne j}\delta_{X_i}$. Each index appears in exactly $N$ of the $N+1$ leave-one-out measures, so
\[
    \widehat P_{N+1}=\frac1{N+1}\sum_{j=1}^{N+1}\widehat P_N^{(-j)},
    \qquad\text{equivalently}\qquad
    \Sigma_{\widehat P_{N+1}}=\frac1{N+1}\sum_{j=1}^{N+1}\Sigma_{\widehat P_N^{(-j)}}.
\]
Concavity of $H$ gives, pathwise,
\[
    H(\widehat P_{N+1})\ge\frac1{N+1}\sum_{j=1}^{N+1}H\bigl(\widehat P_N^{(-j)}\bigr).
\]
Each $\widehat P_N^{(-j)}$ has the same distribution as $\widehat P_N$, so taking expectations yields $\E H(\widehat P_{N+1})\ge\E H(\widehat P_N)$.

Finally, for part (iii), note that since $\norm{\phi(x)}_2=1$, the summands $\phi(X_i)\phi(X_i)^\top$ are i.i.d.\ bounded random matrices with mean $\Sigma_{\Pdata}$, so by the (matrix) strong law of large numbers $\Sigma_{\widehat P_N}\to\Sigma_{\Pdata}$ almost surely. Both $\Hz$ and $\He$ are continuous functions of the covariance matrix on the compact set of density matrices and bounded in $[0,\log d]$, so $H(\widehat P_N)\to H(\Pdata)$ a.s.; dominated convergence then gives $\E H(\widehat P_N)\to H(\Pdata)$.
\end{proof}

\begin{remark}[Interpretation and experimental consequence]\label{rem:bias-consequence}
Theorem~\ref{thm:bias} upgrades the informal assumption of downward diversity bias to a theorem for the empirical law; it does not by itself establish that a trained generator $Q_0$ satisfies $\Hz(Q_0)<\rho_\star$, which remains a separate empirical claim, consistent with reported spectral deficits in modern generators. Because $\widehat\rho_\star=\Hz(\widehat P_N)$ is downward biased, and a finite generated batch inherits the same downward bias when estimating $\Hz(Q_\lambda)$, wall-crossing plots should use matched sample sizes, repeated subsampling, or a bias-aware estimator.
\end{remark}

\begin{theorem}[Population discrepancy improves up to the wall]\label{thm:repair}
Let $\Phi$ be Fr\'echet differentiable and strictly convex on an open convex set containing $\cP$, with Bregman divergence $D_\Phi(P,Q)=\Phi(P)-\Phi(Q)-\ip{\nabla\Phi(Q)}{P-Q}$, and let $Q_\lambda\in\underset{Q\in\cP}{\arg\!\min}\:\{D_\Phi(Q,Q_0)-\lambda H(Q)\}$ with $H$ concave and $\cP$ convex. If $\Pdata\in\cP$ and $H(Q_\lambda)\le H(\Pdata)$, then
\[
    D_\Phi(\Pdata,Q_\lambda)+D_\Phi(Q_\lambda,Q_0)
    \le
    D_\Phi(\Pdata,Q_0),
\]
and hence in particular $D_\Phi(\Pdata,Q_\lambda)\le D_\Phi(\Pdata,Q_0)$.
\end{theorem}
\begin{proof}
By Proposition~\ref{prop:attained} applied to $J(Q)=D_\Phi(Q,Q_0)$ over the convex $\cP$, $Q_\lambda\in\argmin\{D_\Phi(Q,Q_0):Q\in\cP,\ H(Q)\ge\rho_\lambda\}$ with $\rho_\lambda=H(Q_\lambda)$. The feasible set $\cC=\cP\cap\{H\ge\rho_\lambda\}$ is convex ($H$ concave). The below-wall hypothesis $H(\Pdata)\ge\rho_\lambda$ and $\Pdata\in\cP$ give $\Pdata\in\cC$.

Because $Q\mapsto D_\Phi(Q,Q_0)$ is convex and differentiable in its first argument (with $\nabla_QD_\Phi(Q,Q_0)=\nabla\Phi(Q)-\nabla\Phi(Q_0)$) and $\cC$ is convex, first-order optimality of the minimizer $Q_\lambda$ over $\cC$ gives, for every feasible $P$,
\[
    \ip{\nabla\Phi(Q_\lambda)-\nabla\Phi(Q_0)}{P-Q_\lambda}\ge0.
\]
Setting $P=\Pdata\in\cC$ and using the three-point Bregman identity
\begin{align*}
    D_\Phi(\Pdata,Q_0)&=D_\Phi(\Pdata,Q_\lambda)+D_\Phi(Q_\lambda,Q_0)\\
    &\qquad+\ip{\nabla\Phi(Q_\lambda)-\nabla\Phi(Q_0)}{\Pdata-Q_\lambda},
\end{align*}
together with nonnegativity of the inner-product term, yields $D_\Phi(\Pdata,Q_0)\ge D_\Phi(\Pdata,Q_\lambda)+D_\Phi(Q_\lambda,Q_0)$. Nonnegativity of $D_\Phi(Q_\lambda,Q_0)$ gives the second inequality.
\end{proof}

\begin{remark}[What the repair theorem does not cover]\label{rem:repair-scope}
Theorem~\ref{thm:repair} concerns the base-anchored orientation $D_\Phi(Q,Q_0)$ (optimized law first) and does not transfer to a data-first objective $D_\Phi(\widehat P_n,Q)$: the gradient of a Bregman divergence in its \emph{second} argument involves the Hessian of $\Phi$ and is not $\nabla\Phi(Q)-\nabla\Phi(\widehat P_n)$, so the three-point identity no longer collapses the cross term. In particular, the maximum-likelihood instantiation of Section~\ref{subsec:training}, whose fidelity is the data-first KL by Proposition~\ref{prop:mle-vae}(i), inherits the monotone path of Theorem~\ref{thm:mono} but not the repair guarantee. Two clean options remain: (i) the base-anchored objective $D_\Phi(Q,Q_0)-\lambda H(Q)$, especially with $Q_0=P_\theta$. (ii) a symmetric Hilbertian discrepancy such as squared MMD, for which the orientation is immaterial.
\end{remark}

\begin{proposition}[Local wall crossing]\label{prop:crossing}
Let $\cP$ be convex, $\Pdata\in\cP$, and suppose there exists $R\in\cP$ with $\Hz(R)>\Hz(\Pdata)$. Put $Q_t=(1-t)\Pdata+tR$ for $t\in(0,1]$. If $\Dref(Q_t;\Pdata)\to0$ as $t\downarrow0$, then for every $\delta>0$ there exists $Q\in\cP$ with $\Dref(Q;\Pdata)\le\delta$ and $\Hz(Q)>\Hz(\Pdata)$.
\end{proposition}
\begin{proof}
By convexity of $\cP$, $Q_t\in\cP$. By concavity of $\Hz$ (Lemma~\ref{lem:concave}), $\Hz(Q_t)\ge(1-t)\Hz(\Pdata)+t\Hz(R)$; since $\Hz(R)>\Hz(\Pdata)$, the right-hand side equals $\Hz(\Pdata)+t(\Hz(R)-\Hz(\Pdata))>\Hz(\Pdata)$ for all $t\in(0,1]$, so $\Hz(Q_t)>\Hz(\Pdata)$ (plain concavity suffices; no strictness is used). By hypothesis choose $t_\delta>0$ with $\Dref(Q_{t_\delta};\Pdata)\le\delta$; then $Q=Q_{t_\delta}$ satisfies both requirements.
\end{proof}

\section{Proofs for Section~\ref{subsec:deployment} and Section~\ref{sec:guidance}}
\label{app:guidance}

This appendix proves the sampling-time results in the following order. We first record a tightness lemma for KL sublevel sets, which drives the existence argument in Theorem~\ref{thm:tilt}, and compute the first variation of the smoothed entropy (Lemma~\ref{lem:firstvar}). We then record the formal stationarity condition behind the density-ratio representations (Remark~\ref{rem:bregman-stationarity}), prove the tilt characterization (Proposition~\ref{prop:kl-tilt}), and derive Theorem~\ref{thm:tilt} by combining the existence argument with that proposition. The remainder of the appendix treats the propagation of the tilt through the noising process and the endpoint guarantees.

\begin{lemma}[Relative-entropy sublevel tightness]\label{lem:kl-tight}
Let $P$ be a probability measure on a Polish space and $C\ge0$. The sublevel set $\{Q:\KL(Q\|P)\le C\}$ is tight. Concretely, for any measurable $A$ with $0<P(A)<1$ and any $Q$ with $\KL(Q\|P)\le C$,
\begin{equation}\label{eq:kl-tight}
    Q(A)\;\le\;\frac{C+1}{\log(1/P(A))}.
\end{equation}
\end{lemma}
\begin{proof}
By the data-processing inequality applied to the binary partition $\{A,A^c\}$,
\[
    \KL(Q\|P)\ge d_2\bigl(Q(A)\,\|\,P(A)\bigr),
    \qquad
    d_2(q\|p)=q\log\tfrac qp+(1-q)\log\tfrac{1-q}{1-p},
\]
where $d_2$ is the binary KL, with the usual conventions and $d_2(q\|p)=+\infty$ if $p\in\{0,1\}$ while $q\notin\{0,1\}$. Write $p=P(A)\in(0,1)$ and $q=Q(A)$. Since $1-p\le1$ gives $\log\tfrac{1-q}{1-p}\ge\log(1-q)$, and $t\log t\ge-e^{-1}$ on $[0,1]$, the terms $q\log q$ and $(1-q)\log\tfrac{1-q}{1-p}$ together contribute at least $-2/e\ge-1$; hence
\[
    d_2(q\|p)=q\log(1/p)+q\log q+(1-q)\log\tfrac{1-q}{1-p}
    \ge q\log(1/p)-1;
\]
hence $Q(A)\log(1/P(A))\le\KL(Q\|P)+1\le C+1$, which is \eqref{eq:kl-tight}.

For tightness, fix $\eta>0$. Since every probability measure on a Polish space is tight, there is a compact $K$ with $P(K^c)$ as small as desired; taking $P(K^c)$ small enough that $(C+1)/\log(1/P(K^c))\le\eta$ yields $\sup_{\KL(Q\|P)\le C}Q(K^c)\le\eta$.
\end{proof}

\begin{proof}[Proof of Lemma~\ref{lem:firstvar}]
Write $A_\nu\defeq\int\phi(x)\phi(x)^\top d\nu(x)$, a symmetric matrix (finite since $\phi$ is bounded and $\nu$ is finite). Since $Q\mapsto\Sigma_Q$ is affine, $\Sigma_{Q+t\nu}=\Sigma_Q+tA_\nu$ and hence $S_{Q+t\nu}^\varepsilon=S_Q^\varepsilon+t(1-\varepsilon)A_\nu$ for all $t$ for which $Q+t\nu$ is a probability measure. The matrix entropy $\mHe(S)=-\Tr(S\log S)$ is Fr\'echet differentiable at every positive definite $S$, with derivative $D\mHe(S)[B]=-\Tr\bigl(B(\log S+I_d)\bigr)$ for symmetric $B$; this applies at $S=S_Q^\varepsilon$ because $S_Q^\varepsilon\succeq(\varepsilon/d)I_d\succ0$, and the perturbed matrices $S_Q^\varepsilon+t(1-\varepsilon)A_\nu$ remain in a compact neighborhood of positive definite matrices for small $t$. By the chain rule along the affine path,
\begin{align*}
    \frac{d}{dt}\,\He(Q+t\nu)\Big|_{t=0^+}
    &=-(1-\varepsilon)\,\Tr\bigl(A_\nu(\log S_Q^\varepsilon+I_d)\bigr)\\
    &=(1-\varepsilon)\!\int\!\phi(x)^\top\bigl(-\log S_Q^\varepsilon\bigr)\phi(x)\,d\nu(x)\\
    &\qquad
    -(1-\varepsilon)\!\int\!\norm{\phi(x)}_2^2\,d\nu(x).
\end{align*}
Since $\norm{\phi(x)}_2^2=1$, the last integral equals $\nu(\cX)=0$, and the first term is $\int\Gen{Q}\,d\nu$ by definition \eqref{eq:energy-def}.
\end{proof}

We note that the normalization $\norm{\phi}_2=1$ makes the trace term of the derivative drop out, so $\Gen{Q}$ represents the first variation of the entropy functional: the first variation of $\He$ at $Q$, along any admissible mass-preserving perturbation, integrates $\Gen{Q}$ against the perturbation.

\begin{remark}[Bregman stationarity]\label{rem:bregman-stationarity}
For a general Bregman anchor, a formal first-order condition explains how the geometry of $\Phi$ converts the entropy first variation into a displacement of the law: an interior optimizer $Q^\star$ of \eqref{eq:IGA-sampling-bregman} satisfies
\[
    \nabla\Phi(Q^\star)-\nabla\Phi(P_\theta)=\lambda\,g_{Q^\star}+c,
\]
where $g_Q=\Gen{Q}$ is the first variation of $\He$ (Lemma~\ref{lem:firstvar}) and $c$ is the scalar multiplier of the unit-mass constraint. We do not rely on this identity: in the KL geometry, the derivation below obtains the density ratio directly from the first-variation computation, with no interiority hypothesis.
\end{remark}

We now prove the tilt characterization of Section~\ref{subsec:deployment}.

\begin{proof}[Proof of Proposition~\ref{prop:kl-tilt}]
Throughout, write $\cP_\theta=\{Q\in\cP:Q\ll P_\theta\}$, a convex set, let $F$ be extended by $+\infty$ off $\cP_\theta$, and recall from the statement that $F$ attains a finite minimum on $\cP_\theta$; let $Q^\star$ be any minimizer and $q^\star=dQ^\star/dP_\theta$. We first prove uniqueness, then mutual absolute continuity, then derive the tilt from the first-variation computation, and finally bound the density ratio.

\emph{Uniqueness.} On its finite domain the KL term is strictly convex in $Q$, and $-\lambda\He$ is convex by Lemma~\ref{lem:concave}. Hence $F$ is strictly convex where finite, and its minimizer is unique.

\emph{Mutual absolute continuity.} We have $Q^\star\ll P_\theta$ with $\KL(Q^\star\|P_\theta)<\infty$ by finiteness of the minimum; it remains to prove $q^\star>0$ $P_\theta$-a.s. Suppose instead that $q^\star=0$ on a measurable set $A$ with $P_\theta(A)>0$. Let $R=P_\theta(\cdot\mid A)$ and $Q_t=(1-t)Q^\star+tR\in\cP_\theta$ for $t\in(0,1)$. Exactly as in the corresponding computation for the KL term (splitting the integral over $A$, where the density of $Q_t$ is $t\,\1_A/P_\theta(A)$, and $A^c$, where it is scaled by $1-t$),
\[
    \KL(Q_t\|P_\theta)-\KL(Q^\star\|P_\theta)=t\log t+O(t).
\]
The remaining entropy term of $F$ changes by only $O(t)$. Indeed, $S_{Q_t}^\varepsilon-S_{Q^\star}^\varepsilon=t(1-\varepsilon)(\Sigma_R-\Sigma_{Q^\star})$ has norm $O(t)$ (with constants depending only on $\norm{\phi}_2=1$), and $S\mapsto-\Tr(S\log S)$ is Lipschitz on the compact spectral range $[\varepsilon/d,1]$, its derivative $-(\log S+I)$ being bounded in operator norm by $\log(d/\varepsilon)+1$ there. Therefore
\[
    F(Q_t)-F(Q^\star)=t\log t+O(t)<0
    \qquad\text{for small }t>0,
\]
since $t\log t\to0^-$ dominates $O(t)$; this contradicts optimality. Thus $q^\star>0$ $P_\theta$-a.s.\ and $Q^\star\sim P_\theta$.

\emph{First variation and the tilt.} For bounded measurable $h$ with $\E_{Q^\star}h=0$, set $dQ_t=(1+th)\,dQ^\star$, a valid probability law for $|t|\le1/(1+\norm h_\infty)$, and let $\nu=h\,dQ^\star$, a finite signed measure with $\nu(\cX)=0$. By Lemma~\ref{lem:firstvar} and the identity $g_Q=\Gen{Q}$, the entropy term has derivative $\tfrac{d}{dt}\He(Q_t)|_{t=0}=\int g_{Q^\star}\,h\,dQ^\star$. The KL term has derivative
\[
    \frac{d}{dt}\KL(Q_t\|P_\theta)\Big|_{t=0}
    =\int(\log q^\star+1)\,h\,dQ^\star
    =\int\log q^\star\,h\,dQ^\star,
\]
where the $+1$ term vanishes since $\E_{Q^\star}h=0$; differentiation under the integral is justified by dominated convergence, as $h$ is bounded and $\int q^\star|\log q^\star|\,dP_\theta<\infty$ from $\KL(Q^\star\|P_\theta)<\infty$ together with the uniform bound $t\log t\ge-e^{-1}$. First-order optimality $\tfrac{d}{dt}F(Q_t)|_{t=0}=0$ for all such $h$ therefore gives
\[
    \int\Bigl(\log q^\star-\lambda g_{Q^\star}\Bigr)\,h\,dQ^\star=0
\]
for all bounded $h$ with $\E_{Q^\star}h=0$, so the integrand in parentheses is $Q^\star$-a.s.\ (and, by mutual absolute continuity, $P_\theta$-a.s.) equal to a constant. Exponentiating and normalizing, with the constant absorbed into the normalizer, yields \eqref{eq:kl-tilt}.

\emph{Boundedness of the ratio.} The eigenvalues of $S_Q^\varepsilon$ lie in $[\varepsilon/d,1]$, so $\opnorm{-\log S_Q^\varepsilon}\le\log(d/\varepsilon)$ and $0\le\lambda g_Q\le\lambda(1-\varepsilon)\log(d/\varepsilon)$ uniformly over $Q$ and $x$. The exponent in \eqref{eq:kl-tilt} is therefore uniformly bounded, so the normalizer lies in $(0,\infty)$ and the density ratio is bounded above and below by positive constants.
\end{proof}

\begin{proof}[Proof of Theorem~\ref{thm:tilt}]
We first establish existence, then obtain the remaining claims from Proposition~\ref{prop:kl-tilt}.

\emph{Existence.} Work on $\cP_\theta=\{Q:Q\ll P_\theta\}$, a convex set, and extend $F$ by $+\infty$ off it. The infimum is finite: $F(P_\theta)=-\lambda\He(P_\theta)\in[-\lambda\log d,0]$, while $F\ge-\lambda\log d>-\infty$ termwise. Let $(Q_n)$ be a minimizing sequence. Since $0\le\He(Q)\le\log d$, boundedness of $F(Q_n)$ implies $\sup_n\KL(Q_n\|P_\theta)=:C<\infty$. By Lemma~\ref{lem:kl-tight}, the KL sublevel set $\{Q:\KL(Q\|P_\theta)\le C\}$ is tight, so $(Q_n)$ is tight. By Prokhorov's theorem a subsequence converges weakly to some $Q^\star$. Since $\phi$ is bounded and continuous, $Q\mapsto\Sigma_Q,S_Q^\varepsilon$ are weakly continuous; because $S_Q^\varepsilon\succeq(\varepsilon/d)I_d\succ0$ uniformly, $S\mapsto-\Tr(S\log S)$ is continuous on the relevant compact spectral range, so $Q\mapsto\He(Q)$ is weakly continuous. The map $Q\mapsto\KL(Q\|P_\theta)$ is weakly l.s.c.: by the Donsker--Varadhan formula, $\KL(Q\|P_\theta)=\sup_{f\in C_b(\cX)}\{\E_Q[f]-\log\E_{P_\theta}[e^f]\}$ is a supremum of weakly continuous functions of $Q$, exactly as in the proof of Proposition~\ref{prop:mle-vae}(ii) with the roles of the two arguments exchanged. Therefore $F(Q^\star)\le\liminf_nF(Q_n)$, so $Q^\star$ attains the infimum; in particular $\KL(Q^\star\|P_\theta)<\infty$, so $Q^\star\ll P_\theta$.

\emph{Specialization.} By the existence step, $F$ attains a finite minimum on $\cP_\theta$, so Proposition~\ref{prop:kl-tilt} applies: the minimizer $Q^\star$ is unique and mutually absolutely continuous with $P_\theta$, which proves part (i), and the tilt \eqref{eq:kl-tilt} holds with the total reward $R_{Q^\star}$ of \eqref{eq:energy-reward}, proving part (ii).

\emph{Boundedness.} The spectral floor gives $0\le\lambda\,\Gen{Q^\star}\le\lambda(1-\varepsilon)\log(d/\varepsilon)$. Thus $R_{Q^\star}$ is uniformly bounded, so $Z=\E_{P_\theta}\exp(R_{Q^\star})\in(0,\infty)$ and $dQ^\star/dP_\theta$ is bounded above and below by positive constants; hence $Q^\star$ has the same $P_\theta$-essential support as $P_\theta$, proving part (iii).
\end{proof}

\begin{remark}[Scope of the tilt characterization]\label{rem:tilt-scope}
The characterization \eqref{eq:kl-tilt} is a fixed point: $R_{Q^\star}$ depends on $Q^\star$ through $S_{Q^\star}^\varepsilon$ (contrast Lemma~\ref{lem:gibbs}, where the reward is fixed and the tilt is explicit). The KL anchor only reweights within $\supp(P_\theta)$ and creates no mass where $P_\theta=0$. Strict convexity proves uniqueness of the target law but does not imply that any particular fixed-point iteration is contractive; convergence of a numerical solver requires a separate argument.
\end{remark}

\begin{proof}[Proof of Theorem~\ref{thm:twist}]
For measurable $A$,
\[
    q_t^\star(A)=\int K_t(A\mid x_0)\,Q^\star(dx_0)
    =Z^{-1}\!\int K_t(A\mid x_0)\,w(x_0)\,P_\theta(dx_0).
\]
Disintegrate the base joint law of $(X_0,X_t)$ as $K_t(dx_t\mid x_0)P_\theta(dx_0)=P_\theta(dx_0\mid x_t)\,p_t(dx_t)$. By Fubini's theorem (applicable since $w$ is bounded, by Theorem~\ref{thm:tilt}),
\[
    q_t^\star(A)
    =Z^{-1}\!\int_A\!\Bigl(\int w(x_0)P_\theta(dx_0\mid x_t)\Bigr)p_t(dx_t)
    =Z^{-1}\!\int_A h_t(x_t)\,p_t(dx_t),
\]
with $h_t(x_t)=\E_{P_\theta}[w(X_0)\mid X_t=x_t]$. Since $A$ was arbitrary, $dq_t^\star/dp_t=h_t/Z$. If both marginals have positive differentiable densities, then $\log q_t^\star=\log p_t+\log h_t-\log Z$; as $Z$ is constant in $x_t$, $\nabla\log q_t^\star=\nabla\log p_t+\nabla\log h_t$.
\end{proof}

\begin{theorem}[Exact reverse process]\label{thm:reverse}
Suppose the forward SDE $dX_t=f(X_t,t)\,dt+g(t)\,dW_t$ (for $t$ increasing from $0$ to $T$) admits strictly positive differentiable marginal densities and satisfies the standard regularity conditions for time reversal and the probability-flow construction. We use the standard reverse-time convention in which the displayed equations are integrated with \emph{decreasing} $t$ from $T$ to $0$. Then the reverse-time SDE, initialized at $q_T^\star$,
\[
    d X_t=\bigl[f(X_t,t)-g(t)^2\bigl(\nabla\log p_t(X_t)+u_t(X_t)\bigr)\bigr]\,dt
    +g(t)\,d\overline W_t,
\]
has time-zero law exactly $Q^\star$, where $\overline W$ is a reverse-time Brownian motion. The probability-flow ODE
\[
    \dot X_t=f(X_t,t)-\tfrac12\,g(t)^2\bigl(\nabla\log p_t(X_t)+u_t(X_t)\bigr),
\]
likewise integrated with decreasing $t$, has the same one-time marginals. Equivalently, under the forward reparameterization $\tau=T-t$ and $Y_\tau=X_{T-\tau}$, both dynamics run with increasing $\tau$ and their drifts are the \emph{negatives} of the displayed drifts evaluated at $t=T-\tau$ (the diffusion term is unchanged).
\end{theorem}
\begin{proof}
Anderson's time-reversal theorem, under the stated regularity, gives the reverse-time SDE (integrated with decreasing $t$) for the process with marginals $q_t^\star$ as having drift $f-g^2\nabla\log q_t^\star$ and initial law $q_T^\star$. Substituting $\nabla\log q_t^\star=\nabla\log p_t+u_t$ from Theorem~\ref{thm:twist} yields the stated drift, and the marginals are $q_t^\star$ for all $t$, in particular $Q^\star$ at $t=0$. The probability-flow ODE $\dot X_t=f-\tfrac12g^2\nabla\log q_t^\star$ is the deterministic process with identical one-time marginals under the $q_T^\star$ initialization. The $\tau=T-t$ statement follows from the chain rule $\tfrac{d}{d\tau}Y_\tau=-\tfrac{d}{dt}X_t|_{t=T-\tau}$, which flips the sign of every drift while preserving the (sign-indifferent) diffusion coefficient. If instead one initializes at $p_T\ne q_T^\star$, the time-zero law is not $Q^\star$; the discrepancy is quantified in Theorem~\ref{thm:endpoint}.
\end{proof}

\begin{remark}[Initialization]\label{rem:init}
Practical samplers initialize from $p_T$, not $q_T^\star$. These coincide only when $h_T$ is constant. Because $dq_T^\star/dp_T=h_T/Z$, the mismatch is $\KL(p_T\|q_T^\star)=\E_{p_T}\log\frac{dp_T}{dq_T^\star}=\log Z-\E_{p_T}\log h_T(X_T)$, which enters any rigorous comparison between a deployed sampler and $Q^\star$ (Theorem~\ref{thm:endpoint}). Theorem~\ref{thm:reverse} is exact only with the $q_T^\star$ initialization.
\end{remark}

\begin{definition}[Plug-in fields]\label{def:plugin}
Assume $\cX\subseteq\R^D$ and that $\phi$ and the denoiser are differentiable. With guidance scale $\omega_t\ge0$,
\[
    \widetilde u_t^{\mathrm{chain}}(x_t)
    =\frac{\omega_t}{a}J_{\widehat x_0}(x_t,t)^\top
     \nabla_xR_{Q^\star}(\widehat x_0(x_t,t)),
\]
where $J_{\widehat x_0}$ is the Jacobian of the denoiser. The \emph{direct-injection} variant, applicable when clean and noisy states share dimension, is
\[
\widetilde u_t^{\mathrm{dir}}(x_t)=\omega_t \nabla_xR_{Q^\star}(\widehat x_0(x_t,t)),
\]
which uses a chosen state-space direction rather than the derivative of the composite map $x_t\mapsto R_{Q^\star}(\widehat x_0(x_t,t))$.
\end{definition}

\begin{remark}[Plug-in is uncontrolled]\label{rem:plugin}
No general equality or one-sided bound relates $u_t$ and $\widetilde u_t$: $R_{Q^\star}$ is nonlinear, and neither conditional expectation nor differentiation commutes with a point-mass substitution. For the reward $R_{Q^\star}=\lambda\Gen{Q^\star}$ of \eqref{eq:energy-reward},
\[
    \nabla_xR_{Q^\star}(x)
    =J_\phi(x)^\top\bigl[2\lambda(1-\varepsilon)(-\log S_{Q^\star}^\varepsilon)\phi(x)\bigr],
\]
where $J_\phi$ is the Jacobian of $\phi$ (the factor $2$ comes from differentiating the quadratic form $\phi^\top M\phi$ with symmetric $M=-\log S_{Q^\star}^\varepsilon$, which is the $x$-gradient of $\lambda\Gen{Q^\star}$).
\end{remark}

\begin{remark}[Discrete updates are approximations]\label{rem:ddpm-ddim}
The DDPM/DDIM updates \eqref{eq:eps-IGA}--\eqref{eq:ddpm-ddim} are algebraically consistent with the corrected noise prediction under the noise--score convention $\nabla\log p_t=-\varepsilon_\theta/\sqrt{1-\bar\alpha_t}$, and are discrete implementations inspired by Theorem~\ref{thm:twist}. They do not exactly sample $Q^\star$ even if $\widetilde u_t=u_t$; DDIM adds a further ODE discretization and path-selection approximation. A discretization term must therefore be added to the continuous-time bound of Theorem~\ref{thm:endpoint}.
\end{remark}

\begin{theorem}[Endpoint KL and TV bounds]\label{thm:endpoint}
Let the exact reverse process have initial law $q_T^\star$, base score $s_t=\nabla\log p_t$, and exact guidance $u_t=\nabla\log h_t$; let the deployed process have initial law $\pi_T$, learned score $\widehat s_t$, and approximate guidance $\widetilde u_t$. Put
\[
    e_t=\widehat s_t-s_t,
    \qquad
    \delta_t=\widetilde u_t-u_t.
\]
Assume that:
\begin{enumerate}[leftmargin=*]
    \item[\textup{(i)}] both continuous-time processes share the diffusion coefficient $g(t)I$ with $g(t)>0$ on $(0,T)$ (nondegeneracy on the open interval); any endpoint degeneracy $g(0)=0$ or $g(T)=0$ is handled by truncating to $[\eta,T-\eta]$, applying the bound there, and letting $\eta\downarrow0$, assuming the resulting integral converges;
    \item[\textup{(ii)}] the absolute-continuity and Novikov conditions for Girsanov's theorem hold on each such subinterval.
\end{enumerate}
Then, for the orientation $\KL(\widehat Q\|Q^\star)$ between the clean endpoint laws $Q^\star,\widehat Q$ (deployed law first),
\[
    \KL(\widehat Q\|Q^\star)
    \le\KL(\pi_T\|q_T^\star)
    +\frac12\int_0^T g(t)^2\,
     \E_{\widehat\Pp}\bigl[\norm{e_t(X_t)+\delta_t(X_t)}_2^2\bigr]\,dt,
\]
and, by Pinsker's inequality (again for the orientation $\KL(\widehat Q\|Q^\star)$),
\[
    \TV(\widehat Q,Q^\star)
    \le\Bigl[\tfrac12\KL(\pi_T\|q_T^\star)
    +\tfrac14\int_0^T g(t)^2\,
     \E_{\widehat\Pp}\bigl[\norm{e_t(X_t)+\delta_t(X_t)}_2^2\bigr]\,dt\Bigr]^{1/2}.
\]
\end{theorem}
\begin{proof}
Let $\Pp^\star$ be the path law of the exact reverse process (initial law $q_T^\star$, drift $b_t^\star=f-g^2(s_t+u_t)$) and $\widehat\Pp$ the path law of the deployed process (initial law $\pi_T$, drift $\widehat b_t=f-g^2(\widehat s_t+\widetilde u_t)$), both with diffusion coefficient $g(t)I$. The drift difference is $\widehat b_t-b_t^\star=-g(t)^2(e_t+\delta_t)$. First, we decompose the path-space relative entropy in the direction $\KL(\widehat\Pp\|\Pp^\star)$ (deployed first) by the chain rule over the initial time $T$,
\[
    \KL(\widehat\Pp\|\Pp^\star)
    =\KL(\pi_T\|q_T^\star)
    +\E_{\pi_T}\KL\bigl(\widehat\Pp(\cdot\mid X_T)\,\|\,\Pp^\star(\cdot\mid X_T)\bigr).
\]
Next, we evaluate the conditional term. Conditionally on $X_T$, the two processes share the diffusion coefficient $g(t)I$, which is nondegenerate on $(0,T)$ (or on each $[\eta,T-\eta]$, with $\eta\downarrow0$ afterwards), and differ only in drift, so on that interval the change of measure is absolutely continuous and Girsanov's theorem applies, giving
\begin{align*}
    \KL\bigl(\widehat\Pp(\cdot\mid X_T)\,\|\,\Pp^\star(\cdot\mid X_T)\bigr)
    &=\tfrac12\,\E_{\widehat\Pp}\Bigl[\int_0^T
    \bigl\|g(t)^{-1}(\widehat b_t-b_t^\star)\bigr\|^2\,dt\ \Big|\ X_T\Bigr]\\
    &=\tfrac12\,\E_{\widehat\Pp}\Bigl[\int_0^T g(t)^2\norm{e_t+\delta_t}^2\,dt\ \Big|\ X_T\Bigr],
\end{align*}
using $g(t)^{-1}(\widehat b_t-b_t^\star)=-g(t)(e_t+\delta_t)$. Averaging over $X_T\sim\pi_T$, we arrive at
\[
    \KL(\widehat\Pp\|\Pp^\star)=\KL(\pi_T\|q_T^\star)
    +\frac12\int_0^T g(t)^2\,\E_{\widehat\Pp}\norm{e_t+\delta_t}^2\,dt.
\]
Finally, the clean endpoint laws $\widehat Q,Q^\star$ are measurable images (the time-$0$ coordinate) of the path laws, so the data-processing inequality gives $\KL(\widehat Q\|Q^\star)\le\KL(\widehat\Pp\|\Pp^\star)$, which is the stated KL bound. Pinsker's inequality $\TV(\mu,\nu)\le\sqrt{\tfrac12\KL(\mu\|\nu)}$ applied to $\widehat Q,Q^\star$ gives the TV bound.
\end{proof}

The exact-guidance bound is recovered only when $\pi_T=q_T^\star$ and $e_t\equiv0$; a separate discretization term is still needed for the implemented DDPM/DDIM sampler (Remark~\ref{rem:ddpm-ddim}).

\begin{remark}[i.i.d.\ sampling holds only for fixed guidance]\label{rem:iid}
At the population level $Q^\star$ is a single law, so independent exact samplers with a fixed potential $R_{Q^\star}$ produce i.i.d.\ draws from $Q^\star$; this distinguishes IGA from methods that define diversity only through a coupled batch objective. However, $R_{Q^\star}$ depends on the unknown $Q^\star$ through $S_{Q^\star}^\varepsilon$. If a practical algorithm recomputes covariance or entropy gradients from the same batch being generated, each particle's drift depends on the others: the outputs are exchangeable but not independent. An i.i.d.\ guarantee requires one of the following: \emph{frozen-potential sampling}, in which the potential is estimated in a separate stage, frozen, and used to run independent trajectories; \emph{independent-pilot estimation}, in which the potential is estimated on an independent pilot sample; or a \emph{mean-field analysis}, invoking a propagation-of-chaos argument when the potential is updated from the active batch. Absent these, finite-batch IGA guidance should be described as an interacting particle system.
\end{remark}

\section{Training-Time IGA for Diffusion Models: Proofs and Discussion}
\label{app:diffusion-training}

Diffusion models are trained through variational bounds, which places them in the maximum-likelihood family of Appendix~\ref{app:training}. Attaching $-\lambda\He(Q_\vartheta)$ directly to the denoising objective requires samples from $Q_\vartheta$, and hence full reverse rollouts inside the training loop. The framework offers a rollout-free alternative: perform the IGA correction on the \emph{data} before fitting the denoiser. To this end, we apply Proposition~\ref{prop:kl-tilt} with the reference distribution $\widehat P_n$; note that the proposition depends on $P_\theta$ only through its role as the reference measure, and the finite-minimum hypothesis holds automatically since the feasible set is the simplex over the training atoms. This application yields unique weights
\begin{equation}\label{eq:IGA-weights}
    q_i^\star=\frac{w_i}{\sum_{j=1}^n w_j},
    \qquad
    w_i=\exp\Bigl(\lambda\,\Gen{Q^\star_\lambda}(x_i)\Bigr),
\end{equation}
which form a self-consistent softmax over the training set and define $Q^\star_\lambda=\sum_{i=1}^n q_i^\star\delta_{x_i}$. The weights can be computed as a finite-dimensional convex--concave saddle problem through the spectral dual of Proposition~\ref{prop:spectral-minmax}, with the spectral adversary and the reweighting playing the detection and response roles described after that proposition. The outcome can be viewed as a distributionally robust reweighting of the dataset, although not a worst-case-loss one (Remark~\ref{rem:dro}, Appendix~\ref{app:diffusion-training}). The following proposition shows that training on the reweighted data is justified exactly rather than heuristically:

\begin{proposition}[IGA training as divergence minimization toward reweighted data]\label{prop:IGA-pythagoras}
Let $\Pref$ be a probability measure, let $\lambda\ge0$, $\varepsilon\in(0,1)$, and let $F$ be the objective~\eqref{eq:IGA-sampling-kl} with $\Pref$ in place of $P_\theta$. Suppose $F$ attains a finite minimum over $\{Q\in\cP:Q\ll\Pref\}$, at $Q^\star_\lambda$. Then for every $Q$ with $F(Q)<\infty$,
\begin{equation}\label{eq:IGA-pythagoras}
    F(Q)-F(Q^\star_\lambda)
    =\KL(Q\,\|\,Q^\star_\lambda)
    +\lambda\,B_{-\He}(Q,Q^\star_\lambda),
\end{equation}
where $B_{-\He}(Q,Q')\defeq\He(Q')-\He(Q)+\int\Gen{Q'}\,d(Q-Q')\ge0$ is the Bregman divergence of the convex functional $-\He$.
\end{proposition}

Every term on the right-hand side of \eqref{eq:IGA-pythagoras} is a divergence between $Q$ and the IGA-reweighted reference, and both terms vanish exactly at $Q=Q^\star_\lambda$. Therefore, over any generator class, minimizing the IGA objective is equivalent to matching the reweighted law. For a diffusion model, this equivalence justifies \emph{weighted denoising score matching}, i.e., the standard training loss with clean samples drawn according to the weights $q^\star$ in place of uniform weights, which coincides with diffusion training under the data law $Q^\star_\lambda$. When the variational bound is tight and the generator class is expressive, the minimizers of the weighted bound attain the IGA optimum. In general, the weighted bound controls the data-first divergence $\KL(Q^\star_\lambda\|Q_\vartheta)$, whereas the IGA excess \eqref{eq:IGA-pythagoras} is the model-first sum; this is the standard mass-covering versus mode-seeking asymmetry, stated here in an exact form (Corollary~\ref{cor:diffusion-IGA}, Appendix~\ref{app:diffusion-training}, which also records the empirical-versus-population role of the reference).

The following proves the results of the diffusion-training paragraph of Section~\ref{subsec:training}: the finite-sample IGA reweighting of the data (Corollary~\ref{cor:data-tilt}), the exact decomposition of Proposition~\ref{prop:IGA-pythagoras}, its consequence for weighted denoising training (Corollary~\ref{cor:diffusion-IGA}), and the relation to distributionally robust optimization (Remark~\ref{rem:dro}).

\begin{corollary}[Finite-sample IGA reweighting]\label{cor:data-tilt}
Let $x_1,\dots,x_n$ be the training samples and $\widehat P_n=\tfrac1n\sum_{i=1}^n\delta_{x_i}$. For every $\lambda\ge0$ and $\varepsilon\in(0,1)$, the objective of \eqref{eq:IGA-sampling-kl} with reference $\widehat P_n$ attains a finite minimum over $\{Q\in\cP:Q\ll\widehat P_n\}$, and its unique minimizer $Q^\star_\lambda=\sum_{i=1}^nq_i^\star\delta_{x_i}$ has strictly positive weights given by the self-consistent softmax \eqref{eq:IGA-weights}.
\end{corollary}

\begin{proof}
The map $q\mapsto Q_q=\sum_iq_i\delta_{x_i}$ identifies $\{Q\in\cP:Q\ll\widehat P_n\}$ with the simplex $\Delta_n=\{q\in\R^n:q\ge0,\ \sum_iq_i=1\}$; if some training points coincide, the identification merges the corresponding atoms and the argument below is unchanged. On $\Delta_n$ each term of the objective is finite and continuous: $\KL(Q_q\|\widehat P_n)=\sum_iq_i\log(nq_i)$, continuous with the convention $0\log0=0$; and $\He(Q_q)$, the composition of the matrix entropy $\mHe$, continuous on density matrices, with the affine map $q\mapsto(1-\varepsilon)\sum_iq_i\phi(x_i)\phi(x_i)^\top+\tfrac\varepsilon dI_d$. A continuous function on the compact set $\Delta_n$ attains its minimum, which is finite, so the hypothesis of Proposition~\ref{prop:kl-tilt} holds with $\widehat P_n$ in the role of the reference. That proposition gives uniqueness, mutual absolute continuity (equivalently, $q_i^\star>0$ for every $i$), and the tilt \eqref{eq:kl-tilt}, which on atoms reads $nq_i^\star=w_i/Z$ with $Z=\tfrac1n\sum_{j}w_j$ and with $w_i$ as in \eqref{eq:IGA-weights}; cancelling the factors of $n$ gives \eqref{eq:IGA-weights}.
\end{proof}

\paragraph{Two roles of the reference.}
The corollary solves the IGA problem anchored at the \emph{empirical} training distribution exactly; the resulting $Q^\star_\lambda$ is supported on the training set and is the implicit data law of the weighted training scheme below. When the generated law $Q_\vartheta$ of a continuously supported model is compared against a reference, $\KL(Q_\vartheta\|\widehat P_n)=+\infty$, so the decomposition of Proposition~\ref{prop:IGA-pythagoras} is applied with a population reference, such as $\Pdata$ or a smooth teacher law in fine-tuning, for which $F(Q_\vartheta)$ is finite for absolutely continuous models. The empirical weights \eqref{eq:IGA-weights} are then the plug-in counterpart of the population tilt: the exponent is a fixed continuous function of $x$ once the covariance $S^\varepsilon_{Q^\star_\lambda}$ is given, and the empirical fixed point estimates exactly this covariance. We do not pursue a finite-sample analysis of this plug-in step here; Theorem~\ref{thm:bias} describes the behavior of the underlying moment estimates.

\begin{proof}[Proof of Proposition~\ref{prop:IGA-pythagoras}]
Write $Q^\star=Q^\star_\lambda$ and let
\[
    T(x)\defeq\lambda\,\Gen{Q^\star}(x)
\]
be the total reward \eqref{eq:energy-reward}, with $\Pref$ in the role of the reference, so that $dQ^\star/d\Pref=e^T/Z$ with $Z=\E_{\Pref}[e^{T(X)}]$ by \eqref{eq:kl-tilt}. Since the eigenvalues of $S^\varepsilon_{Q}$ lie in $[\varepsilon/d,1]$, the exponent is uniformly bounded: $0\le T\le\lambda(1-\varepsilon)\log(d/\varepsilon)$. Fix $Q$ with $F(Q)<\infty$. Because $0\le\He\le\log d$, finiteness of $F(Q)$ is equivalent to $\KL(Q\|\Pref)<\infty$, and in particular $Q\ll\Pref$.

\emph{Step 1: chain rule for the KL term.} The ratio $dQ^\star/d\Pref=e^T/Z$ is bounded above and below by positive constants, so $Q\ll\Pref$ if and only if $Q\ll Q^\star$, and $\Pref$-almost surely
\[
    \log\frac{dQ}{d\Pref}
    =\log\frac{dQ}{dQ^\star}+T-\log Z.
\]
The negative part of $\log(dQ/d\Pref)$ is $Q$-integrable (as always for a log-density ratio: $t(\log t)_-\le e^{-1}$ for $t\ge0$), and $T-\log Z$ is bounded, hence $Q$-integrable; therefore all three expectations below are well defined in $(-\infty,+\infty]$ and additivity holds:
\begin{equation}\label{eq:kl-chain}
    \KL(Q\|\Pref)
    =\KL(Q\|Q^\star)+\E_Q[T]-\log Z,
\end{equation}
with $\KL(Q\|Q^\star)<\infty$ exactly when $\KL(Q\|\Pref)<\infty$.

\emph{Step 2: the entropy Bregman term is nonnegative.} Let $\nu=Q-Q^\star$, a finite signed measure with $\nu(\cX)=0$. For $t\in[0,1]$, $Q^\star+t\nu=(1-t)Q^\star+tQ\in\cP$, and $h(t)\defeq\He(Q^\star+t\nu)$ is concave on $[0,1]$ (Lemma~\ref{lem:concave}, through the affine map $Q\mapsto\Sigma_Q$) with right derivative $h'(0^+)=\int\Gen{Q^\star}\,d\nu$ (Lemma~\ref{lem:firstvar}). Concavity places $h(1)$ below the tangent at $0$:
\[
    \He(Q)\le\He(Q^\star)+\int\Gen{Q^\star}\,d(Q-Q^\star),
\]
which is exactly $B_{-\He}(Q,Q^\star)\ge0$, and by the definition of $B_{-\He}$,
\begin{equation}\label{eq:entropy-bregman}
    -\lambda\He(Q)
    =-\lambda\He(Q^\star)
    -\lambda\!\int\!\Gen{Q^\star}\,d(Q-Q^\star)
    +\lambda\,B_{-\He}(Q,Q^\star).
\end{equation}

\emph{Step 3: assembly.} Summing \eqref{eq:kl-chain} and \eqref{eq:entropy-bregman}, and substituting $\E_Q[T]=\lambda\E_Q[\Gen{Q^\star}]$,
\[
    F(Q)=\KL(Q\|Q^\star)
    +\lambda\,B_{-\He}(Q,Q^\star)
    +C,
\]
where the $\E_Q[\Gen{Q^\star}]$ terms cancel against the integral in \eqref{eq:entropy-bregman}, leaving $\lambda\E_{Q^\star}[\Gen{Q^\star}]$, and
\[
    C=\lambda\E_{Q^\star}[\Gen{Q^\star}]-\log Z-\lambda\He(Q^\star).
\]
Finally, applying \eqref{eq:kl-chain} at $Q=Q^\star$ gives
\[
    \KL(Q^\star\|\Pref)=\lambda\E_{Q^\star}[\Gen{Q^\star}]-\log Z,
\]
so $C=\KL(Q^\star\|\Pref)-\lambda\He(Q^\star)=F(Q^\star)$, which is \eqref{eq:IGA-pythagoras}.
\end{proof}

\begin{corollary}[Weighted diffusion training]\label{cor:diffusion-IGA}
Let $q^\star$ be the weights of Corollary~\ref{cor:data-tilt} and let $\ell(x;\vartheta)$ be any per-example diffusion training loss (a denoising-score-matching loss or a negative variational bound). Then:
\begin{enumerate}[label=\textup{(\roman*)}]
    \item The weighted objective $\sum_{i=1}^nq_i^\star\ell(x_i;\vartheta)=\E_{X\sim Q^\star_\lambda}[\ell(X;\vartheta)]$ coincides with the corresponding standard training objective with data law $Q^\star_\lambda$; no other component of the training pipeline changes.
    \item Let $\Pref$ and $Q^\star_\lambda$ be as in Proposition~\ref{prop:IGA-pythagoras}. A law $Q$ with $F(Q)<\infty$ attains $\min_{Q'\ll\Pref}F$ if and only if $Q=Q^\star_\lambda$; and for any sequence $(\vartheta_k)$ with $\KL(Q_{\vartheta_k}\|Q^\star_\lambda)\to0$, $F(Q_{\vartheta_k})\to\min_{Q'\ll\Pref}F(Q')$.
\end{enumerate}
\end{corollary}

\begin{proof}
Part (i) is the definition of expectation under a finitely supported law. For part (ii), the ``only if'' direction: if $F(Q)=F(Q^\star_\lambda)<\infty$, then by \eqref{eq:IGA-pythagoras} the two nonnegative terms vanish, in particular $\KL(Q\|Q^\star_\lambda)=0$, so $Q=Q^\star_\lambda$; the converse is trivial. For the convergence claim, write $Q_k=Q_{\vartheta_k}$ and $\delta_k=\sup_{A}|Q_k(A)-Q^\star_\lambda(A)|$, so that $\bigl|\int f\,d(Q_k-Q^\star_\lambda)\bigr|\le2\norm f_\infty\delta_k$ for bounded measurable $f$, and $\delta_k\le\sqrt{\KL(Q_k\|Q^\star_\lambda)/2}\to0$ by Pinsker's inequality. By \eqref{eq:IGA-pythagoras} it suffices that each right-hand term vanishes along the sequence. The KL term does by hypothesis. For the Bregman term, each entry of $\Sigma_{Q_k}-\Sigma_{Q^\star_\lambda}$ is $\int\phi_a\phi_b\,d(Q_k-Q^\star_\lambda)$ with $|\phi_a\phi_b|\le1$, so $\Sigma_{Q_k}\to\Sigma_{Q^\star_\lambda}$, hence $S^\varepsilon_{Q_k}\to S^\varepsilon_{Q^\star_\lambda}$ within the compact set of density matrices with spectrum in $[\varepsilon/d,1]$, on which $\mHe$ is continuous, giving $\He(Q_k)\to\He(Q^\star_\lambda)$; and $\bigl|\int\Gen{Q^\star_\lambda}\,d(Q_k-Q^\star_\lambda)\bigr|\le2(1-\varepsilon)\log(d/\varepsilon)\,\delta_k\to0$. Hence $B_{-\He}(Q_k,Q^\star_\lambda)\to0$, completing the proof.
\end{proof}

\begin{remark}[Saddle computation of the weights, and the relation to DRO]\label{rem:dro}
\emph{(a) Computation.} By Proposition~\ref{prop:spectral-minmax}, on $\Delta_n$ the weights of Corollary~\ref{cor:data-tilt} solve the finite-dimensional saddle problem
\[
\begin{aligned}
    \min_{q\in\Delta_n}\;\max_{\Theta\in\Tset}\;
    \Bigl\{
    \sum_iq_i\log(nq_i)
    &-\lambda(1-\varepsilon)\sum_iq_i\,\phi(x_i)^\top\Theta\,\phi(x_i)
    -\lambda\log\Tr\bigl(e^{-\Theta}\bigr)
    \Bigr\},
\end{aligned}
\]
whose objective is convex and continuous in $q$ on the compact $\Delta_n$ and concave and continuous in $\Theta$ on the compact $\Tset$; by Sion's minimax theorem the order of optimization may be interchanged, and both optima are attained. Alternating best responses are natural: at fixed $q$ the inner maximum is the closed form $\Theta^\star(Q_q)$ of Proposition~\ref{prop:spectral-minmax}, while at fixed $\Theta$ the outer minimization is, by Lemma~\ref{lem:gibbs} with reward $G(x)=\lambda(1-\varepsilon)\phi(x)^\top\Theta\,\phi(x)$, the explicit softmax $q_i\propto\exp\bigl(\lambda(1-\varepsilon)\phi(x_i)^\top\Theta\,\phi(x_i)\bigr)$.

\emph{(b) Not worst-case-loss DRO.} It is instructive to contrast \eqref{eq:IGA-weights} with KL-penalized distributionally robust training of the denoiser,
\[
    \min_\vartheta\;\max_{Q\ll\widehat P_n}\;\bigl\{\E_Q[\ell(X;\vartheta)]-\eta\,\KL(Q\|\widehat P_n)\bigr\},
    \qquad\eta>0.
\]
There, by Lemma~\ref{lem:gibbs}, the inner maximizer reweights the data by the \emph{loss}, $q_i\propto\exp(\ell(x_i;\vartheta)/\eta)$: the adversarial reweighting tracks loss hardness and changes with $\vartheta$ at every step. The IGA reweighting \eqref{eq:IGA-weights} is $\vartheta$-independent and tilts by the \emph{entropy energy}: it up-weights points along spectral directions the data underpopulates, whether or not the current model finds them hard. Comparing the two Gibbs exponents, the schemes produce the same weights only when $\ell(\cdot;\vartheta)$ is, on the training set, an affine function of the IGA exponent, which is a nongeneric coincidence. Replacing the denoiser's training distribution by a worst-case-loss adversary therefore optimizes robustness, not spectral diversity, and is not equivalent to IGA training; the rigorous route to training-time IGA for diffusion models is the reweighting-and-refit composition of Corollaries~\ref{cor:data-tilt} and~\ref{cor:diffusion-IGA}.
\end{remark}

\section{Additional Numerical Results}
\label{app:additional}

\begin{figure*}[!t]
    \centering
    \captionsetup[subfigure]{font=normalsize,labelfont=bf,skip=3pt}

    \begin{subfigure}[t]{0.48\textwidth}
        \centering
        \includegraphics[width=\linewidth]
            {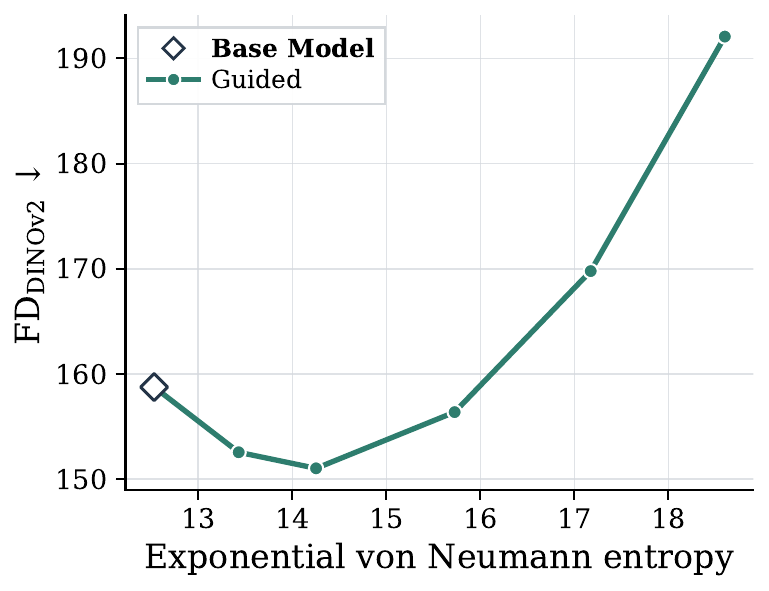}
        \caption{CelebA-HQ: Fr\'echet distance (FD).}
        \label{fig:celeba-dino-fd}
    \end{subfigure}
    \hfill
    \begin{subfigure}[t]{0.48\textwidth}
        \centering
        \includegraphics[width=\linewidth]
            {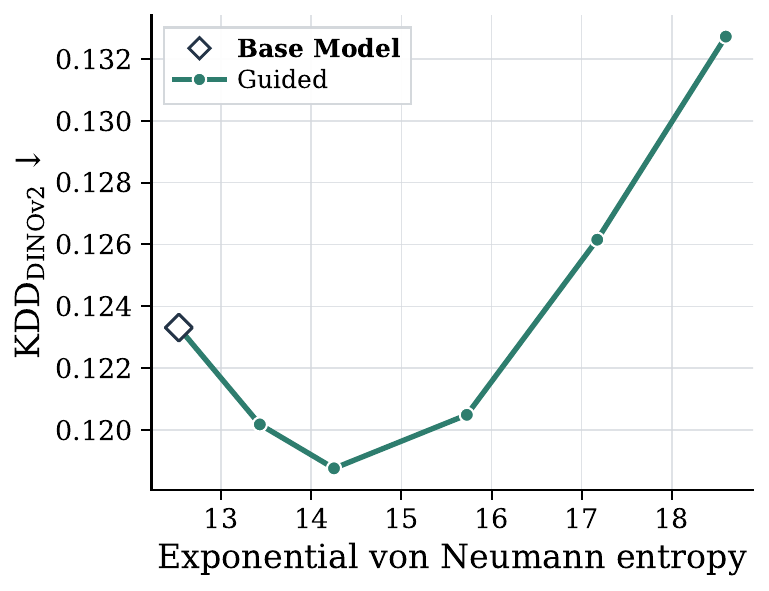}
        \caption{CelebA-HQ: kernel distance (KD).}
        \label{fig:celeba-dino-kd}
    \end{subfigure}

    \vspace{0.6em}

    \begin{subfigure}[t]{0.48\textwidth}
        \centering
        \includegraphics[width=\linewidth]
            {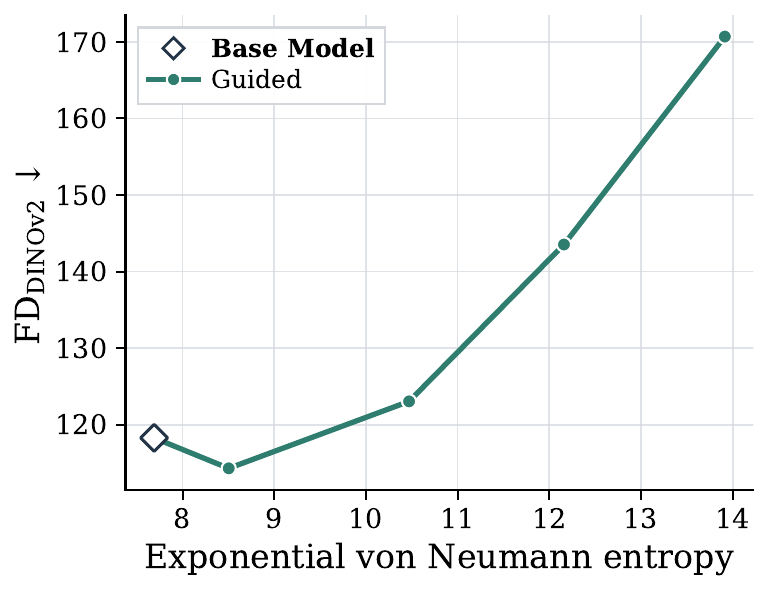}
        \caption{ImageNet: Fr\'echet distance (FD).}
        \label{fig:imagenet-dino-fd}
    \end{subfigure}
    \hfill
    \begin{subfigure}[t]{0.48\textwidth}
        \centering
        \includegraphics[width=\linewidth]
            {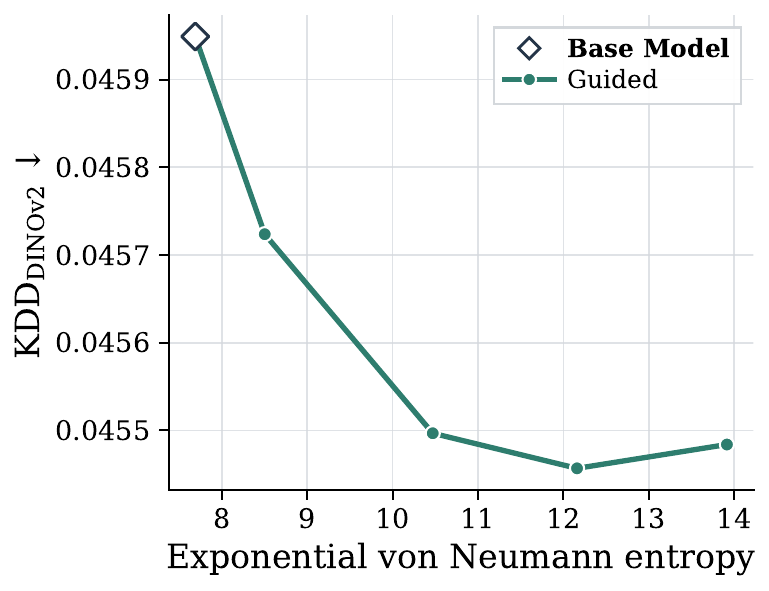}
        \caption{ImageNet: kernel distance (KD).}
        \label{fig:imagenet-dino-kd}
    \end{subfigure}

    \caption{
        \textbf{DINOv2-space distributional distances along the IGA
        target path.}
        Fr\'echet distance (FD) and kernel distance (KD) are evaluated
        in the DINOv2 representation used for the spectral analysis.
        The initial portion of the IGA path reduces the diversity
        deficit while improving or maintaining distributional
        agreement with the data. As $\lambda$ increases further,
        the distance minima occur at metric- and dataset-dependent
        operating points, consistent with the transition from
        below-wall diversity repair to beyond-wall spectral
        extrapolation.
    }
    \label{fig:dino-tradeoffs}
\end{figure*}

This appendix collects supplementary numerical results for the sampling-time and training-time experiments. 
Section~\ref{app:dino} reports Dinov2-space distributional distances for the CelebA-HQ and ImageNet
experiments, complementing the independent Inception-v3 evaluation in the main text. Section~\ref{app:mnist-training-results} reports the MNIST GAN results for the training-time realization of IGA discussed in Section~\ref{subsec:training-time-iga}.

\subsection{DINOv2-Space Distributional Distances}
\label{app:dino}

The main text reports FID and KID in Inception-v3 feature space as an
evaluation independent of the DINOv2 representation used to define the
spectral entropy and entropy wall. Here we provide the corresponding
DINOv2-space distributional distances. These results complement the
independent Inception-v3 evaluation and make explicit the
representation dependence of the precise fidelity optimum along the
IGA regularization path.

\begin{figure}[!t]
    \centering

    \begin{subfigure}[t]{0.95\textwidth}
        \centering
        \includegraphics[
            width=\linewidth
        ]{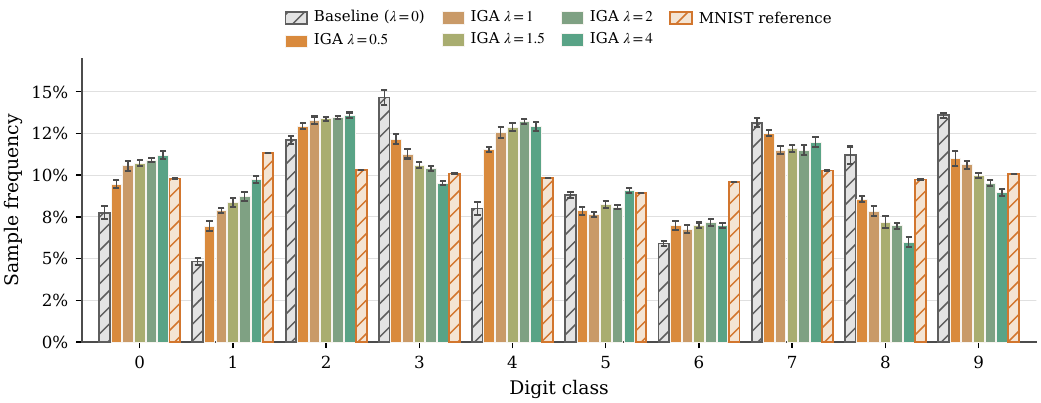}
        \caption{}
        \label{fig:mnist-class-frequencies}
    \end{subfigure}

    \vspace{0.3em}

    \begin{subfigure}[t]{0.46\textwidth}
        \centering
        \includegraphics[
            width=\linewidth
        ]{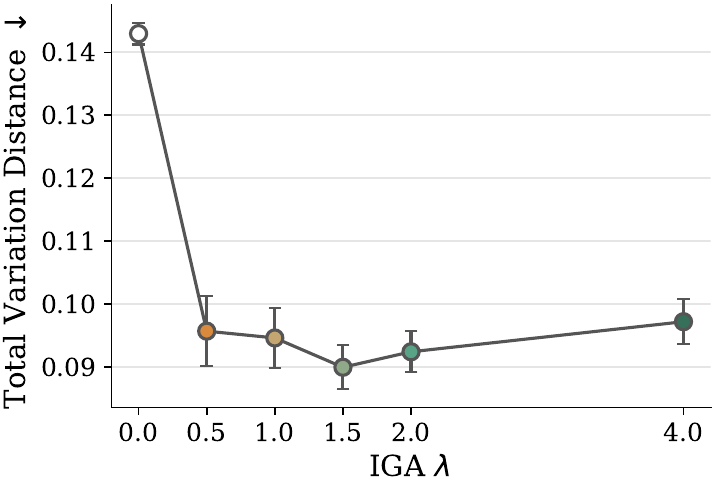}
        \caption{}
        \label{fig:mnist-class-tv}
    \end{subfigure}
    \hfill
    \begin{subfigure}[t]{0.46\textwidth}
        \centering
        \includegraphics[
            width=\linewidth
        ]{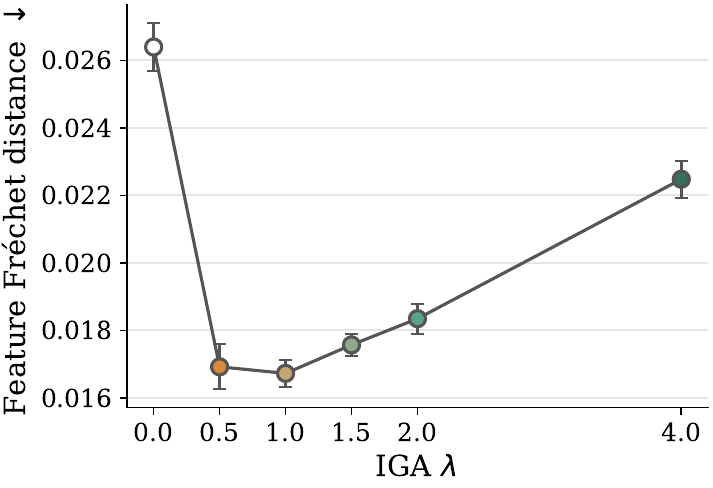}
        \caption{}
        \label{fig:mnist-feature-frechet}
    \end{subfigure}

    \caption{
        \textbf{Training-time IGA reduces class imbalance in an
        MNIST GAN.}
        \textbf{(a)} Generated digit frequencies move toward the
        empirical MNIST distribution as the IGA multiplier increases.
        \textbf{(b,c)} Moderate regularization reduces both
        class-distribution total variation and independent
        evaluator-feature Fr\'echet distance. Results are averaged
        over five seeds; error bars denote standard error.
    }
    \label{fig:mnist-IGA-repair}
\end{figure}

\subsection{Training-Time IGA on MNIST}
\label{app:mnist-training-results}
The training-time experiment of Section~\ref{subsec:training-time-iga}
evaluates whether the same spectral-entropy regularizer used for
sampling-time IGA can be incorporated directly into adversarial
training. Figure~\ref{fig:mnist-IGA-repair} reports the resulting
class-frequency and distributional-fidelity measurements across the
IGA path.

\paragraph{Experimental details.}
We train a convolutional GAN on MNIST for 20 epochs. The generator maps a 64-dimensional standard-normal latent through a fully connected projection and two transposed-convolution stages to a $28\times28$ image, while the discriminator uses two strided convolutional layers followed by a linear output. We use batch size 128 and Adam with learning rate $2\times10^{-4}$ and
$(\beta_1,\beta_2)=(0.5,0.999)$ for both networks, with the
non-saturating logistic generator objective and one discriminator
update per generator update. We evaluate $\lambda\in\{0,0.5,1,1.5,2,4\}$ over five random seeds.

\noindent The fixed IGA representation is the unit-normalized 64-dimensional embedding of a separately trained MNIST classifier. At each generator update, the spectral adversary is recomputed from the current generated minibatch at its closed-form best response and then held fixed during the generator update, as in Proposition~\ref{prop:entropy-grad}. We set $\varepsilon=0.05$. For independent evaluation, we use a second frozen classifier with a different architecture and a 96-dimensional embedding. For each run, 10,000 generated samples are used to compute class frequencies, total variation distance to the empirical MNIST test-set class distribution, and feature Fr\'echet distance between generated and real test samples in the independent evaluator space.

\noindent As shown in Figure~\ref{fig:mnist-IGA-repair}(a), increasing $\lambda$ initially redistributes generated mass away from overrepresented classes and toward classes that are underrepresented by the baseline GAN. This redistribution is reflected quantitatively in panels (b) and (c): both class-distribution total variation and independent evaluator-feature Fr\'echet distance improve substantially at intermediate values of $\lambda$. Their nonmonotone behavior at larger $\lambda$ illustrates the tradeoff between the GAN fidelity objective and the distribution-level entropy reward.

\end{document}